\documentclass{article}

\usepackage{etoolbox}
\newcommand{\arxiv}[1]{\iftoggle{iclr}{}{#1}}
\newcommand{\iclr}[1]{\iftoggle{iclr}{#1}{}}
\newtoggle{iclr}
\global\togglefalse{iclr}

\iclr{
\usepackage[final]{neurips_2025}
}

\newtoggle{horizon}
\global\toggletrue{horizon}
\newcommand{\horizon}[1]{\iftoggle{horizon}{#1}{}}
\newcommand{\nothorizon}[1]{\iftoggle{horizon}{}{#1}}

\usepackage{wrapfig}
\usepackage{needspace}

\newcommand{\loose}{\looseness=-1}

\usepackage[utf8]{inputenc}
\usepackage[T1]{fontenc}
\usepackage{url}
\usepackage{booktabs}
\usepackage{amsfonts}
\usepackage{nicefrac}
\usepackage{microtype}

\usepackage{tocloft}

\usepackage{enumitem}

\usepackage{breakcites}

\usepackage{mathrsfs}

\usepackage{algorithm}
\usepackage[noend]{algpseudocode}

\usepackage{multicol}

\usepackage{colortbl}

\usepackage{setspace}

\usepackage{inconsolata}
\usepackage[scaled=.90]{helvet}
\usepackage{xspace}

\usepackage{pifont}

\arxiv{
\usepackage[letterpaper, left=1in, right=1in, top=1in, bottom=1in]{geometry}
\PassOptionsToPackage{hypertexnames=false}{hyperref}
\usepackage{parskip}
\usepackage[dvipsnames]{xcolor}
\usepackage[colorlinks]{hyperref}
\hypersetup{
    citecolor=[RGB]{50,100,170},
    linkcolor=[RGB]{50,100,170},
    urlcolor=[RGB]{255,102,178}}
}

\iclr{
\usepackage[colorlinks=true, linkcolor=blue!70!black, citecolor=blue!70!black,urlcolor=blue!70!black,breaklinks=true]{hyperref}
\PassOptionsToPackage{dvipsnames}{xcolor}
}

\usepackage{microtype}
\usepackage{hhline}

\usepackage{algorithm}

\arxiv{
\usepackage{natbib}
\bibliographystyle{plainnat}
\bibpunct{(}{)}{;}{a}{,}{,}
}

\usepackage{amsthm}
\usepackage{mathtools}
\usepackage{amsmath}
\usepackage{bbm}
\usepackage{amsfonts}
\usepackage{amssymb}

\usepackage{xpatch}

\usepackage{thmtools}
\usepackage{thm-restate}
\declaretheorem[name=Theorem,parent=section]{theorem}
\declaretheorem[name=Lemma,parent=section]{lemma}

\declaretheorem[qed=$\triangleleft$,name=Example,parent=section]{example}
\declaretheorem[name=Remark,style=definition, parent=section]{remark}
\declaretheorem[name=Proposition, parent=section]{proposition}

\makeatletter
  \renewenvironment{proof}[1][Proof]{\par\noindent{\bfseries\upshape {#1.}\ }}{\qed\newline}
  \makeatother

\theoremstyle{definition}

\newtheorem{corollary}{Corollary}[section]

\theoremstyle{plain}
\newtheorem{definition}{Definition}[section]

\xpatchcmd{\proof}{\itshape}{\normalfont\proofnameformat}{}{}
\newcommand{\proofnameformat}{\bfseries}

\usepackage[nameinlink,capitalize]{cleveref}

\newcommand{\pref}[1]{\cref{#1}}

\renewcommand{\eqref}[1]{\texorpdfstring{\hyperref[#1]{(\ref*{#1})}}{(\ref*{#1})}}

\crefformat{equation}{#2Eq. (#1)#3}
\Crefformat{equation}{#2Eq. (#1)#3}

\Crefformat{figure}{#2Figure #1#3}
\Crefname{assumption}{Assumption}{Assumptions}
\Crefformat{assumption}{#2Assumption #1#3}
\Crefname{subsubsection}{Section}{Sections}
\crefformat{subsubsection}{#2Section #1#3}
\Crefformat{subsubsection}{#2Section #1#3}
\Crefname{alg}{Alg.}{Algs.}

\usepackage{crossreftools}
\usepackage{xparse}

\ExplSyntaxOn
\DeclareDocumentCommand{\XDeclarePairedDelimiter}{mm}
 {
  \__egreg_delimiter_clear_keys:
  \keys_set:nn { egreg/delimiters } { #2 }
  \use:x
   {
    \exp_not:n {\NewDocumentCommand{#1}{sO{}m} }
     {
      \exp_not:n { \IfBooleanTF{##1} }
       {
        \exp_not:N \egreg_paired_delimiter_expand:nnnn
         { \exp_not:V \l_egreg_delimiter_left_tl }
         { \exp_not:V \l_egreg_delimiter_right_tl }
         { \exp_not:n { ##3 } }
         { \exp_not:V \l_egreg_delimiter_subscript_tl }
       }
       {
        \exp_not:N \egreg_paired_delimiter_fixed:nnnnn
         { \exp_not:n { ##2 } }
         { \exp_not:V \l_egreg_delimiter_left_tl }
         { \exp_not:V \l_egreg_delimiter_right_tl }
         { \exp_not:n { ##3 } }
         { \exp_not:V \l_egreg_delimiter_subscript_tl }
       }
     }
   }
 }

\keys_define:nn { egreg/delimiters }
 {
  left      .tl_set:N = \l_egreg_delimiter_left_tl,
  right     .tl_set:N = \l_egreg_delimiter_right_tl,
  subscript .tl_set:N = \l_egreg_delimiter_subscript_tl,
 }

\cs_new_protected:Npn \__egreg_delimiter_clear_keys:
 {
  \keys_set:nn { egreg/delimiters } { left=.,right=.,subscript={} }
 }

\cs_new_protected:Npn \egreg_paired_delimiter_expand:nnnn #1 #2 #3 #4
 {\mathopen{}
  \mathclose\c_group_begin_token
   \left#1
   #3
   \group_insert_after:N \c_group_end_token
   \right#2
   \tl_if_empty:nF {#4} { \c_math_subscript_token {#4} }
 }
\cs_new_protected:Npn \egreg_paired_delimiter_fixed:nnnnn #1 #2 #3 #4 #5
 {
  \mathopen{#1#2}#4\mathclose{#1#3}
  \tl_if_empty:nF {#5} { \c_math_subscript_token {#5} }
 }
\ExplSyntaxOff

\XDeclarePairedDelimiter{\supnorm}{
  left=\lVert,
  right=\rVert,
  subscript=\infty
  }

\renewcommand{\ast}{\star}

\newcommand{\pitil}{\widetilde{\pi}}

\newcommand{\latent}{\mathsf{latent}}
\newcommand{\imitation}{\mathsf{imitation}}

\newcommand{\forward}{\mathtt{Forward}}
\newcommand{\Dagger}{\mathtt{DAgger}}
\newcommand{\pilat}{\pi^{\latent}}
\newcommand{\Pilat}{\Pi^{\latent}}
\newcommand{\rl}{\mathsf{rl}}
\newcommand{\decodability}{\mathsf{decode}}
\newcommand{\contraction}{\mathsf{contract}}
\newcommand{\asmooth}{\mathsf{act}}
\newcommand{\Ynuis}{Y_{\mathsf{nuis}}}

\newcommand{\beltil}{\widetilde\bel}
\newcommand{\BOT}{\widetilde \BO}

\newcommand{\unif}{\mathsf{unif}}

\renewcommand{\epsilon}{\varepsilon}

\newcommand{\pistar}{\pi^{\star}}
\newcommand{\pihat}{\wh{\pi}}

\newcommand{\epopt}{\epsilon_{\mathsf{opt}}}

\newcommand{\poly}{\mathrm{poly}}
\newcommand{\bigoh}{\cO}
\newcommand{\littleoh}{o}

\newcommand{\MO}{\mathcal{O}}
\newcommand{\MS}{\mathcal{S}}
\newcommand{\BO}{\mathbb{O}}
\newcommand{\MA}{\mathcal{A}}

\newcommand{\norm}[1]{\left \lVert #1 \right \rVert}
\newcommand{\Dpr}[2]{D_\star(#1 \| #2)}
\newcommand{\bel}{\mathbf{b}}
\newcommand{\belief}{\mathbf{b}}
\newcommand{\BB}{\mathbb{B}}
\newcommand{\BU}{\mathbb{U}}
\newcommand{\belapx}{\mathbf{b}^\mathsf{apx}}
\newcommand{\TV}{\mathsf{TV}}

\renewcommand{\t}{\top}
\DeclareMathOperator{\Unif}{Unif}
\newcommand{\Vinf}{\mathsf{V}_{\infty}}

\newcommand{\MT}{\mathcal{T}}

\newcommand{\EE}{\mathbb{E}}

\newcommand{\NN}{\mathbb{N}}

\newcommand{\RR}{\mathbb{R}}

\DeclarePairedDelimiter{\abs}{\lvert}{\rvert}
\DeclarePairedDelimiter{\brk}{[}{]}
\DeclarePairedDelimiter{\crl}{\{}{\}}
\DeclarePairedDelimiter{\prn}{(}{)}
\DeclarePairedDelimiter{\nrm}{\|}{\|}

\let\Pr\undefined

\DeclareMathOperator{\En}{\mathbb{E}}

\DeclareMathOperator{\Pr}{Pr}

\DeclareMathOperator*{\argmax}{arg\,max}

\newcommand{\wh}[1]{\widehat{#1}}

\newcommand{\Dkl}[2]{\mathsf{KL}\prn*{#1\,\|\,#2}}

\def\ddefloop#1{\ifx\ddefloop#1\else\ddef{#1}\expandafter\ddefloop\fi}
\def\ddef#1{\expandafter\def\csname bb#1\endcsname{\ensuremath{\mathbb{#1}}}}
\ddefloop ABCDEFGHIJKLMNOPQRSTUVWXYZ\ddefloop
\def\ddefloop#1{\ifx\ddefloop#1\else\ddef{#1}\expandafter\ddefloop\fi}
\def\ddef#1{\expandafter\def\csname b#1\endcsname{\ensuremath{\mathbf{#1}}}}
\ddefloop ABCDEFGHIJKLMNOPQRSTUVWXYZ\ddefloop
\def\ddef#1{\expandafter\def\csname sf#1\endcsname{\ensuremath{\mathsf{#1}}}}
\ddefloop ABCDEFGHIJKLMNOPQRSTUVWXYZ\ddefloop
\def\ddef#1{\expandafter\def\csname c#1\endcsname{\ensuremath{\mathcal{#1}}}}
\ddefloop ABCDEFGHIJKLMNOPQRSTUVWXYZ\ddefloop
\def\ddef#1{\expandafter\def\csname h#1\endcsname{\ensuremath{\widehat{#1}}}}
\ddefloop ABCDEFGHIJKLMNOPQRSTUVWXYZ\ddefloop
\def\ddef#1{\expandafter\def\csname hc#1\endcsname{\ensuremath{\widehat{\mathcal{#1}}}}}
\ddefloop ABCDEFGHIJKLMNOPQRSTUVWXYZ\ddefloop
\def\ddef#1{\expandafter\def\csname t#1\endcsname{\ensuremath{\widetilde{#1}}}}
\ddefloop ABCDEFGHIJKLMNOPQRSTUVWXYZ\ddefloop
\def\ddef#1{\expandafter\def\csname tc#1\endcsname{\ensuremath{\widetilde{\mathcal{#1}}}}}
\ddefloop ABCDEFGHIJKLMNOPQRSTUVWXYZ\ddefloop
\def\ddefloop#1{\ifx\ddefloop#1\else\ddef{#1}\expandafter\ddefloop\fi}
\def\ddef#1{\expandafter\def\csname scr#1\endcsname{\ensuremath{\mathscr{#1}}}}
\ddefloop ABCDEFGHIJKLMNOPQRSTUVWXYZ\ddefloop

\newcommand{\veps}{\varepsilon}

\iclr{\renewcommand{\citet}{\citep}}

\iclr{\setlist[itemize]{leftmargin=*}}
\iclr{\setlist[enumerate]{leftmargin=*}}

\makeatletter
\let\OldStatex\Statex
\renewcommand{\Statex}[1][3]{\setlength\@tempdima{\algorithmicindent}\OldStatex\hskip\dimexpr#1\@tempdima\relax}
\makeatother

\usepackage{accents}
\usepackage{tikz}
\usetikzlibrary{decorations.pathreplacing,arrows.meta}
\usepackage{varwidth}

 \addtocontents{toc}{\protect\setcounter{tocdepth}{0}}

\let\oldparagraph\paragraph
\arxiv{\renewcommand{\paragraph}[1]{\oldparagraph{#1.}}}
\iclr{\renewcommand{\paragraph}[1]{\textbf{#1.}}}

\algrenewcommand\algorithmicrequire{\textbf{require}}
\usepackage{yfonts}

\arxiv{
    \title{To Distill or Decide? Understanding the Algorithmic Trade-off in Partially Observable Reinforcement Learning}
    \author{
      Yuda Song$^1$ \; Dhruv Rohatgi$^2$ \; Aarti Singh$^1$ \; J. Andrew Bagnell$^{1,3}$\\
      \vspace{-2mm} \\
      \normalsize{$^1$Carnegie Mellon University \qquad $^2$Massachusetts Institute of Technology \qquad $^3$Aurora Innovation}\\
      \vspace{-2mm} \\
      \normalsize{\texttt{yudas@cs.cmu.edu}},\; \texttt{drohatgi@mit.edu},\; \texttt{aarti@cs.cmu.edu},\; \texttt{dbagnell@aurora.tech}
    }
\date{}
}

\iclr{
  \title{To Distill or Decide? The Algorithmic Trade-off in Partially Observable Reinforcement Learning}

\author{Yuda Song \\
  Carnegie Mellon University\\
  \texttt{yudas@cs.cmu.edu} \\
  \And
  Dhruv Rohatgi \\
  Massachusetts Institute of Technology \\
  \texttt{drohatgi@mit.edu} \\
  \AND
  Aarti Singh \\
  Carnegie Mellon University \\
  \texttt{aarti@cs.cmu.edu} \\
  \And
  J. Andrew Bagnell \\
  Aurora Innovation,  Carnegie Mellon University\\
  \texttt{dbagnell@aurora.tech} \\
}

}

\begin{document}

\maketitle

\begin{abstract}
Partial observability is a notorious challenge in reinforcement learning (RL), due to the need to learn complex, history-dependent
policies. Recent empirical successes have used \emph{privileged expert distillation} --- which leverages availability of latent state information during training (e.g., from a simulator) to learn and imitate the optimal latent, Markovian policy --- to disentangle the task of ``learning to see'' from ``learning to act'' \citep{pan2017agile,choudhury2018data, chen2019lbc}.
While expert distillation is more computationally efficient
than RL without latent state information, it also has well-documented failure modes. In this paper --- through a simple but instructive theoretical model called the \emph{perturbed Block MDP}, and controlled experiments on challenging simulated locomotion tasks --- we investigate the algorithmic trade-off between privileged expert distillation and standard RL without privileged information. Our main findings are: \textbf{(1)} The trade-off empirically hinges on the \emph{stochasticity} of the latent dynamics, consistent with the theoretical distinction between \emph{approximate decodability} and \emph{belief contraction} in the perturbed Block MDP; and \textbf{(2)} An expert trained for the target environment is not always the best expert to distill: training with additional motor noise can improve distillation performance.
Our results suggest new guidelines for effectively exploiting privileged information, potentially advancing the efficiency of policy learning across many practical partially observable domains.
\end{abstract}

\section{Introduction}
\label{sec:intro}
Partial observability is a common challenge in applied reinforcement learning: the decision-making agent may not see the true state of the environment at all time-steps, whose information might only be probabilistically inferred from the history of observations.
An illustrative task is robot learning for robots with \emph{image-based perception} \citep{Pinto2017AsymmetricAC,Shridhar2021CLIPortWA}. A single image of the robot (or, in first-person perspective, of the environment) will not capture important elements of the state such as the robot's velocity, and may miss other features due to e.g. occlusion or limited view.

The canonical theoretical model for such tasks is \emph{Partially Observable Markov Decision Process (POMDP)}. Unfortunately, there are well-documented computational \citep{papadimitriou1987complexity} and statistical \citep{jin2020sample} barriers to planning and learning in POMDPs, which have motivated many theoretical works that seek to bypass these barriers by making additional structural assumptions \citep{jin2020sample,kwon2021rl, efroni2022provable, golowich2023planning, golowich2022learning,liu2023optimistic}. On the empirical side, the standard technique for mitigating partial observability is \emph{frame-stacking}, which enabled notable successes for learning to play Atari games \citep{mnih2013playing,mnih2015human}. The idea is to treat the ``state'' of the environment as the concatenation of a short window of $L$ recent observations, and apply a standard algorithm for fully-observed reinforcement learning (RL). This technique inspired theoretical developments such as $L$-step decodability \citep{efroni2022provable}, and has some theoretical underpinnings for \emph{$\gamma$-observable} POMDPs \citep{golowich2023planning}. Yet frame-stacking is not a silver bullet for partially observable decision-making: sometimes effective planning requires long memory \citep{eberhard2025partially}. Also, high-dimensional observations (such as stacks of images) can confound learning complex behaviors \citep{Pinto2017AsymmetricAC,warrington2021robust}.

\paragraph{Learning from latent state information}
A common heuristic for planning in known POMDPs is to use the optimal \emph{latent} policy (also known as the state-based policy or privileged policy) --- i.e., the optimal policy that is allowed to ``cheat'' and see the underlying state of the environment --- as a starting point for computing an \emph{executable} policy --- i.e. a policy that only depends on the observable history \citep{littman1995learning,ross2008online,choudhury2018data}. More recent works have brought this ansatz to the learning task, where the description of the POMDP is a priori unknown. In the standard theoretical formalization of this task \citep{krishnamurthy2016pac}, the latent states of the POMDP are never observed (nor even identifiable); however, for applications such as robotics, it is often practically reasonable to construct a \emph{simulator} of the environment \citep{christiano2016transfer}, from which the learning agent may draw trajectories that include both the observations as well as the latent states --- ``privileged'' information that is only available at training time, not at test time.

The most prominent paradigm for exploiting this additional information is called \emph{privileged expert distillation},\footnote{The same paradigm is sometimes called \emph{Learning by Cheating} or \emph{Teacher to Student Learning}.} which applies methods from imitation learning and structured prediction \citep{daume2009search,ross2010efficient,ross2011reduction,chang2015learning} to learning in POMDPs. Expert distillation has two steps: (1) learn an optimal latent policy, using a standard RL algorithm with the latent state information provided by the simulator; and (2) distill the latent policy to an executable policy, using an imitation learning algorithm such as $\Dagger$ \citep{ross2011reduction}. This paradigm has achieved impressive success in applications such as autonomous driving \citep{chen2019lbc}, robotics \citep{Lee2020LearningQL,doi:10.1126/scirobotics.abk2822,zhuang2023robot} and LLMs \citep{choudhury2025process}.

These successes suggest a fundamental question: \emph{when does expert distillation help in realistic decision-making tasks}? On the one hand, in controlled experiments, expert distillation uniformly converges faster and more stably than RL without latent state information \citep{mu2025should}, likely because it disentangles representation learning from decision-making \citep{chen2019lbc}. Moreover, expert distillation enjoys a provable computational advantage in decodable POMDPs\footnote{Decodable POMDPs without any prefix refer to $H$-step decodable POMDPs, where $H$ is the horizon of the POMDPs.} \citep{cai2024provable}. On the other hand, there are well-documented failure modes of expert distillation --- most notably, due to its inability to encourage purely information-gathering actions \citep{arora2018hindsight, weihs2021bridging} --- where more expensive hybrid methods such as Asymmetric Actor-Critic \citep{Pinto2017AsymmetricAC} are fundamentally required.

In this paper, motivated by image-based locomotion tasks, we focus on the middle ground where (perfect) decodability may fail, yet the observations are still highly informative of the latent state. In this regime, we ask: (i) when and how is expert distillation as performant as standard RL with frame-stacking, and (ii) are there lightweight \emph{improvements} to expert distillation? We use simple theoretical models in tandem with controlled experiments to address the preceding questions.
\paragraph{Our contributions}
\begin{enumerate}
\item The prior theoretical model for understanding the benefits of latent state information was a \emph{perfectly decodable} POMDP \citep{cai2024provable}. We begin with empirical evidence suggesting that this model is too restrictive for image-based locomotion tasks.
\item We then introduce \emph{approximate decodability}, and connect it to the success of expert distillation --- in analogy with the connection between \emph{belief contraction} and the success of standard reinforcement learning with frame-stacking. But when are these conditions satisfied? As a theoretical testbed, we introduce the \emph{perturbed Block MDP}.
\item Our analysis of the perturbed Block MDP shows that distillation is competitive with RL under deterministic dynamics, while stochastic dynamics can cause irreducible distillation error. In our experiments, distillation becomes less competitive as injected motor noise increases.
\item Finally, we show empirically that the expert trained in the original environment need not be the best one to distill: \emph{adding stochasticity to the latent MDP} during expert training can improve distillation performance and reduce action-prediction error, consistent with the benefits of expert \emph{smoothness}.
\end{enumerate}

\section{Preliminaries}
\label{sec:prelim}
A (finite-horizon, layered) Partially Observable Markov Decision Process (POMDP) is a tuple
$\cP = (H,\cX,\cS,\cA,\bbP,\bbO,R)$, where $H \in \NN$ is the horizon, $\cX = \crl*{\cX_h}_{h=1}^H$ is the observation space, $\cS = \crl*{\cS_h}_{h=1}^H$ is the latent state space, $\cA = \crl*{\cA_{h}}_{h=1}^H$ is the action space, $\bbP$ consists of an initial distribution $\bbP_1\in\Delta(\cS_1)$ and transition kernels $\bbP_h:\cS_{h-1}\times\cA_{h-1}\to\Delta(\cS_h)$ for $h=2,\ldots,H$, $\bbO = \crl*{\bbO_h: \cS_h \to \Delta(\cX_h)}_{h=1}^H$ describes the emission distributions, and $R = \crl*{R_h: \cS_h \times \cA_h \to [0,1]}$ describes the rewards. We write $A := \max_h \abs*{\cA_h}$, $S := \max_h \abs*{\cS_h}$, and $X := \max_h\abs*{\cX_h}$. Given any timestep $h$ and $L \in [H]$, we denote $\cX^{h-L:h} := \cX_{h-L} \times \cX_{h-L+1} \times \dots \times\cX_{h}$, and similarly for $\cA^{h-L:h}$, where lower endpoints of time windows are truncated at $1$, and a window with lower endpoint larger than its upper endpoint is empty. Then an $L$-step \emph{executable} policy is a collection $\pi = \{\pi_h: \cX^{h-L+1:h} \times \cA^{h-L:h-1} \to \Delta(\cA_h)\}$; we let $\Pi^L$ denote the class of such policies.
Given any executable policy $\pi\in\Pi:=\Pi^H$, a trajectory $\tau = \prn*{s_1,x_1,a_1,r_1, \dots, s_H,x_H,a_H,r_H}$ is generated by drawing $s_1\sim\bbP_1$, then, at each step $h$, drawing $x_h\sim\bbO_h(s_h)$ and $a_h\sim\pi(x_{1:h},a_{1:h-1})$ and receiving $r_h=R_h(s_h,a_h)$; for $h<H$, the next state is drawn from $\bbP_{h+1}(s_h,a_h)$. We use $\bbP^\pi$ and $\En^\pi$ to denote the law and expectation under this process. Following convention, we assume $\sum_{h=1}^H r_h \leq 1$ almost surely under all policies. The \emph{value} of a policy $\pi$ is $J(\pi) := \En^\pi\brk*{\sum_{h=1}^H r_h}$.

Note that the POMDP $\cP$ also defines an underlying Markov Decision Process (MDP) $\cM = \{\cS,\cA,\bbP,R,H\}$ (which we call the \emph{latent MDP}) where the state is fully observable. A latent (Markovian) policy is a collection $\pi^{\latent} = \{\pi^{\latent}_h: \cS_h \to \Delta(\cA_h)\}$, and we let $\Pi^{\latent}$ denote the class of latent policies.
A latent trajectory $\tau^{\latent} = (s_1,a_1, \dots, s_H,a_H)$ is generated by drawing $s_1\sim\bbP_1$, $a_h\sim\pi^{\latent}_h(s_h)$, and, for $h<H$, $s_{h+1}\sim\bbP_{h+1}(s_h,a_h)$; we define $\bbP^{\pi^{\latent}}$ and $\En^{\pi^{\latent}}$ accordingly. When comparing trajectory laws of latent and executable policies, we include observations and rewards generated by the same environment in both laws, so they are distributions on the same space. For any policy $\pi$, write $d_h^\pi$ for the marginal law of $s_h$. \loose

\paragraph{Learning with/without latent state information} In the standard theoretical RL access model (i.e. without latent state information) \citep{kearns2002near,jiang2017contextual}, at training time, the learning agent can repeatedly interact with the POMDP $\cP$ by playing an executable policy $\pi$ and observing the partial trajectory $(x_{1:H},a_{1:H},r_{1:H})$. In contrast, in the \emph{learning with latent state information} model \citep{cai2024provable}, at training time, the learning agent can play \emph{any} policy, and observes the full trajectory $(s_{1:H},x_{1:H},a_{1:H},r_{1:H})$. In both settings, the goal is to eventually produce an \emph{executable} policy $\pihat$ that minimizes $J(\pistar) - J(\pihat)$ (where $\pistar$ is the optimal executable policy).

\paragraph{Belief states} A \emph{belief state} is a distribution over latent states. For a prior $b$ on the latent state at step $h-1$, let $\BU_h(b;a_{h-1},x_h)$ be the posterior on the latent state at step $h$ after taking action $a_{h-1}$ and then observing $x_h$ (see \cref{def:bel-update} for the formal algebraic definition).

\begin{definition}\label{def:true-belief-state}
For any observation/action sequence $(x_{1:h}, a_{1:h-1})$, the \emph{true belief state} $\belief_h(x_{1:h},a_{1:h-1})$ is defined as follows. For $h = 1$ with observation $x_1$, let $\belief_1(x_1) := \BB_1(\bbP_1;x_1)$. For any $2 \leq h \leq H$, let
\begin{align}
  \label{eq:belief-defn}
  \belief_h(x_{1:h},a_{1:h-1}) := \BU_h(\belief_{h-1}(x_{1:h-1}, a_{1:h-2}); a_{h-1}, x_h).
  \end{align}
\end{definition}
For any executable policy $\pi$, step $h$, and history $(x_{1:h},a_{1:h-1})$ of positive probability under $\pi$, $\belief_h(x_{1:h},a_{1:h-1})$ is the distribution of the latent state $s_h$ under $\bbP^\pi$, conditioned on $(x_{1:h},a_{1:h-1})$ (\cref{lemma:belief-is-cond-prob}).

Many methods for efficient planning in POMDPs are based on \emph{approximate belief states} that only depend on a short window of recent actions and observations \citep{kara2022near,golowich2023planning}. Informally, the approximate belief state $\belapx_h(x_{h-L+1:h}, a_{h-L:h-1};\cD)$ is the posterior on state $s_h$ after observing $(x_{h-L+1:h}, a_{h-L:h-1})$ with prior $\cD$ on state $s_{h-L}$. See \cref{def:apx-belief} for the formal definition (analogous to \cref{def:true-belief-state}).

\paragraph{Additional notation} For distributions $b,b' \in \Delta(\cS_h)$, the density ratio is $\norm{b/b'}_\infty = \sup_{s \in \cS_h} b(s)/b'(s) \in [1,\infty]$, with the convention that $0/0 = 1$. For a belief state $b \in \Delta(\cS_h)$ and conditional distribution $\pi_h: \cS_h \to \Delta(\cA_h)$, we let $\pi_h \circ b$ denote the distribution over $\cA_h$ obtained as \loose
\begin{equation}\label{eq:belief-pi-comp}
\textstyle    (\pi_h\circ b)(a_h) := \sum_{s_h\in\cS_h} b(s_h) \pi_h(a_h \mid s_h).
\end{equation}
\textbf{Experimental Setup.} We use three tasks in the Deepmind control suite \citep{Tassa2018DeepMindCS}: \texttt{walker-run}, \texttt{dog-walk} and the challenging \texttt{humanoid-walk}. To implement online (resp., offline) expert distillation, we (1) train an expert on the latent state information using \texttt{MrQ} \citep{fujimoto2025towards}, and (2) imitate the expert via $\Dagger$ \citep{ross2011reduction} (resp., Behavior Cloning (BC)) on $L$-step executable policies. Unless otherwise specified, we use the standard choice of $L=3$, and we use mean squared error (MSE) as the loss function: give input $X = \{x^i\}_{i=1}^N$ and target $Y = \{y^i \in \bbR^d\}_{i=1}^N$, the loss of a function $f$ is $\ell(f,X,Y) = \frac{1}{Nd}\sum_{i=1}^N \sum_{j=1}^d (f(x^i)_j - y^i_j)^2$. To implement reinforcement learning (RL), we use \texttt{MrQ} \citep{fujimoto2025towards} on $L$-step executable policies. In experiments, we follow the common empirical practice of only stacking observations (rather than both observations and actions). \loose

\paragraph{Appendices} See \cref{sec:related} for additional related work, and \cref{app:experiment} for experimental details.

\section{Approximate Decodability and Belief Contraction}
\label{sec:decodable}
Even with access to latent state information during training, the problem of learning a near-optimal policy in a POMDP is as hard as the \emph{planning} task (where a description of the POMDP is already known), which is well-known to be computationally intractable in the worst case \citep{papadimitriou1987complexity}. However, POMDPs encountered in practice will often satisfy additional structural properties that may mitigate this hardness. Some of the most widely-studied properties are \emph{decodability} \citep{efroni2022provable,cai2024provable} and \emph{belief contraction} (also known as \emph{filter stability}) \citep{kara2022near, golowich2023planning}.

Privileged information is known to yield a provable computational benefit in decodable POMDPs \citep{cai2024provable}. However, the experiments in \cref{sec:perfect-decodable} suggest that perfect decodability may be too restrictive for our motivating tasks. For this reason, in \cref{sec:errors} we introduce the notion of \emph{approximate decodability}. Heuristically, this property governs the success of expert distillation with $L$-step framestacking, whereas belief contraction governs the success of standard RL (also with $L$-step framestacking). But when are these properties satisfied? As a clean theoretical testbed for studying this question, in \cref{sec:perturbed-mdp} we introduce the \emph{$\delta$-perturbed Block MDP}. \loose

\subsection{Prior Work: Perfectly Decodable POMDPs}\label{sec:perfect-decodable}
In some applications, such as video games, it is plausible that the agent can deduce the latent state from a small number of recent observations. This was empirically substantiated by the success of DQN \citep{mnih2013playing} and its variants, which only use the most recent four observations as policy inputs.
Theoretically, this motivated the study of the $L$-step decodable model \citep{efroni2022provable}, which posits that the most recent $L$ observations and actions suffice to fully disambiguate the latent state (\cref{def:perfectly-decodable}).

Without latent state information (i.e. in the standard RL access model), learning a near-optimal policy in an $L$-step decodable POMDP can require $\Omega(A^L)$ samples at fixed accuracy in the worst case \citep{efroni2022provable}. However, with latent state information, \citet{cai2024provable} give an algorithm that learns an executable policy $\hat\pi\in\Pi$ satisfying $J(\hat\pi)\geq J(\pistar)-\epsilon$ with probability at least $1-\beta$, using $\poly(S,A,X,H,\epsilon^{-1},\log(\beta^{-1}))$ samples and time. Thus, for large $L$, there is a clear theoretical benefit of latent state information (both statistically and computationally). However, unfortunately, $L$-step decodability is not always a realistic assumption:

\textbf{Empirical test: does perfect decodability hold?} Through controlled experiments on our three chosen locomotion tasks (\cref{sec:prelim}), we observe persistent state-prediction error, especially in early timesteps. This suggests that perfect decodability may be too strong, although prediction error alone cannot exclude approximation or optimization error. We defer details of this experiment to \pref{app:exp-mis-dec}.
\subsection{Errors in POMDPs}\label{sec:errors}

The above empirical result motivates the following theoretical definition of decodability error:

\begin{definition}[Decodability Error]\label{def:decodability-error} Fix a POMDP $\cP$. The \emph{decodability error} for an executable policy $\pi$ and timestep $h \in [H]$ is
  \begin{align*}
        \epsilon^{\decodability}_h(\pi) := \En^{\pi} \brk*{ 1- \nrm*{\belief_h(x_{1:h}, a_{1:h-1})}_{\infty}}.
\end{align*}

\end{definition}

Intuitively, decodability error quantifies stochasticity of the true belief. Below, we show that it upper bounds the \emph{misspecification} of any latent policy $\pilat$ with respect to the class of executable policies.

\begin{lemma}[See \cref{lemma:tv-beltil-to-latent}]\label{lemma:tv-beltil-to-latent-main}
Let $\pilat \in \Pilat$ be a latent policy and let $\beltil_{1:H}$ be a collection of functions $\beltil_h: \cX^h \times \cA^{h-1} \to \Delta(\cS_h)$. Define executable policies $\pitil,\pi$ by \arxiv{\[\pitil(x_{1:h},a_{1:h-1}) := \pilat \circ \beltil_h(x_{1:h},a_{1:h-1}).\]}\iclr{$\pitil(x_{1:h},a_{1:h-1}) := \pilat \circ \beltil_h(x_{1:h},a_{1:h-1})$} and $\pi(x_{1:h},a_{1:h-1}) := \pilat \circ \bel_h(x_{1:h},a_{1:h-1})$ (see \pref{eq:belief-pi-comp}). Then
\begin{equation} \arxiv{J(\pilat) - J(\pitil) \leq} \TV(\bbP^{\pilat}, \bbP^{\pitil}) \leq 2\sum_{h=1}^H \epsilon^{\decodability}_h(\pi) + \sum_{h=1}^H\EE^{\pitil}\left[\norm{\bel_h(x_{1:h}, a_{1:h-1}) - \beltil_h(x_{1:h}, a_{1:h-1})}_1\right].\label{eq:misspec-decodability-bound}\end{equation}\end{lemma}
Low decodability error is not strictly required for low misspecification (see \cref{sec:smooth}\arxiv{ for an improvement}), but some such assumption is needed to rule out models requiring active information-gathering \citep{weihs2021bridging}. As a special case, \cref{lemma:tv-beltil-to-latent-main} implies that $\pilat$ is $2\sum_{h=1}^H \epsilon^{\decodability}_h(\pi)$-close to the executable policy $\pi$, which evaluates $\pilat$ at a random state $s'_h$ sampled from the true belief $\belief_h(x_{1:h},a_{1:h-1})$; this is because if $\belief_h(x_{1:h},a_{1:h-1})$ is highly concentrated, then $s'_h$ likely matches the true latent state. The second error term in \cref{eq:misspec-decodability-bound} quantifies error in learning the true belief --- e.g., due to using only $L$-step histories.\arxiv{\\}
Next, it is instructive to contrast decodability error with \emph{belief contraction error}, the discrepancy between the true belief and the approximate belief induced by the $L$ most recent observations/actions:

\begin{definition}[Belief Contraction Error \citep{golowich2023planning}]\label{def:belief-contraction-main} Fix a POMDP $\cP$. For an executable policy $\pi$, and timestep $h \in [H]$, the $L$-step \emph{belief contraction error} ($L \in [h-1]$) is
\begin{align*}
    \epsilon^{\contraction}_h(\pi;L) :=
\En^\pi \brk*{ \nrm*{\belief_h(x_{1:h}, a_{1:h-1}) - \belapx_h(x_{h-L+1:h},a_{h-L:h-1}; \unif(\cS_{h-L}))}_{1}}.
\end{align*}
\end{definition}

In the absence of latent state information, bounding the belief contraction error is the standard method of analyzing provably efficient algorithms for RL in POMDPs \citep{kara2022near,uehara2022provably,golowich2023planning}. Indeed, belief contraction implies that the POMDP with $L$-step frame-stacking is approximately Markovian, which heuristically suggests that a standard RL algorithm \citep{kearns2002near,brafman2002r} with frame-stacking should achieve low error in time $\approx(AX)^{\bigoh(L)}$. Due to technical issues with error compounding, this is not formally true, but under an additional \emph{observability} condition, there \emph{is} an algorithm that provably achieves that guarantee: \loose

\begin{theorem}[Informal; see \cref{thm:golowich}; adapted from \cite{golowich2022learning}]\label{thm:golowich-main}
    Suppose the POMDP is $\gamma$-observable (\pref{def:observable}) and satisfies the generalized belief-contraction hypothesis of \cref{thm:golowich} with window length $L$ and error $\epsilon$. There exists a reinforcement learning algorithm that achieves the sub-optimality bound
    \begin{align*}
        J(\pi^{\ast}) - J(\pi^{\rl}) \leq \epsilon \cdot \poly(S,X,H,\gamma^{-1}),
    \end{align*}
    with probability at least $2/3$, in time $(XA)^{\bigoh(L)} \cdot \poly(H,S,\gamma^{-1}, \epsilon^{-1})$.
\end{theorem}

The result of \citet{golowich2022learning} chooses $L$ large enough for observability to imply belief contraction. In \cref{thm:golowich}, we extend their analysis to any $L$ satisfying generalized belief contraction and to state/action rewards using empirical reward estimates. In particular, \cref{thm:golowich-main} gives a polynomial-time algorithm when belief contraction holds for $L=\bigoh(1)$. \loose

\subsection{The Perturbed Block MDP}\label{sec:perturbed-mdp}

Approximate decodability and belief contraction are conditions under which expert distillation and standard RL with frame-stacking, respectively, may be reasonably expected to succeed. But when are these conditions satisfied, and how do they compare? As a theoretical testbed, we introduce the \emph{perturbed Block MDP} model. Block MDPs \citep{du2019provably} are a well-studied abstraction of environments with rich observations yet simple latent dynamics. However, they assume that the latent state is fully determined by the current observation. Below, we generalize Block MDPs by allowing for $\delta$ probability that the observation is sampled from an arbitrary conditional distribution.\footnote{To be clear, our theoretical focus is on issues arising from partial observability, not on representation learning. The size of the observation  space is conceptually tangential, so we omit introducing technical complications such as function approximation\iclr{, which are central to theory for fully-observed Block MDPs \citep{zhang2022efficient,mhammedi2023representation,rohatgi2025necessary}.}\arxiv{. In contrast, representation learning and the associated computational challenges are the focus of the literature on standard Block MDPs \citep{zhang2022efficient,mhammedi2023representation,rohatgi2025necessary}.}}

\begin{definition}\label{ass:perturbed-block-main}
Fix a parameter $\delta\in[0,1]$. A POMDP $\cP$ is a \emph{$\delta$-perturbed Block MDP} if, for each $h \in [H]$, there are $\BOT_h,E_h: \cS_h\to\Delta(\cX_h)$ such that $\BOT_h: \cS_h\to\Delta(\cX_h)$ satisfies the \emph{block} property \citep{du2019latent}, i.e. $\BOT_h(\cdot\mid{}s_h), \BOT_h(\cdot\mid{}s_h')$ have disjoint supports for all $s_h \neq s'_h$, and moreover the emission distribution $\BO_h$ at step $h$ can be decomposed as follows:
\arxiv{\[\BO_h(x_h\mid{}s_h) = (1-\delta) \BOT_h(x_h\mid{} s_h) + \delta E_h(x_h\mid{}s_h).\]}\iclr{$\BO_h(x_h\mid{}s_h) = (1-\delta) \BOT_h(x_h\mid{} s_h) + \delta E_h(x_h\mid{}s_h).$}
\end{definition}

\begin{figure}[t]
\centering
\begingroup
\definecolor{aInk}{HTML}{263943}
\definecolor{aMuted}{HTML}{667782}
\definecolor{aClean}{HTML}{16756C}
\definecolor{aNoise}{HTML}{C16B32}
\definecolor{aRule}{HTML}{CDD5D9}
\resizebox{\linewidth}{!}{\begin{tikzpicture}[
  x=1cm,y=1cm,font=\small,text=aInk,
  state/.style={circle,draw=aInk!75,fill=white,line width=.65pt,
    minimum size=8.2mm,inner sep=0pt},
  branch/.style={rounded corners=2pt,minimum width=2.9cm,text width=2.75cm,
    minimum height=9.5mm,inner xsep=3pt,align=center,line width=.65pt},
  edge/.style={-{Stealth[length=1.8mm,width=1.25mm]},line width=.65pt,draw=aInk!80},
  noise edge/.style={edge,draw=aNoise,dash pattern=on 3pt off 2pt}
]
\path[use as bounding box] (0,.35) rectangle (15,-4.2);
\node[anchor=west,font=\small\bfseries] at (0,0)
  {(a) Perturbed Block MDP};
\node[anchor=west,font=\small\bfseries] at (7.65,0)
  {(b) Observation Space};
\draw[aRule,line width=.45pt] (7.15,-.05) -- (7.15,-4.10);

\node[state] (aPrev) at (.55,-.95) {$s_{h-1}$};
\node[state,fill=aInk!4] (aState) at (3.5,-.95) {$s_h$};
\node[state] (aNext) at (6.45,-.95) {$s_{h+1}$};
\draw[edge] (aPrev) -- node[above=3pt]
  {$\bbP_h$} (aState);
\draw[edge] (aState) -- node[above=3pt]
  {$\bbP_{h+1}$} (aNext);

\node[branch,draw=aClean,fill=aClean!5,text=aClean] (aBlock) at (1.85,-2.45)
  {Block emission\\$\BOT_h(\cdot\mid s_h)$};
\node[branch,draw=aNoise,fill=aNoise!5,text=aNoise] (aError) at (5.15,-2.45)
  {Arbitrary emission\\$E_h(\cdot\mid s_h)$};
\draw[edge,draw=aClean] (aState.south west) --
  node[above left=1pt,text=aClean] {$1-\delta$} (aBlock.north);
\draw[noise edge] (aState.south east) --
  node[above right=1pt,text=aNoise] {$\delta$} (aError.north);
\node[state] (aObs) at (3.5,-3.70) {$x_h$};
\draw[edge,draw=aClean] (aBlock.south) -- (aObs.north west);
\draw[noise edge] (aError.south) -- (aObs.north east);

\begin{scope}[yshift=.2cm]
\draw[draw=aRule,line width=.55pt,rounded corners=3pt,fill=white]
  (7.78,-1.02) rectangle (14.93,-3.98);
\node[anchor=north east,text=aMuted,font=\small] at (14.78,-1.12) {$\cX_h$};

\begin{scope}[yshift=.35cm]
\path[draw=aMuted!48,fill=aMuted!5,line width=.65pt]
  (8.22,-2.48)
  .. controls (8.17,-1.95) and (8.51,-1.69) .. (9.05,-1.69)
  .. controls (9.63,-1.69) and (9.88,-2.04) .. (9.84,-2.62)
  .. controls (9.80,-3.19) and (9.34,-3.45) .. (8.78,-3.30)
  .. controls (8.40,-3.20) and (8.25,-2.91) .. (8.22,-2.48) -- cycle;
\node[text=aMuted] at (9.02,-1.98) {$\mathcal B_1$};
\foreach \x/\y in {8.62/-2.46,9.25/-2.65,8.91/-3.02,9.48/-2.29}{
  \fill[aMuted!65] (\x,\y) circle (1.6pt);
}
\end{scope}
\begin{scope}[yshift=-.30cm]
\path[draw=aClean!85,fill=aClean!7,line width=.85pt]
  (10.44,-2.27)
  .. controls (10.46,-1.70) and (10.88,-1.48) .. (11.43,-1.55)
  .. controls (11.99,-1.62) and (12.24,-2.02) .. (12.15,-2.60)
  .. controls (12.08,-3.17) and (11.70,-3.44) .. (11.14,-3.33)
  .. controls (10.61,-3.23) and (10.42,-2.82) .. (10.44,-2.27) -- cycle;
\node[text=aClean] at (11.30,-1.87) {$\mathcal B_2$};
\foreach \x/\y in {10.85/-2.34,11.28/-2.52,11.71/-2.30,11.83/-2.84,11.04/-3.00}{
  \fill[aClean] (\x,\y) circle (1.8pt);
}
\end{scope}
\begin{scope}[yshift=.30cm]
\path[draw=aMuted!48,fill=aMuted!5,line width=.65pt]
  (12.87,-2.36)
  .. controls (12.86,-1.92) and (13.22,-1.77) .. (13.68,-1.80)
  .. controls (14.18,-1.83) and (14.43,-2.18) .. (14.37,-2.74)
  .. controls (14.32,-3.25) and (13.95,-3.50) .. (13.43,-3.35)
  .. controls (12.98,-3.23) and (12.86,-2.81) .. (12.87,-2.36) -- cycle;
\node[text=aMuted] at (13.62,-2.09) {$\mathcal B_3$};
\foreach \x/\y in {13.23/-2.51,13.91/-2.54,13.60/-2.94,14.02/-3.05}{
  \fill[aMuted!65] (\x,\y) circle (1.6pt);
}
\end{scope}

\foreach \x/\y in {9.25/-2.30,11.28/-2.82,12.48/-3.56}{
  \draw[aNoise,line width=.85pt] (\x,\y) circle (4.2pt);
}

\end{scope}
\end{tikzpicture}
}
\endgroup
\caption{A $\delta$-perturbed Block MDP. \textbf{(a)} Latent dynamics and emission distribution. \textbf{(b)} The block structure
$\mathcal B_i=\operatorname{supp}\!\bigl(\BOT_h(\cdot\mid s^{(i)})\bigr)$
are disjoint (three shown). For $s_h=s^{(2)}$, clean emissions lie in
$\mathcal B_2$; orange rings mark possible arbitrary emissions within any
block or outside their union.}
\label{fig:perturbed-block-mdp}
\end{figure}

\Cref{fig:perturbed-block-mdp} illustrates the model. A simple example is the \emph{noisy sensor} model where $\cS = \cX$ and the true state is observed with probability at least $1-\delta$. Later, we will examine the empirical validity of this model; for now we study its theoretical implications. Below, we prove that for $\delta<1/8$, the belief contraction error has a geometric upper bound with factor $\sqrt{8\delta}$ per additional frame.

\begin{theorem}[See \cref{thm:belief-contraction}]
\label{thm:belief-contraction-main}
Suppose that $\cP$ is a $\delta$-perturbed Block MDP, with $0\leq\delta\leq1$.
Fix an executable policy $\pi$, indices $1\leq h-L<h\leq H$, and a
positive-probability partial history
$\tau=(x_{1:h-L},a_{1:h-L-1})$ under $\pi$.
Set $b=\bel_{h-L}(\tau)$ and let $\cD\in\Delta(\MS_{h-L})$ satisfy
$b\ll\cD$. Then
\begin{align*}
 &\EE^\pi\!\left[
 \norm{\bel_h(x_{1:h},a_{1:h-1})
 -\belapx_h(x_{h-L+1:h},a_{h-L:h-1};\cD)}_1
 \,\middle|\,\tau\right] \\
 &\qquad\leq
 2\min\left\{1,(8\delta)^{L/2}
 \sqrt{\TV(b,\cD)\norm{b/\cD}_\infty}\right\}.
\end{align*}
In particular,
\[
 \epsilon^{\contraction}_h(\pi;L)
 \leq 2\min\{1,\sqrt S\,(8\delta)^{L/2}\}.
\]
\end{theorem}

While prior belief contraction results \citep{golowich2023planning} apply to this model for $\delta<1/2$, they only yield contraction by $1 - (1-2\delta)^4/C$ per frame, for a large constant $C>1$ (\cref{remark:perturbed-block-golowich}), and so are vacuous for $L = \littleoh(\log S)$, even in the regime $\delta \ll 1$ (i.e. low observation noise). \cref{thm:belief-contraction-main} remedies this limitation; e.g. for $\delta = 1/S$ it yields $\epsilon^{\contraction}_h(\pi;L) \leq \bigoh(1/S)$ with only $L = \bigoh(1)$. To prove \cref{thm:belief-contraction-main}, one might hope that each new observation contracts the $\TV$-distance by $\poly(\delta)$ in expectation. This is false (\cref{example:tv-not-contract}). Instead, \cref{lemma:dpr-bayes-update} proves contraction in expectation of the square root of the product of $\TV$-distance and the maximum density ratio; data processing under transitions and the tower property then give the theorem.

Heuristically, \cref{thm:belief-contraction-main} suggests that standard RL with $L$-step frame-stacking should progressively improve as $L$ increases. Formally, the arbitrary-prior contraction bound in \cref{thm:belief-contraction} and the learning guarantee in \cref{thm:golowich} imply the following end-to-end learning guarantee for \texttt{BaSeCAMP} with empirical reward estimates, which does not use latent state information:

\begin{corollary}[See \cref{cor:golowich-perturbed-mdp}]
\label{cor:golowich-perturbed-mdp-main}
For any $\delta,\beta\in(0,1/3)$ and $L\in\NN$, a
$\delta$-perturbed Block MDP admits a reinforcement learning algorithm
that returns a policy $\pihat$ satisfying, with probability at least
$1-\beta$,
\[
 J(\pistar)-J(\pihat)
 \leq\min\{1,(8\delta)^{L/3}\poly(S,X,H)\},
\]
in time $(XA/\delta)^{\bigoh(L)}\poly(H,S,\log(1/\beta))$.
\end{corollary}

From a theoretical view, it remains to understand the decodability error for the perturbed Block MDP. As we will show, this qualitatively depends on the stochasticity of the transition dynamics. \loose

\section{Distillation is Competitive for Deterministic Dynamics}
\label{sec:misspecified}
In some environments, it is reasonable to assume that the latent transition dynamics are \emph{deterministic} (e.g., if the dynamics are governed by simple Newtonian mechanics). Simulation benchmarks with this property include some Atari games as well as MuJoCo tasks. In this section, we theoretically and empirically study the performance of expert distillation, versus standard RL with frame-stacking, in such environments (with deterministic latent transitions, but stochastic initial state and observations).

\begin{figure*}[t]
    \centering
    \includegraphics[width=0.32\textwidth]{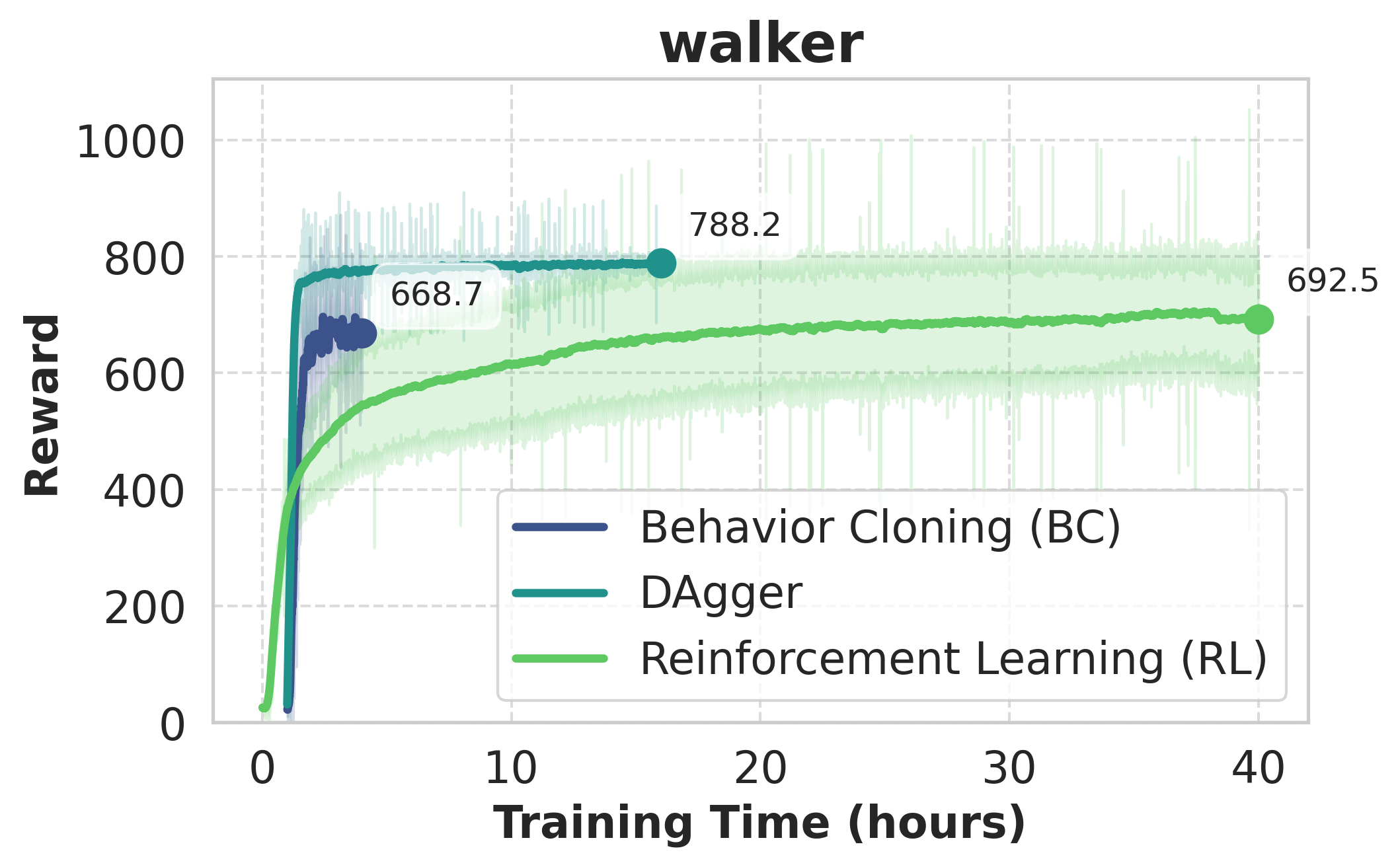}
    \includegraphics[width=0.32\textwidth]{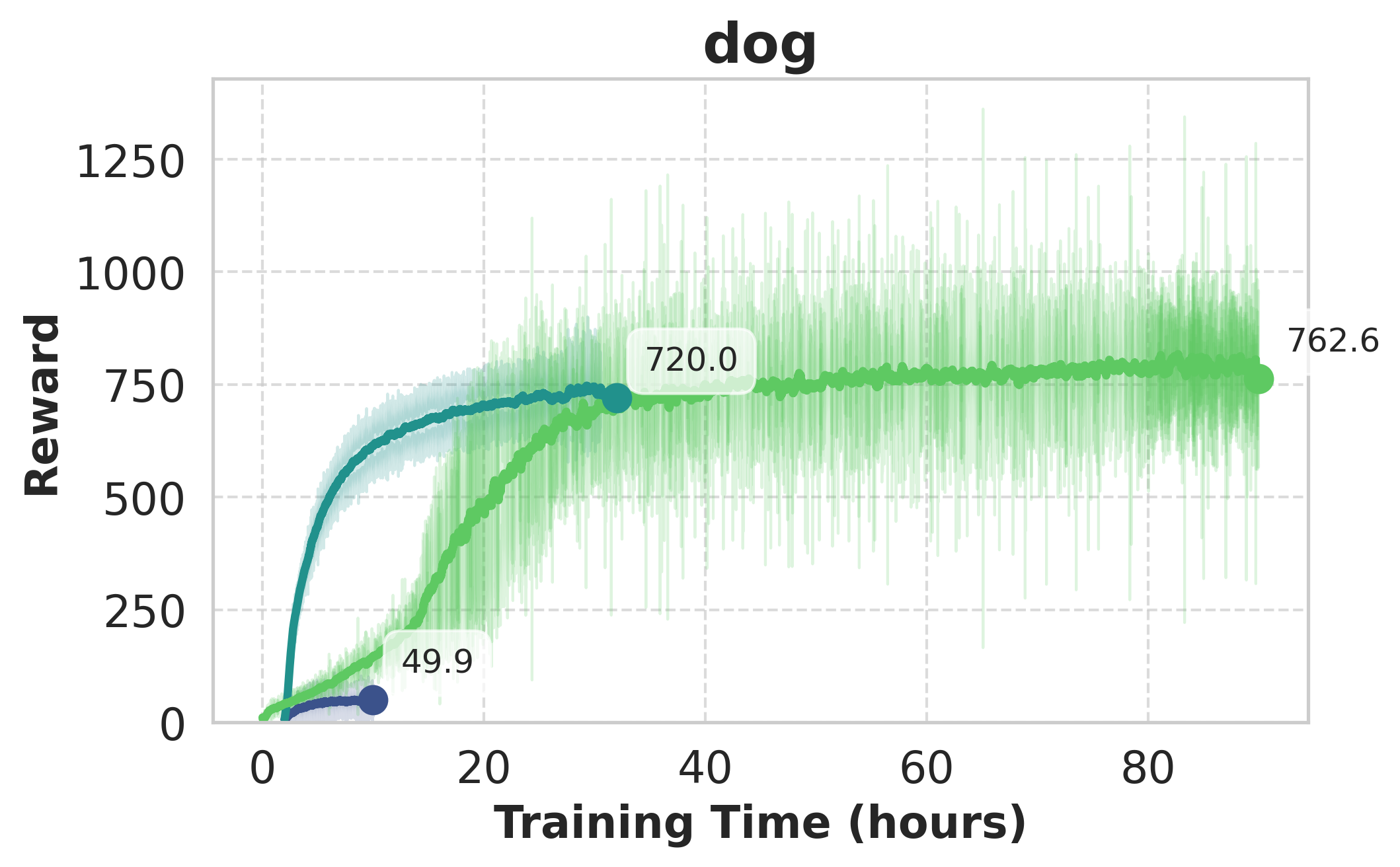}
    \includegraphics[width=0.32\textwidth]{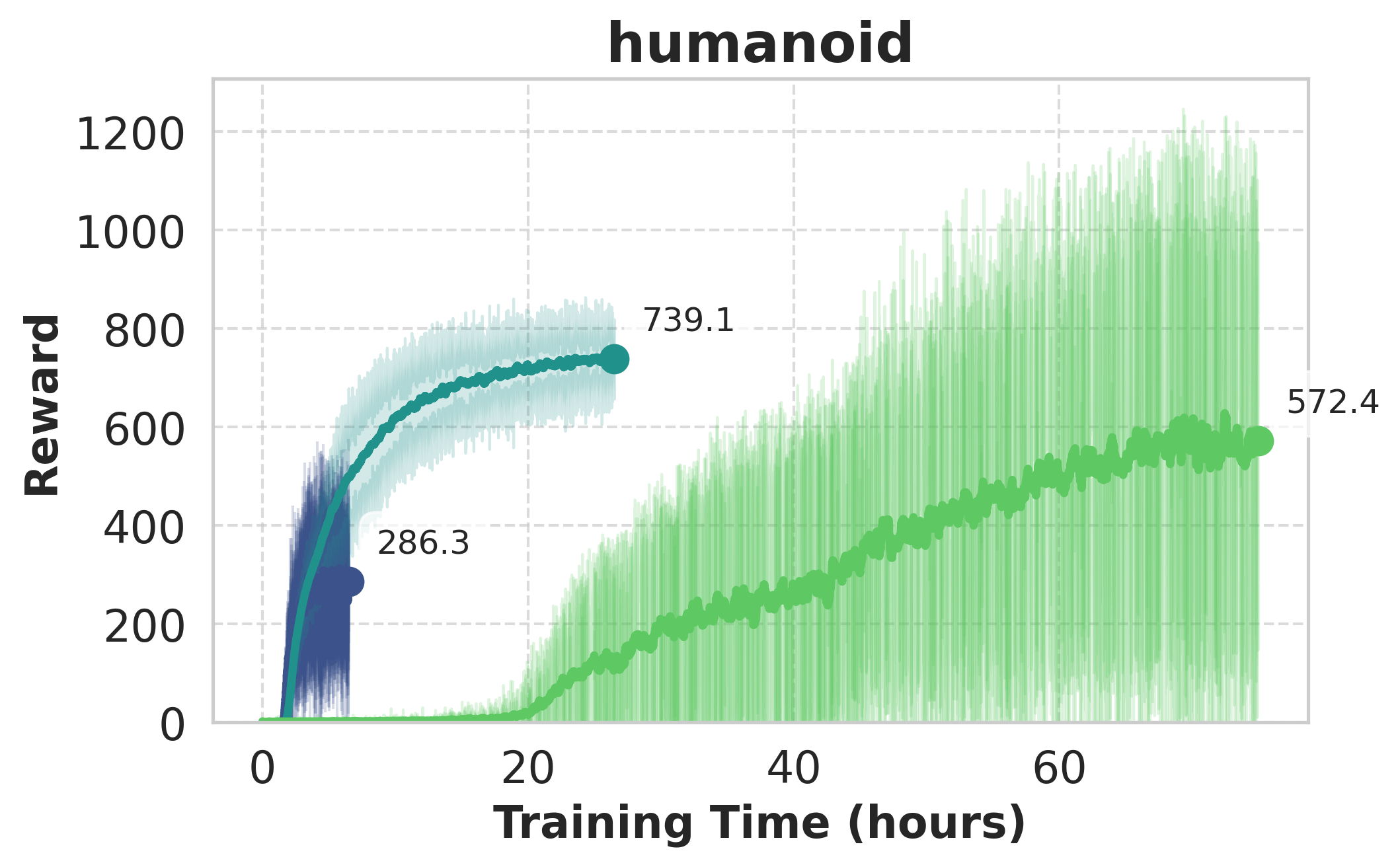}
    \caption{The performance of (offline/online) expert distillation and RL with respect to wall-clock time. We repeat each experiment 5 times and plot the mean and standard deviation. For the time complexity of BC, we include the data collection time, and amortize it over the training steps. For both BC and $\Dagger$, we include the time to train the latent expert (also amortized). \loose}
    \label{fig:det-il-rl}
\end{figure*}

\subsection{Theoretical Analysis under Deterministic Dynamics} Below, we show that for perturbed Block MDPs with deterministic dynamics, the decodability error decays exponentially as the step $h \in [H]$ increases. Intuitively, each observation concentrates the true belief state further, and the deterministic transitions cannot ``spread out'' the belief state. While this intuition is not quite rigorous, it can be proven that \emph{most} observations concentrate the belief state; the result follows from an appropriate martingale analysis (\cref{lemma:martingale-bound-improved}).

\Needspace{7\baselineskip}
\begin{proposition}[See \cref{prop:vinf-decay}]
\label{prop:det-vinf-decay-main}
There is a universal constant $C_{\ref{prop:det-vinf-decay-main}}>1$ such that the following holds.
Suppose that $\cP$ is a $\delta$-perturbed Block MDP with deterministic
transitions. For any executable policy $\pi$ and index $h\in[H]$,
\[
 \epsilon^{\decodability}_h(\pi)
 \leq\min\left\{\delta,(C_{\ref{prop:det-vinf-decay-main}}\delta)^{(h-1)/6}\right\}.
\]
\end{proposition}

From \cref{lemma:tv-beltil-to-latent-main}, the ``ideal'' distillation of a latent expert $\pilat$ is $\pi^{\imitation} := \pi^{\latent} \circ \belief$, i.e., given any history, query the latent expert based on the true belief. Combining \pref{lemma:tv-beltil-to-latent-main} and \pref{prop:det-vinf-decay-main} immediately yields a strong, horizon-independent guarantee for this policy: if $\pilat$ is the optimal latent policy, then
\begin{align*}
J(\pi^\star)-J(\pi^{\imitation})
&\leq J(\pilat)-J(\pi^{\imitation})\\
&\leq\min\left\{1,2\sum_{h=1}^H
 \min\left(\delta,(C_{\ref{prop:det-vinf-decay-main}}\delta)^{(h-1)/6}\right)\right\}
\leq\bigoh(\delta),
\end{align*}
where the first inequality is by \pref{lemma:mdp_upper_pomdp}.

 Of course, exactly learning the true belief state may be unrealistic, since this would require conditioning on the entire history. However, we can prove that (a slight modification of) the $\forward$ algorithm \citep{ross2010efficient} (the non-stationary version of $\Dagger$) on $L$-step executable policies learns the following approximation of $\pi^{\imitation}$,\footnote{Note that behavior cloning will not learn the same policy, due to the latching effect \cite{swamy2022sequence} (i.e. conditioning on past actions of the latent expert). It may nevertheless achieve the same regret bound as $\forward$ in our setting: theoretically separating these algorithms likely requires assuming e.g. recoverability \citep{foster2024is}.
} in the infinite-sample limit, on windows of positive probability under its sampling policy (with arbitrary choices elsewhere):
    \[\pi^{\forward}_h(\cdot \mid{} x_{h-L+1:h}, a_{h-L:h-1}) =\begin{cases}
    \pilat_h \circ \belapx_h(x_{h-L+1:h},a_{h-L:h-1}; d^{\pi^{\forward}}_{h-L}) & \text{ if } h > L \\
    \pilat_h \circ \bel_h(x_{1:h},a_{1:h-1}) &
    \text{ otherwise}
    \end{cases}
    \]
See \pref{sec:forward-population} for the algorithm and proof. Applying this derivation to \cref{lemma:tv-beltil-to-latent-main}, then using \cref{prop:det-vinf-decay-main} to bound the decodability error and \cref{thm:belief-contraction-main} to bound the error in approximate beliefs, gives the following guarantee for expert distillation under deterministic latent dynamics:
\begin{theorem}[See \cref{thm:perturbed-mdp-il-guarantee}]
\label{thm:perturbed-mdp-il-guarantee-main}
Suppose that the POMDP $\cP$ is a $\delta$-perturbed Block MDP with deterministic transitions, and fix an integer $L\geq1$. Let $\pilat\in\Pilat$ be the optimal latent policy, and let $\pi^{\forward}$ be the policy computed by $\forward$ with $L$-step random actions, policy class $\Pi^L$, and expert $\pilat$, in the infinite-sample limit. Then
\begin{align*}
J(\pistar)-J(\pi^{\forward})
&\leq J(\pilat)-J(\pi^{\forward})\\
&\leq \TV(\bbP^{\pilat},\bbP^{\pi^{\forward}})
\leq \bigoh(\delta)+2H\sqrt S\,(8\delta)^{L/2}.
\end{align*}
\end{theorem}

\paragraph{Comparison with RL} While \cref{thm:perturbed-mdp-il-guarantee-main} is presented in the infinite-sample limit, the effective sample complexity is only $\approx(XA)^{\bigoh(L)}$, since the optimization is over $L$-step executable policies. More concretely, up to additional error $\epopt$, the above guarantee can be achieved by the same algorithm with only $\poly((AX)^L,H,\epopt^{-1})$ time and samples (\cref{thm:perturbed-mdp-il-finite-sample}).
Thus, the guarantee for $\forward$ qualitatively matches the guarantee for RL (\cref{cor:golowich-perturbed-mdp-main}), aside from the additional horizon-independent term of $\bigoh(\delta)$ incurred above (due to poor decodability in initial steps).

\begin{figure*}[t]
    \centering
    \includegraphics[width=0.3\textwidth]{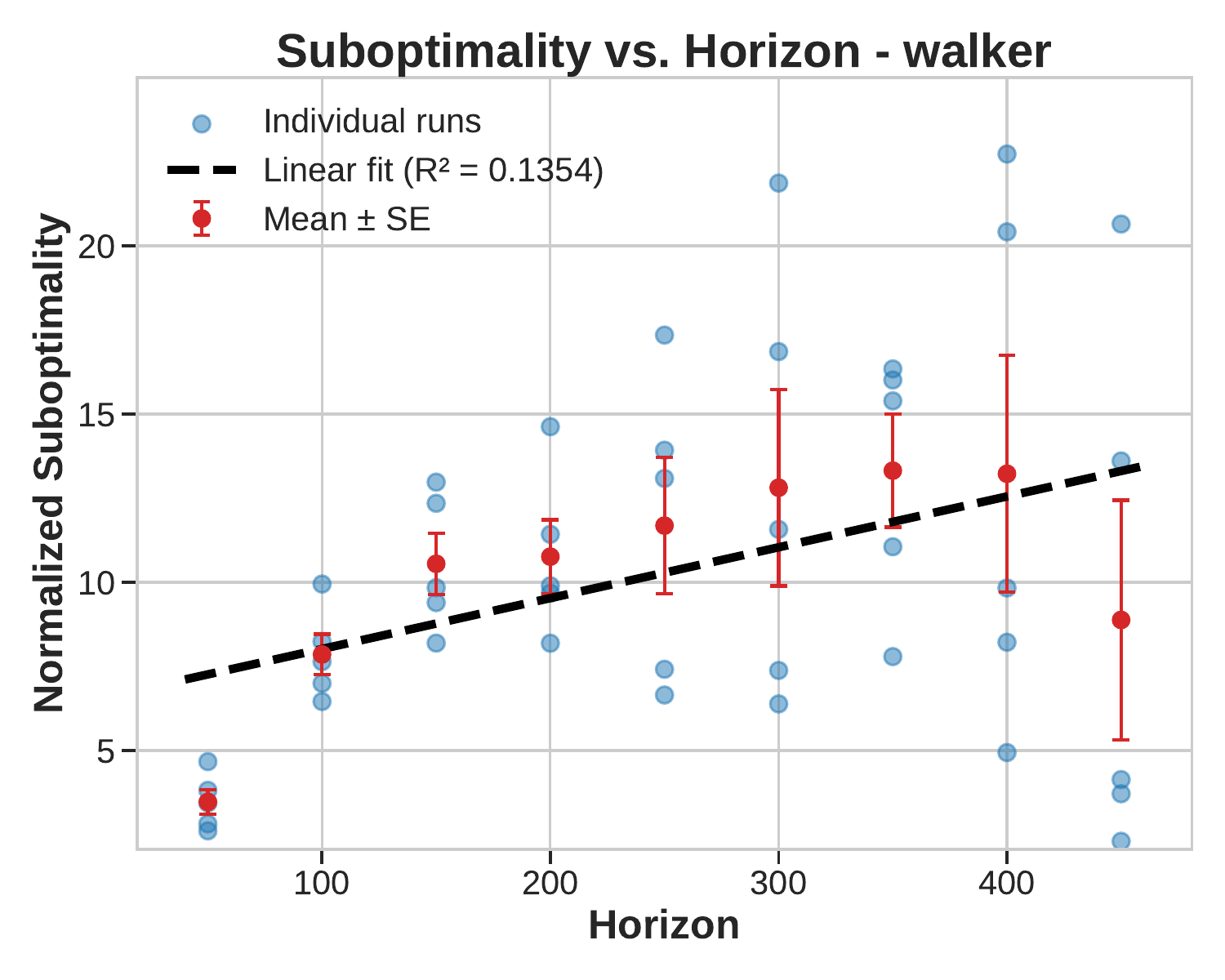}
    \includegraphics[width=0.3\textwidth]{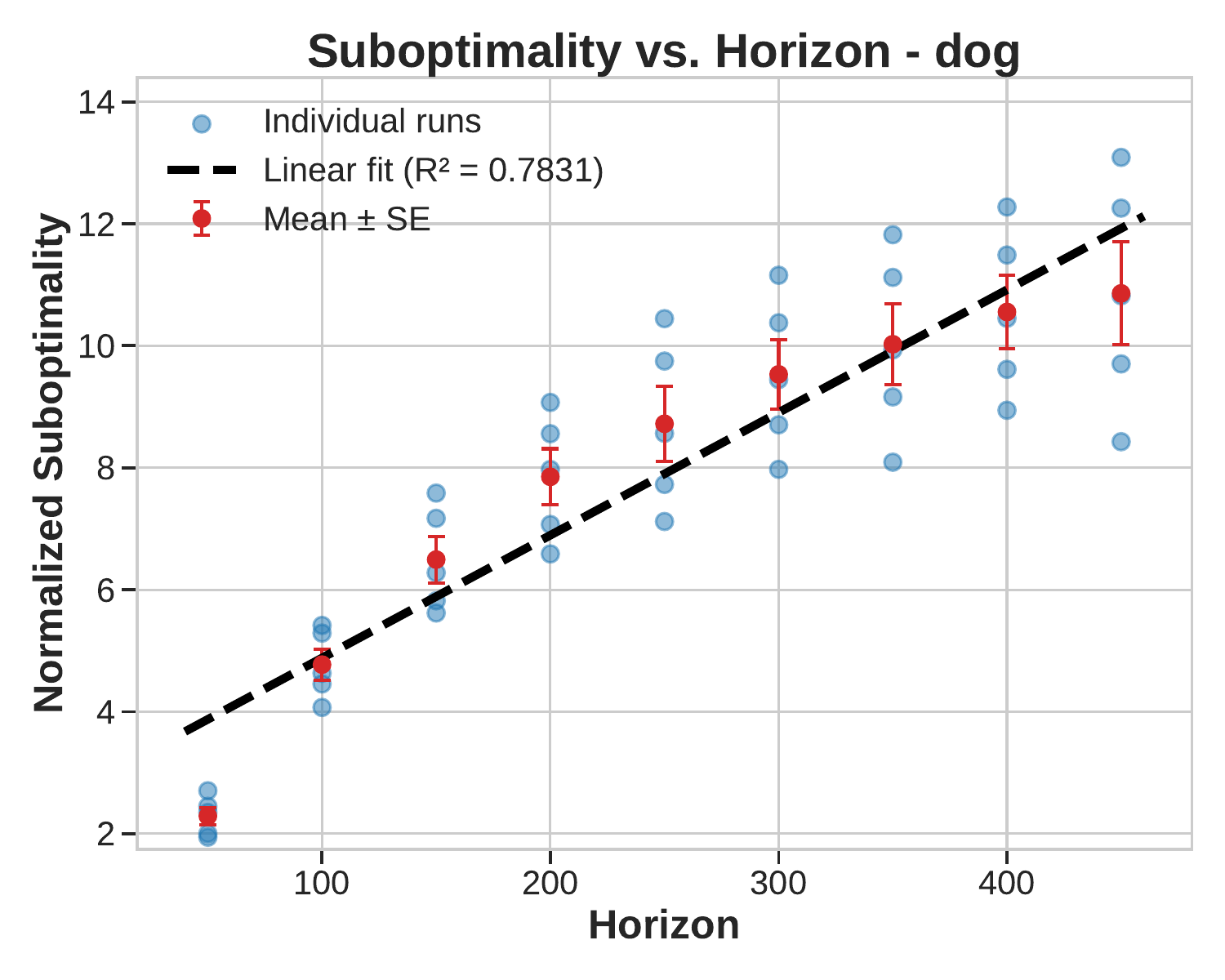}
    \includegraphics[width=0.3\textwidth]{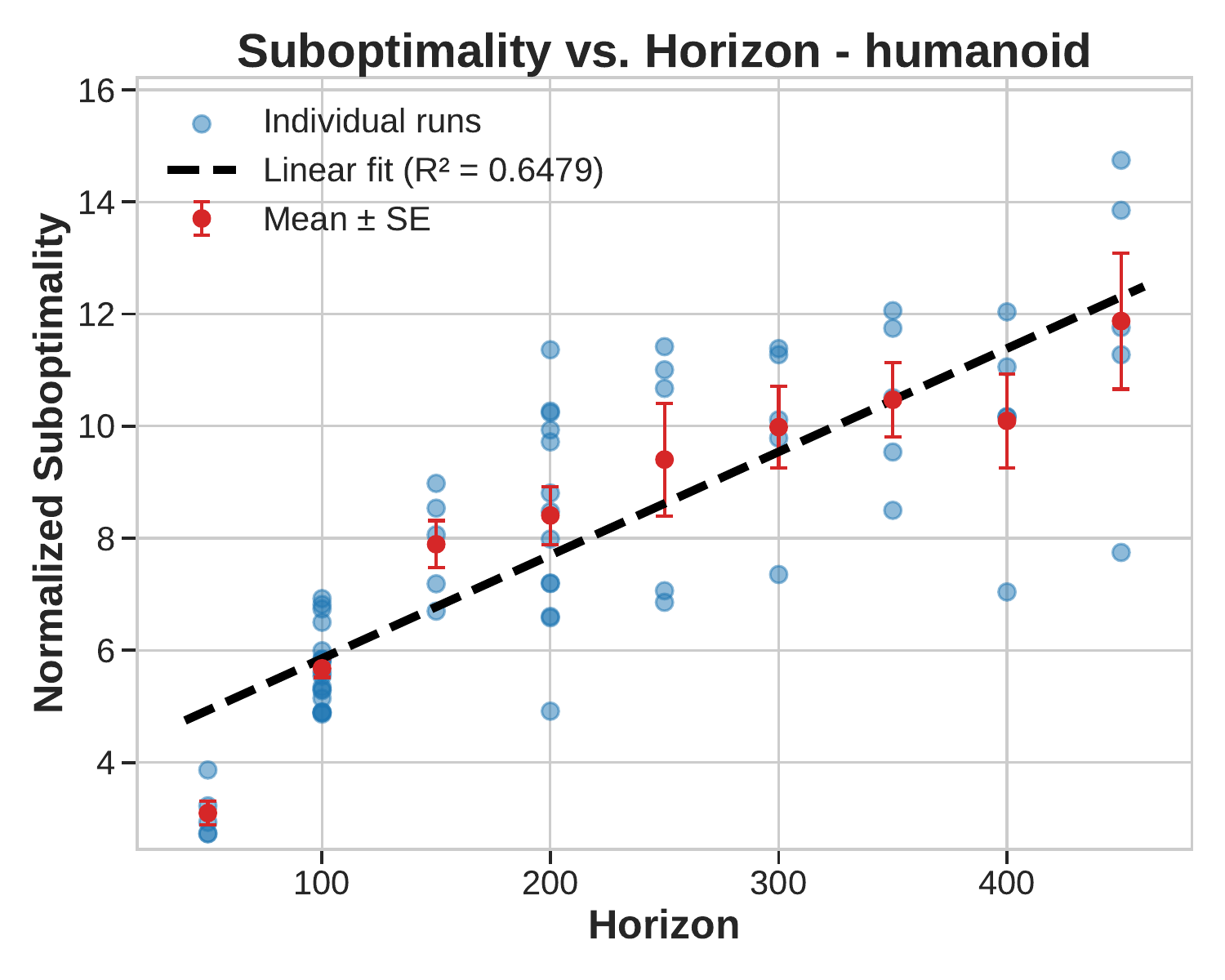}
    \includegraphics[width=0.3\textwidth]{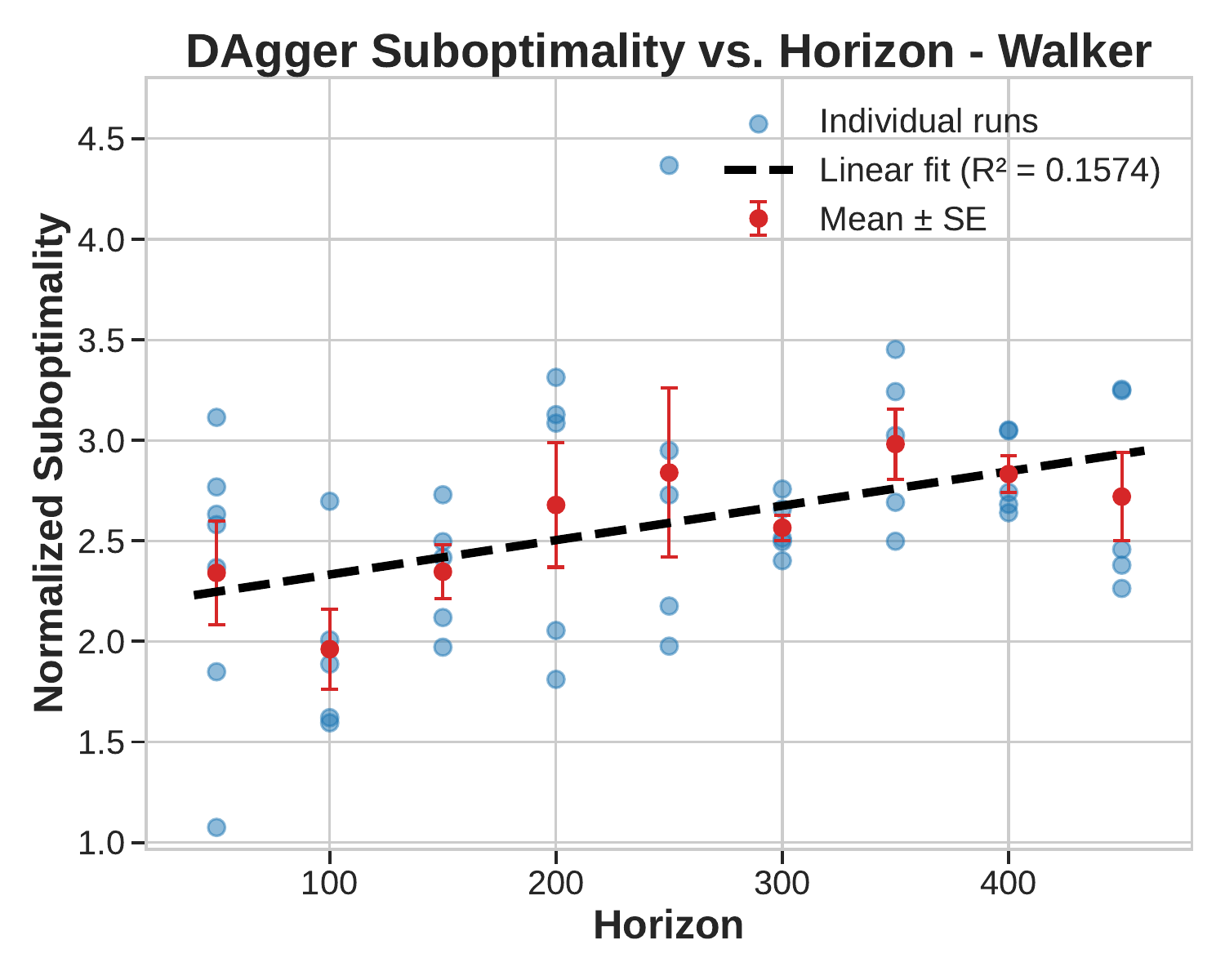}
    \includegraphics[width=0.3\textwidth]{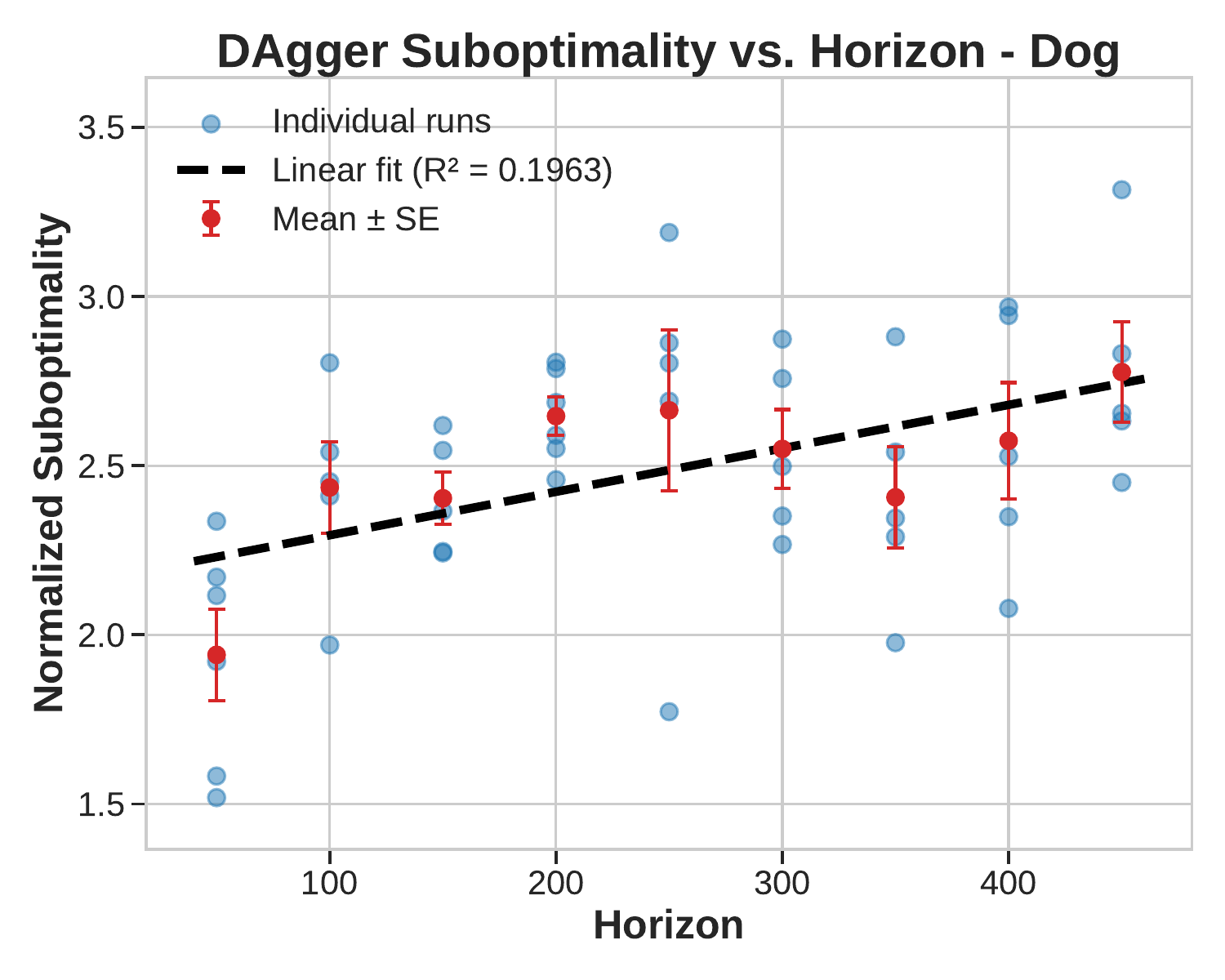}
    \includegraphics[width=0.3\textwidth]{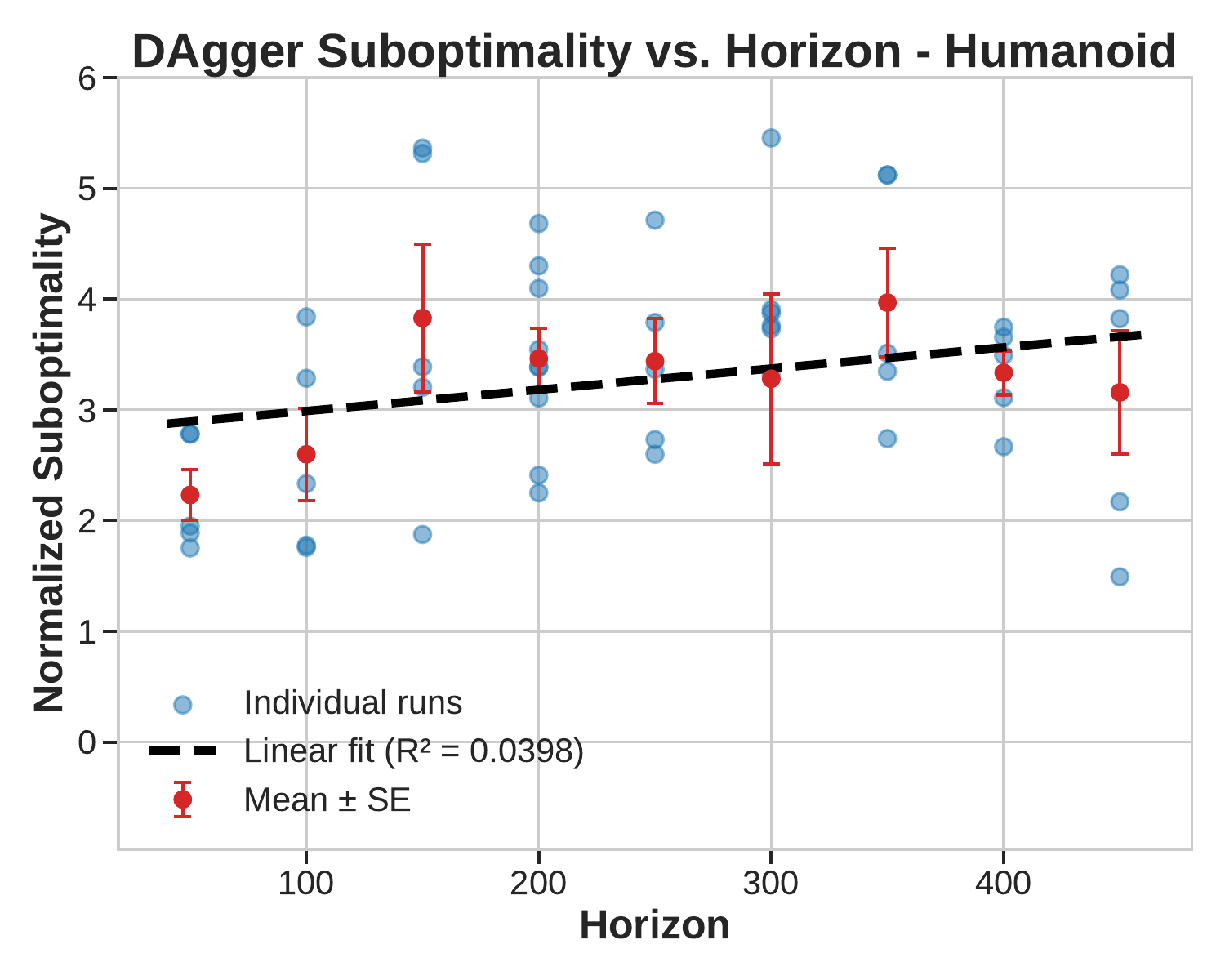}
    \caption{The normalized suboptimality of the expert distillation algorithms (top: behavior cloning; bottom: $\Dagger$) with respect to the horizon. We repeat 5 runs for each horizon and task, and perform linear regression on the results from each task. Note that the trajectory rewards for this plot have been normalized by horizon (and by action-prediction error), so linear scaling indicates compounding errors.}
    \label{fig:horizon}
\end{figure*}

\subsection{Empirical Analysis under Deterministic Dynamics}

\pref{thm:perturbed-mdp-il-guarantee-main} gives a strong performance guarantee for expert distillation under deterministic latent dynamics, nearly matching that of RL. This suggests that expert distillation may be preferred over standard RL due to its (practical) efficiency. Also, \pref{thm:perturbed-mdp-il-guarantee-main} suggests that error may compound with the horizon $H$. However, the result is only an upper bound, and only for a stylized setting. We now investigate whether these two theoretical implications hold up empirically.

\paragraph{Expert distillation outperforms RL under deterministic dynamics} In this experiment, we compare the (a) asymptotic performance and (b) computational efficiency of expert distillation and standard RL. We train each method until convergence, and we plot the episodic return with respect to the wall clock time in \pref{fig:det-il-rl}. We see that offline expert distillation (i.e., behavior cloning) is competitive in easier tasks such as \texttt{walker}, but is suboptimal in harder tasks such as \texttt{humanoid} and \texttt{dog}. However, online imitation learning (i.e., $\Dagger$) is able to achieve the best performance in all tasks, and with better computational efficiency (i.e., faster convergence) than RL. This supports our theory that under deterministic dynamics, expert distillation can be close to optimal.

\textbf{Empirical vignette: the source of error compounding?}
The horizon dependence of the error in imitation learning has received intensive empirical \citep{ross2010efficient,laskey2017dart,block2023butterfly} and theoretical \citep{rajaraman2020toward,foster2024is,rohatgi2025computational} study, both from the perspective of sample complexity \citep{rajaraman2020toward,foster2024is} and misspecification \citep{rohatgi2025computational}. It is widely believed that behavior cloning suffers error compounding over the horizon, which is avoided by online methods such as $\Dagger$ that are able to \emph{recover} from mistakes \citep{ross2014reinforcement,rajaraman2020toward}. Does this compounding manifest in expert distillation for POMDPs, and is the cause sampling error or misspecification?
In \pref{fig:horizon}, we vary the horizon $H \in [50,450]$, and measure the sub-optimality of offline and online expert distillation. We normalize rewards so that trajectory reward lies in $[0,1]$. We further normalize by mean action-prediction MSE (averaged over choice of $H$).
We see strong horizon dependence for behavior cloning (and weaker for $\Dagger$, likely due to recoverability). This contrasts with empirical results of \cite{foster2024is}: they perform \emph{well-specified} behavior cloning in similar tasks, and find little horizon dependence. Together, our results therefore suggest that misspecification, rather than sampling error, may be the more fundamental source of horizon dependence for behavior cloning.

\section{RL Outperforms Distillation for Stochastic Dynamics}
\label{sec:observable}
While deterministic dynamics are plausible in some applications, there are also many potential sources of stochasticity; in real-world robotics, stochasticity may be required to model e.g. internal motor noise or unknowable features of the external environment.
Some robotics simulators \citep{Makoviychuk2021IsaacGH} also have stochasticity arising from a PDE solver. How does the stochasticity of the environment affect the performance of expert distillation and RL?

\begin{figure*}[t]
\centering
\includegraphics[width=0.32\textwidth]{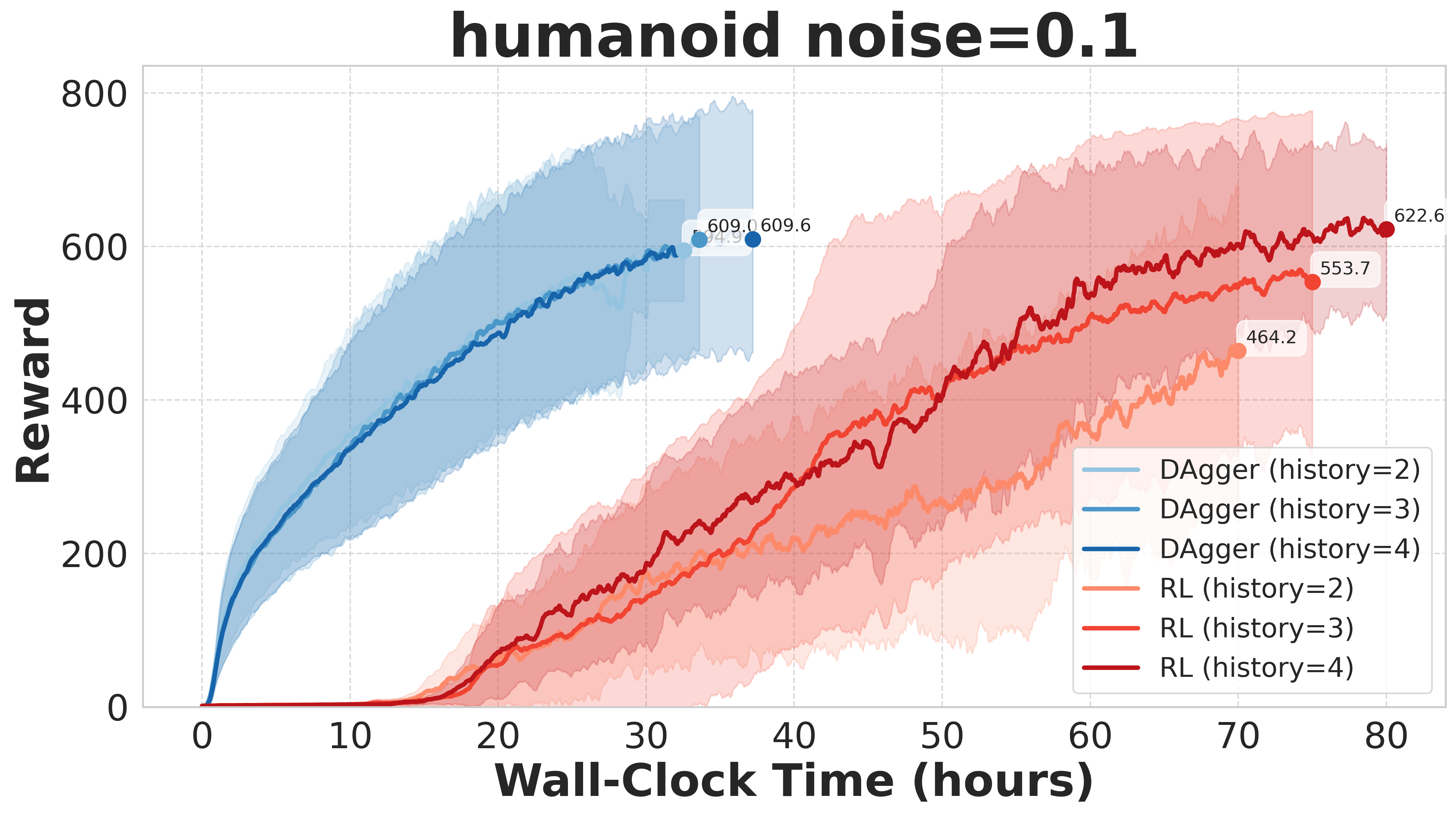}
\includegraphics[width=0.32\textwidth]{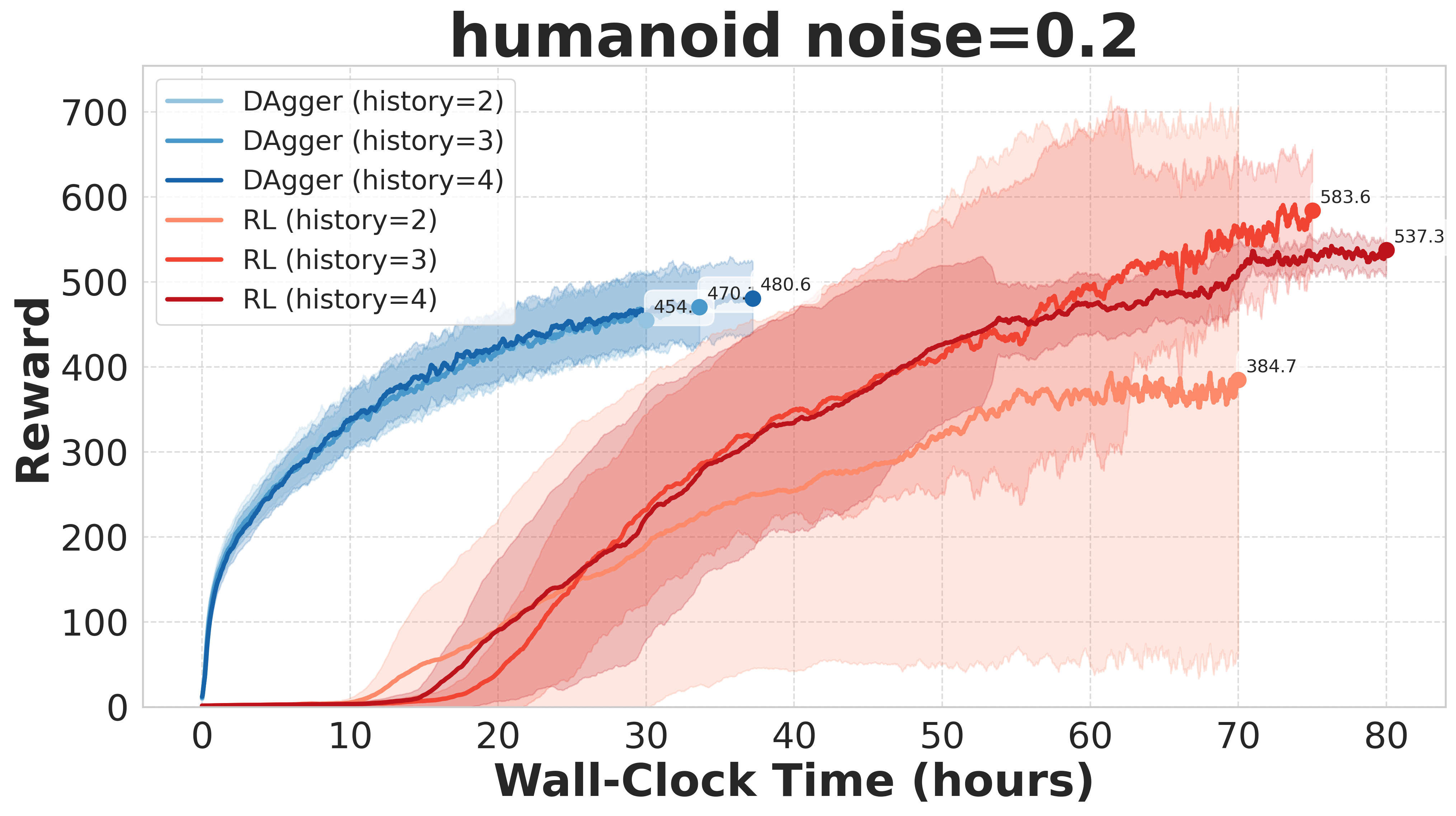}
\includegraphics[width=0.32\textwidth]{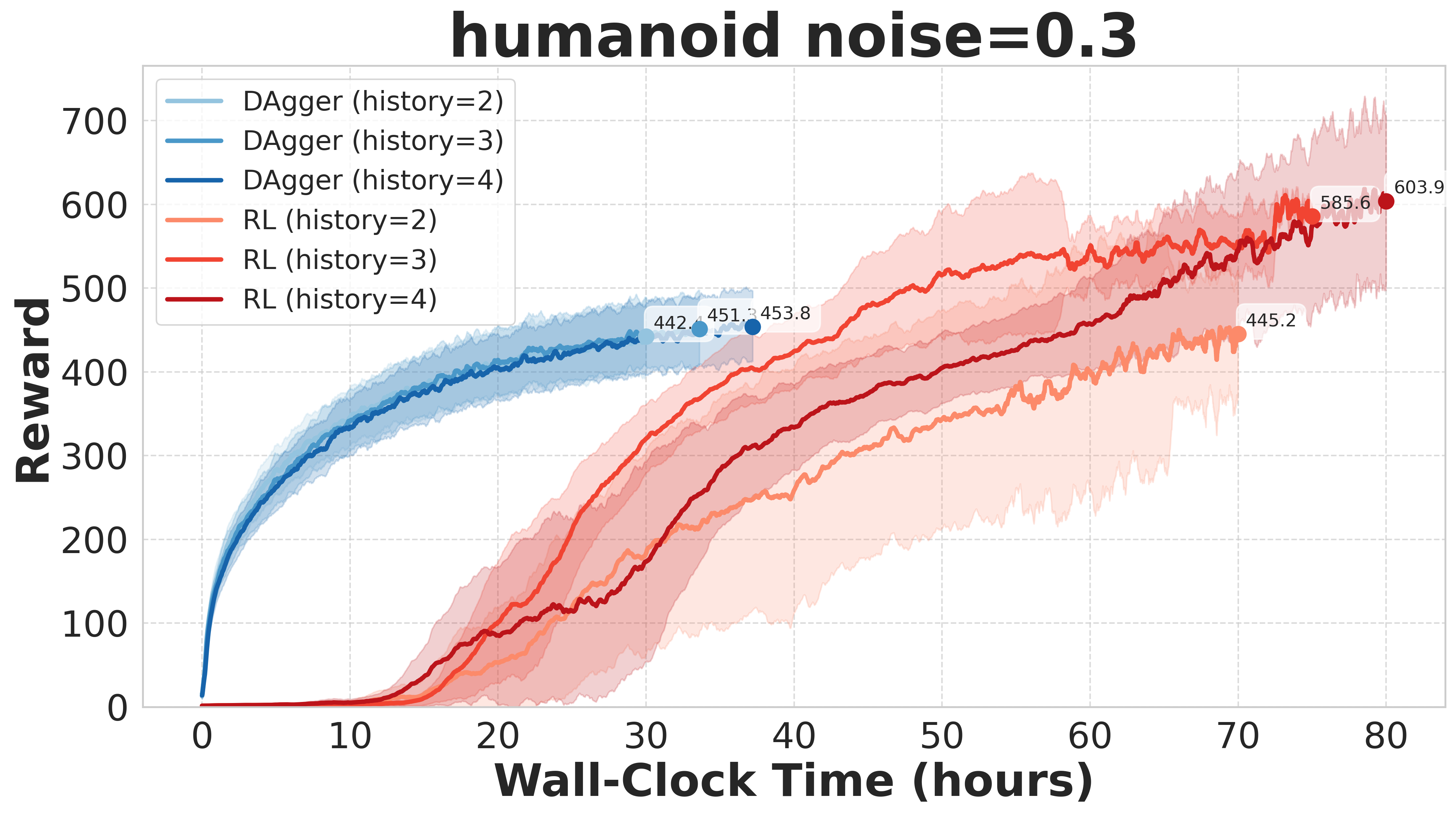}

\includegraphics[width=0.32\textwidth]{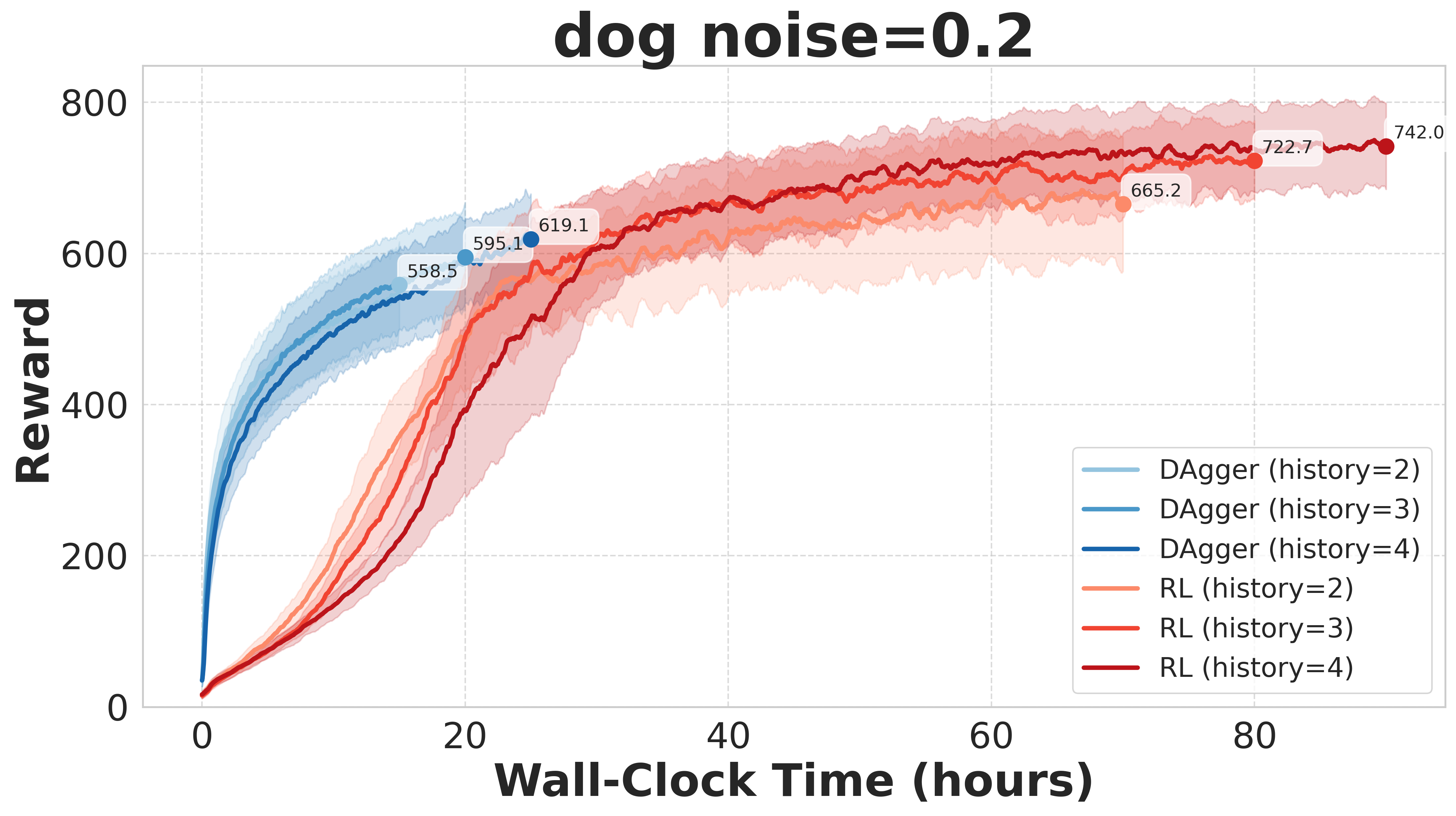}
\includegraphics[width=0.32\textwidth]{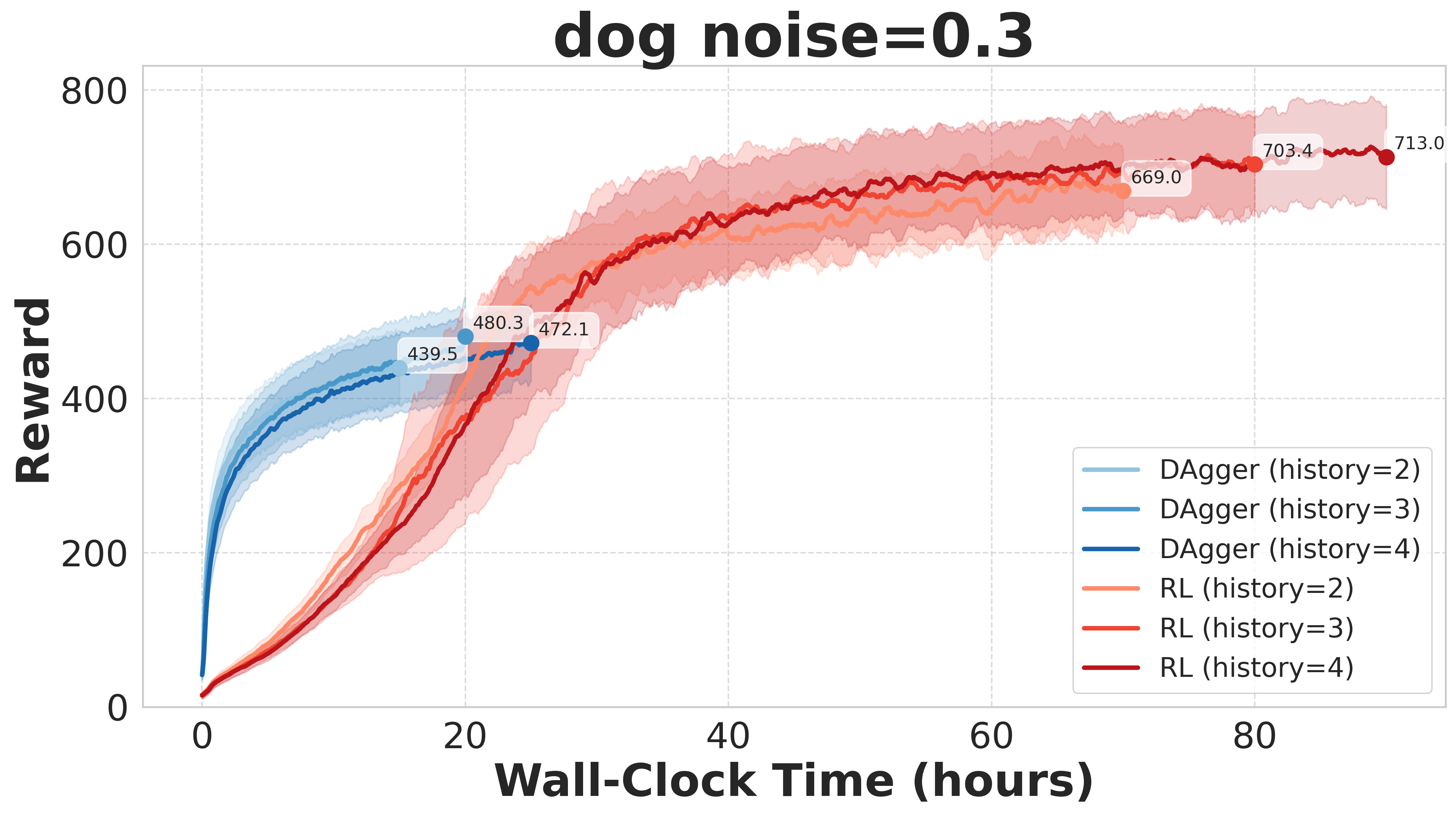}
\includegraphics[width=0.32\textwidth]{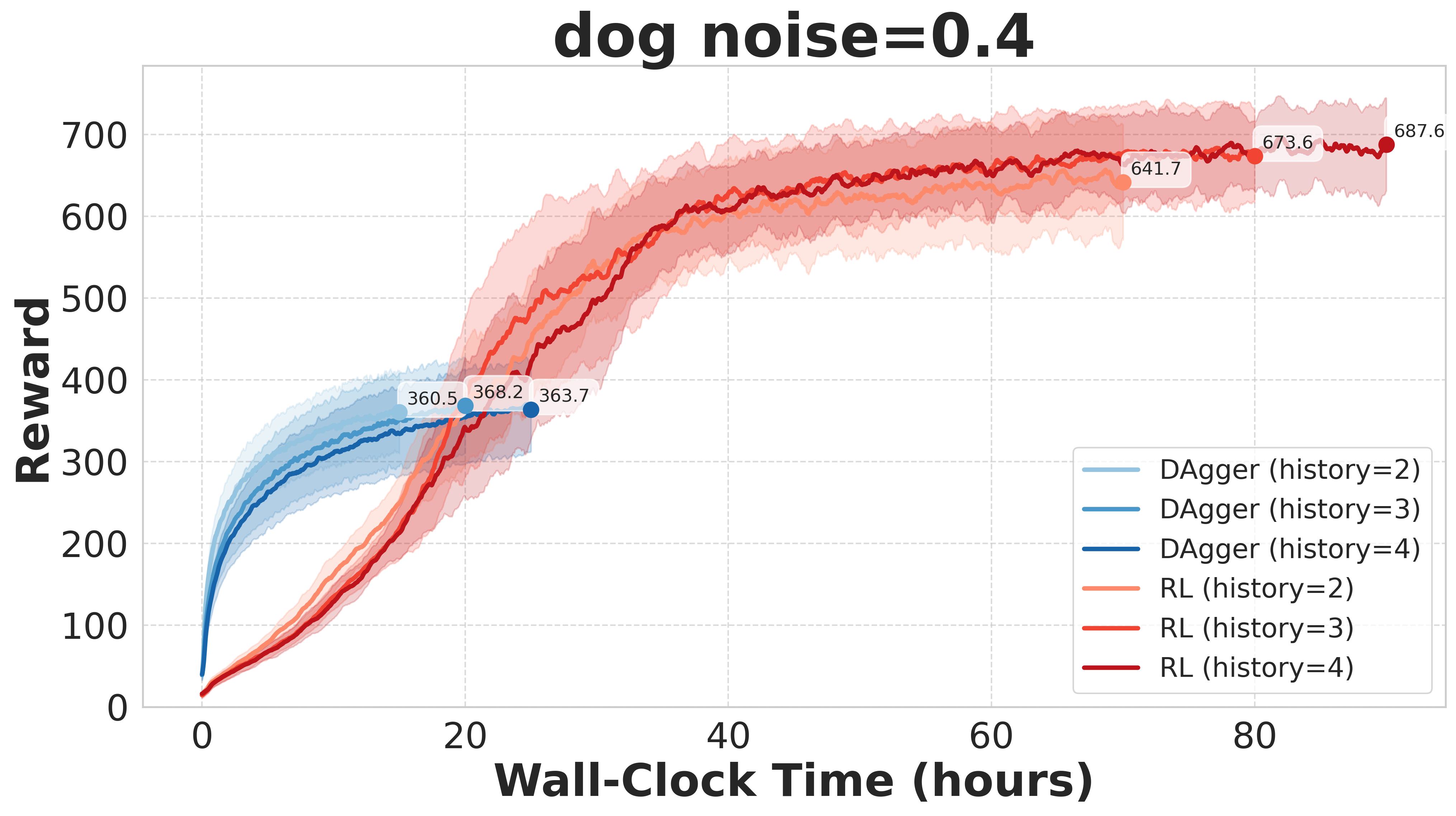}

\caption{Performance of $\Dagger$ and RL with different frame-stacks on \texttt{humanoid-walk} and \texttt{dog-walk} with motor noise. We repeat each experiment 5 times and plot the mean and standard deviation. Note that in general, the improvement of RL over $\Dagger$ increases with the motor noise.}\label{fig:stochastic}
\end{figure*}

\subsection{Theoretical Analysis under Stochastic Dynamics}

We show a negative result in the perturbed Block MDP model: for general dynamics, the misspecification of the optimal latent policy with respect to the class of $L$-step executable policies \emph{does not} necessarily decay as $L$ increases, in contrast with the case of deterministic dynamics (\cref{thm:perturbed-mdp-il-guarantee-main}).

\begin{proposition}[See \cref{prop:stoch-decodability-lb}]\label{prop:stoch-decodability-lb-main}
Let $\delta\in(0,1]$ and $H\in \NN$. There is a $\delta$-perturbed Block MDP $\cP$ with horizon $H$ such that for all $L \in [H]$, the optimal latent policy $\pilat$ satisfies the following bound, where $\Pi^L$ is the class of $L$-step executable policies:
\[\min_{\pi\in\Pi^L} \TV(\bbP^{\pilat}, \bbP^\pi) \geq \Omega(\min(1, \delta H)).\]
\end{proposition}

This result also highlights the difference between decodability error and belief contraction error, whose upper bound decays as $L$ increases for $\delta<1/8$, regardless of the transition dynamics (\cref{thm:belief-contraction-main}). The intuition for \cref{prop:stoch-decodability-lb-main} is simple: in the extreme case where the dynamics are \emph{uniformly mixing} at every step, prior observations yield no information about the current state, so the error $q=\min\{\delta,1/2\}$ incurred by decoding the current observation is irreducible. This decodability error compounds over timesteps, and means that executable policies are unable to simulate the latent policy that plays an action uniquely indexed by the latent state. In contrast, POMDPs with uniform mixing are \emph{easy} for standard RL, precisely because they reduce to $H$ independent horizon-$1$ subproblems.

\paragraph{Comparison with RL} The above result, compared with \cref{cor:golowich-perturbed-mdp-main}, suggests a potential empirical benefit of standard RL over expert distillation: the former may generically be able to trade increased computation (by increasing $L$) for improved performance (by mitigating observation noise), whereas the latter --- at least in the worst case --- incurs irreducible error due to stochasticity in the dynamics. To be sure, the uniformly-mixing construction from \cref{prop:stoch-decodability-lb-main} is practically unrealistic; nevertheless, below we verify that this benefit occurs in more realistic environments.

\subsection{Experimental Analysis under Stochastic Dynamics}\label{sec:exp-stochastic}
\mathchardef\mhyphen="2D

\paragraph{RL with more computation eventually outperforms distillation} To simulate a POMDP with stochastic latent dynamics, we apply motor noise in the \texttt{humanoid-walk} task. We add $0$-mean isotropic Gaussian noise with $\mathsf{std\mhyphen dev}\in\{0.1,0.2,0.3\}$ to each action. We compare $\Dagger$ and RL with frame-stack $L\in\{2,3,4\}$. We run each method until convergence (with the same number of episodes for all runs with fixed algorithm/noise level) and plot episodic return against wall-clock time (\pref{fig:stochastic}). We observe that expert distillation does not benefit from larger $L$, whereas the performance of RL sometimes benefits (at the cost of longer wall-clock time). These results are consistent with the qualitative distinction highlighted by our theory: RL outperforms expert distillation for higher noise levels in these experiments.

\arxiv{
    \begin{figure*}[t]
        \centering
        \includegraphics[width=0.32\textwidth]{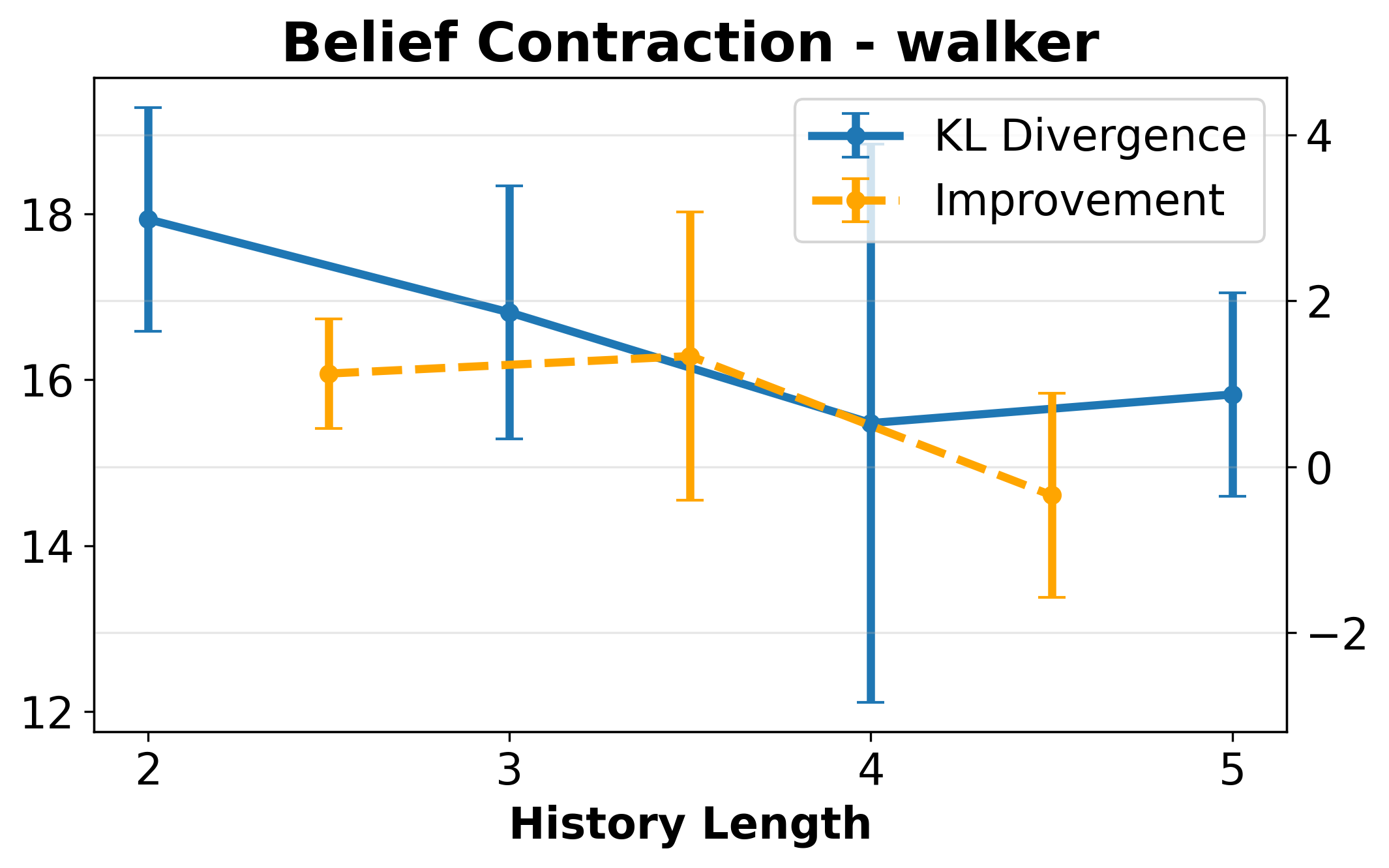}
        \includegraphics[width=0.32\textwidth]{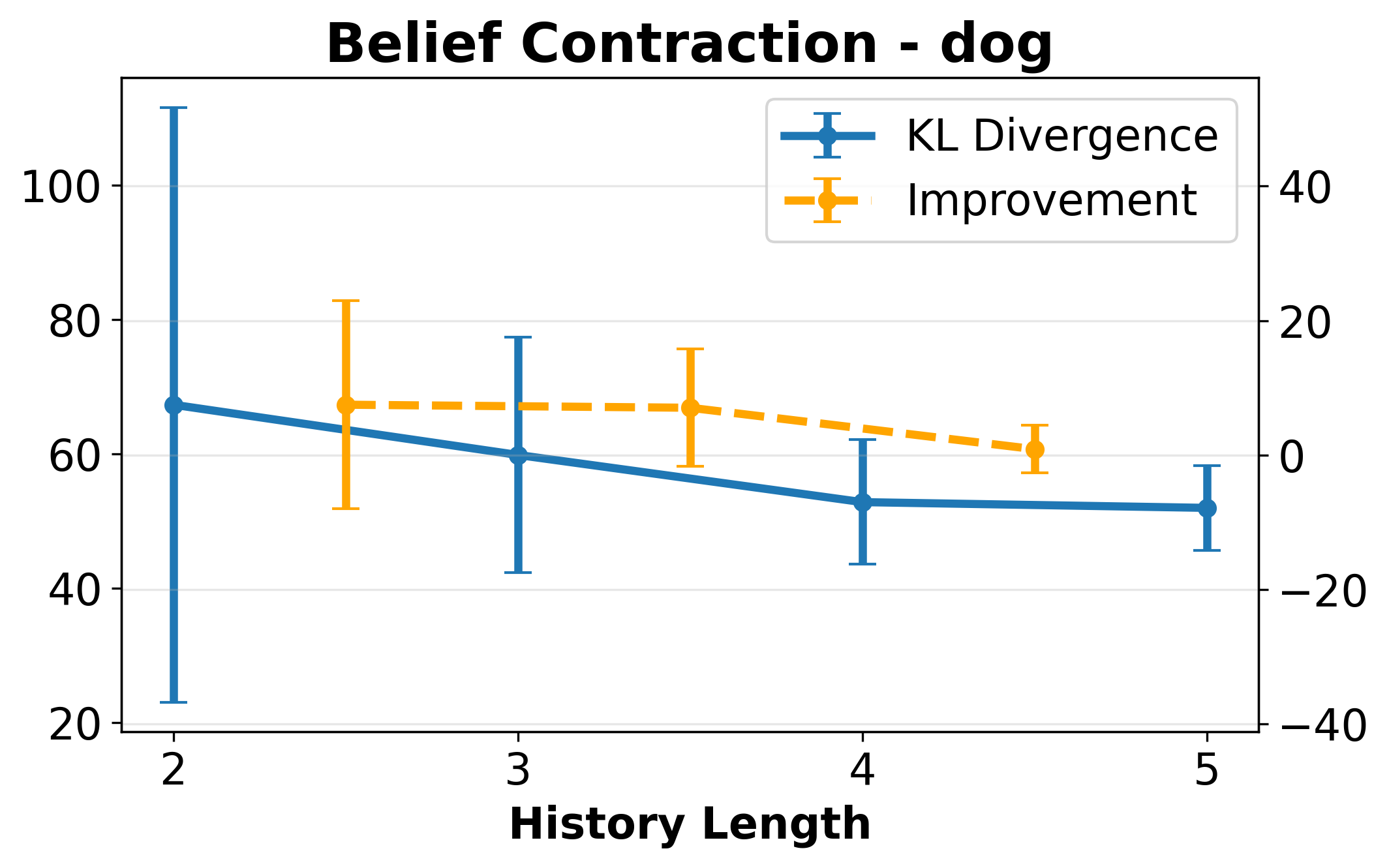}
        \includegraphics[width=0.32\textwidth]{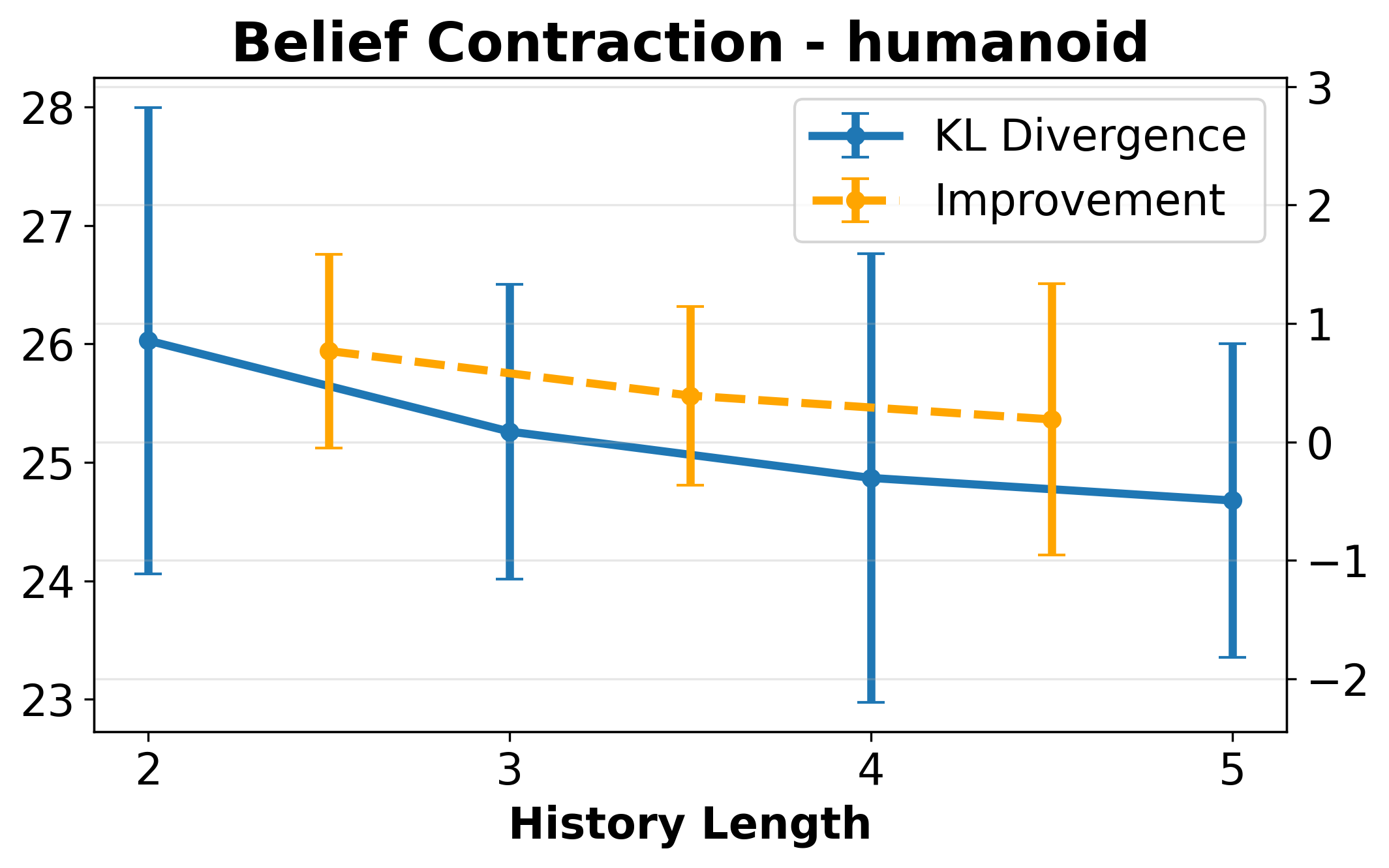}
        \caption{KL divergence between learned belief models with frame-stacks $L\in\{2,3,4,5\}$ and an $L=10$ reference model on all tasks. Each neural network predicts a Gaussian distribution over states. We average the divergence over 100 validation episodes from the latent expert. The orange plot denotes the decrease in KL divergence between frame-stack lengths. We repeat each experiment 5 times and plot the mean and standard deviation. This diagnostic proxy decreases modestly as the frame-stack length increases.}\label{fig:belief-contraction}
        \end{figure*}
}

\textbf{Empirical vignette: does a belief contraction proxy track RL sub-optimality?} We compute a diagnostic proxy for belief contraction for each task with no motor noise, and for \texttt{humanoid-walk} with $\mathsf{std\mhyphen dev}=0.2$. We train a reference model $\wh\belief^{L^\star}$ that takes $L^\star=10$ input frames. We compare against models $\wh \belief^L$ with $L\in[2,5]$ input frames. Each model's output belief is parametrized as a multivariate Gaussian distribution with diagonal covariance. All models are trained on the same $2000$ trajectories collected by the latent expert policy. For each $L$ we compute the $\mathsf{KL}$-divergence between outputs of $\wh \belief^L$ and $\wh\belief^{L^\star}$, and average across $100$ episodes of validation data, also collected by the same latent expert policy. Because this diagnostic uses learned Gaussian models and latent-expert trajectories, it does not directly test the theoretical rate for true beliefs under executable policies. We find that the proxy decreases slightly as $L$ increases (\cref{fig:belief-contraction}). Adding motor noise has little noticeable effect (\pref{fig:belief-contraction-sto}). Interestingly, the proxy is \emph{not} predictive across tasks: \texttt{dog-walk} has highest empirical error among the three tasks, yet RL achieves the lowest sub-optimality on it (\cref{fig:det-il-rl}), showing that this proxy alone does not explain the performance differences. \loose

\section{Towards Better Distillation: Imitating a Smoother Expert}
\label{sec:smooth}
\begin{figure*}[t]
\centering
\includegraphics[width=0.4\textwidth]{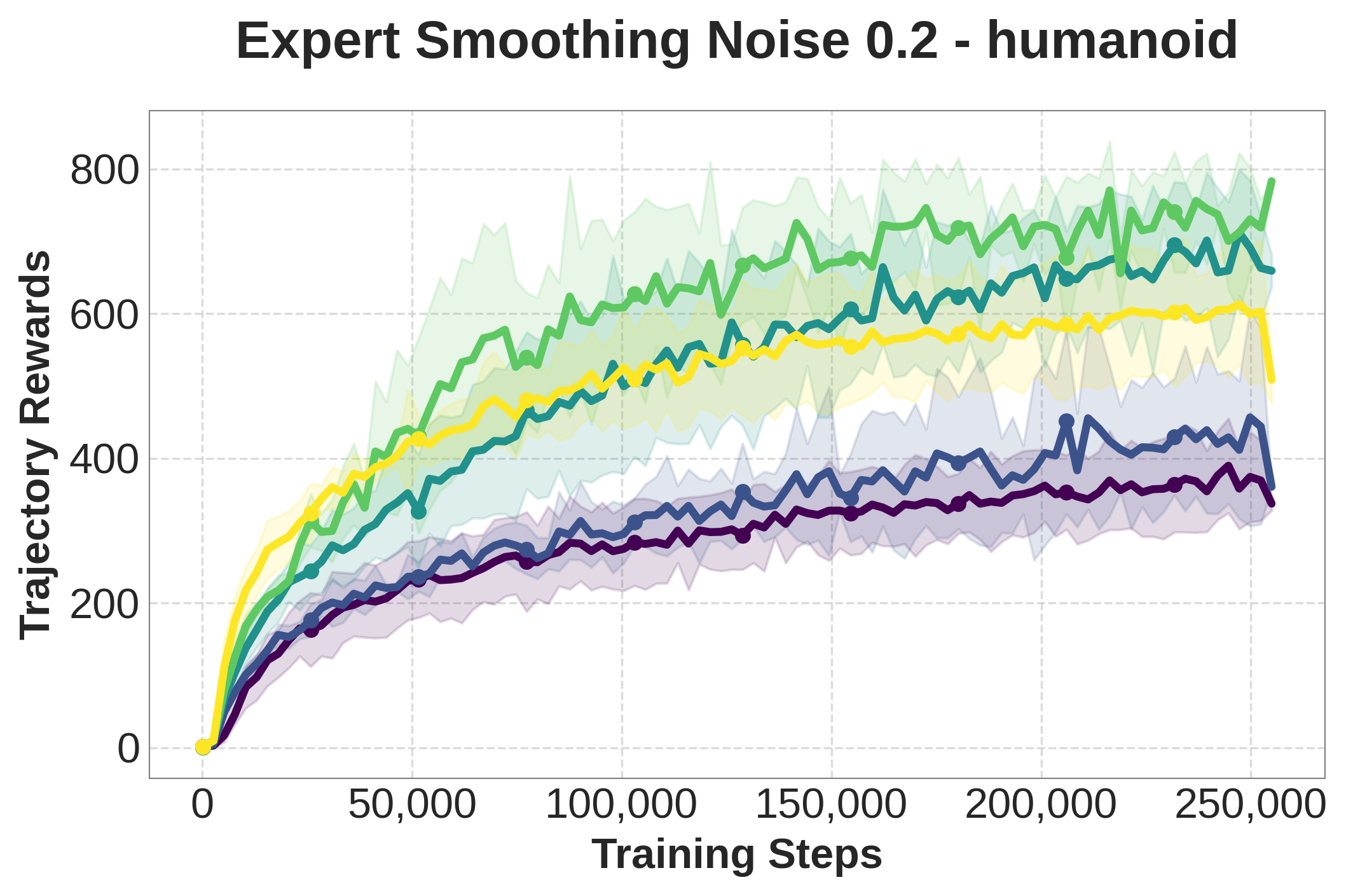}
\includegraphics[width=0.4\textwidth]{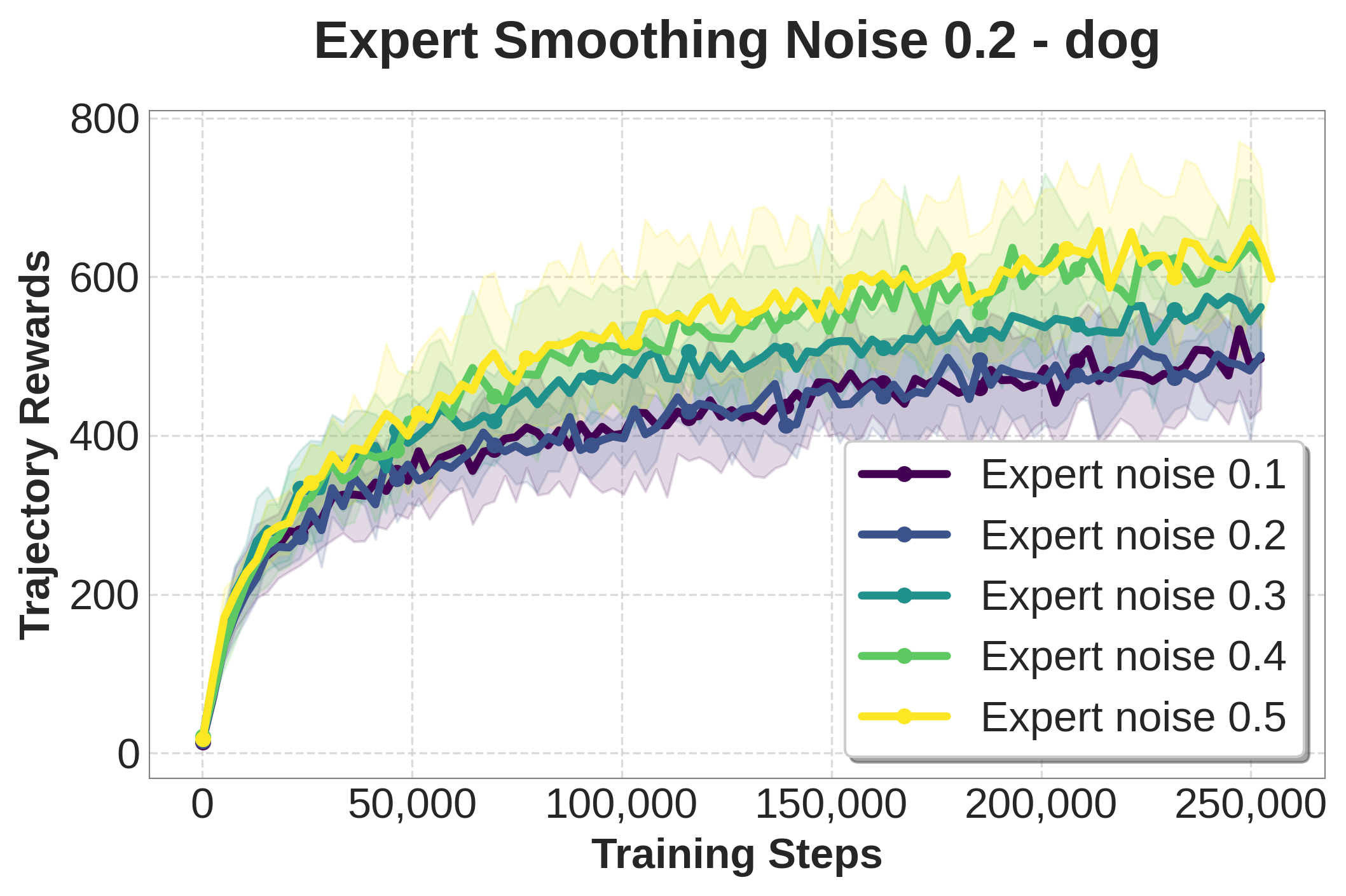}
\caption{Performance of $\Dagger$ on the validation dataset for the \texttt{humanoid-walk} and \texttt{dog-walk} environments with motor noise $\sigma=0.2$, as the noise level for the training environment (i.e. the environment in which the latent expert was trained) varies over $\{0.1,0.2,0.3,0.4,0.5\}$.}
\label{fig:smooth}
\end{figure*}

In this section, we discuss how the bounds via approximate decodability (e.g., \cref{lemma:tv-beltil-to-latent-main}) are loose since they fail to capture the \emph{smoothness} of the latent expert. A tighter bound with smoothness suggests potential benefits of artificially smoothing the latent expert before distillation. We then propose a broadly-applicable method for improving the smoothness, and show that it yields empirical benefits. We view these results as largely a proof-of-concept and leave more detailed investigation to future work. \loose

\paragraph{Smoothness of the latent policy} Suppose that  the true belief state at some step is always uniform over two particular states $\{s, s'\}$. Then decodability error is large, and a worst-case latent policy $\pilat$ --- namely, one that plays different actions on these states --- is unavoidably misspecified with respect to the class of executable policies. However, ambiguity between $s$ and $s'$ is most likely to occur if these states are somehow \emph{similar} (e.g., close w.r.t. a metric). If $\pilat$ is \emph{smooth} in the sense that it plays similar action distributions for nearby states, then the misspecification should be mitigated. This phenomenon can be captured more generally by the following variant of \cref{def:decodability-error}, which measures decodability error of the \emph{actions} (and hence is adaptive to the latent expert):

\begin{definition}[Action-prediction error]\label{def:action-prediction}
    Fix a latent policy $\pi^{\latent}$. For a fixed executable policy $\pi$, and timestep $h$, the \emph{action-prediction error} is defined as
    \begin{align*}
    \epsilon^{\asmooth;\pilat}_h(\pi) = \En^{\pi} \brk*{1 - \nrm*{\pilat\circ \belief_h(x_{1:h},a_{1:h-1})}_{\infty}}.
    \end{align*}
\end{definition}
In \pref{lemma:tv-beltil-to-latent-main}, the decodability error can indeed be replaced by the action-prediction error --- see \cref{lemma:smoothness-to-latent}.
Note that so long as $\pilat$ is deterministic (which is without loss of generality for the optimal latent policy), it generically holds that $\epsilon^{\asmooth;\pilat}_h(\pi) \leq \epsilon^{\decodability}_h(\pi)$.

\paragraph{Algorithmic intervention: smoothing experts with motor noise} One way to construct a smoother expert policy is to pre- or post-compose the optimal latent policy at each step with e.g. a Gaussian convolution kernel (on the state or action space, respectively). However, such approaches ignore the sequential nature of decision-making: smoothing the policy at later steps means that earlier actions may no longer be optimal. We propose instead computing the optimal policy for a \emph{modified} latent MDP with \emph{additional motor noise}. This encourages robustness to motor noise, as a tractable proxy for robustness to observation noise---see \cref{app:smoothing} for an example of one potential mechanism by which the former may lead to the latter.

\paragraph{Experimental results} For both \texttt{humanoid-walk} and \texttt{dog-walk}, for each $\sigma \in \{0.1, 0.2, 0.3, 0.4, 0.5\}$, we train an expert latent policy $\pi^\sigma$ in the environment with mean-$0$, std. dev.-$\sigma$, Gaussian motor noise on each action. We distill each expert to an executable policy via $\Dagger$ in an environment with $\sigma = 0.2$. We observe that $\pi^{0.2}$ incurs worse estimated action-prediction error than some higher-noise experts (\cref{fig:error_smooth}). Moreover, despite being trained for this environment, it is \emph{not} the best expert to distill (\cref{fig:smooth}): the distillations of policies with lower action-prediction error achieve higher reward (substantially for \texttt{humanoid-walk} and modestly for \texttt{dog-walk}). We also observe that the effect disappears when the true environment has deterministic dynamics (\cref{app:smooth}), likely since it is near-decodable.

\paragraph{Related methods} We view this method as a lightweight version of asymmetric RL methods that iteratively refine the expert \citep{warrington2021robust}. It is also closely related to the principle of noise injection in imitation learning \citep{laskey2017dart,block2023provable}, which has been shown to robustify Behavior Cloning---to \emph{match} the performance of $\Dagger$---by mitigating out-of-distribution effects.\footnote{We remark that $\Dagger$ saturates most of the benchmarks used by \cite{laskey2017dart}.} \cref{fig:smooth} demonstrates that even though $\Dagger$ uses online data collection to mitigate out-of-distribution effects, noise injection can still \emph{improve} its performance in challenging image-based tasks---thus suggesting that there may be a qualitatively different phenomenon at play in highly misspecified settings.

\section{Limitations and future work} Our theoretical results are for discrete tabular models with independent observation noise; weakening these assumptions could yield more precise understanding of the fundamental challenges that arise in applications with rich partial observations. Our experiments use synthetic injected motor noise; extending to more natural sources of stochasticity could be valuable. Also, there is a vast design space of algorithmic interventions for smoothing, of which we have only touched the surface.
Finally, our work is motivated by applications like robot learning where near-decodability is plausible, but an important problem --- which we did not explore --- is to understand the algorithmic trade-offs in applications that require active information-gathering.

\section*{Acknowledgements}
The authors are grateful to Noah Golowich, Audrey Huang, Nan Jiang, Akshay Krishnamurthy, Ankur Moitra, Wen Sun, Gokul Swamy, and Kaiqing Zhang for their insightful discussion. AS and YS acknowledge and thank the support of ONR grant N000142212363 and NSF AI Institute for Societal Decision Making AI-SDM grant IIS2229881. DR is supported by NSF awards CCF-2430381 and DMS-2022448, and ONR grant N00014-22-1-2339.

\bibliography{refs}

@article{li2025guided,
  title={Guided Policy Optimization under Partial Observability},
  author={Li, Yueheng and Xie, Guangming and Lu, Zongqing},
  journal={arXiv preprint arXiv:2505.15418},
  year={2025}
}

@article{rajaraman2021value,
  title={On the value of interaction and function approximation in imitation learning},
  author={Rajaraman, Nived and Han, Yanjun and Yang, Lin and Liu, Jingbo and Jiao, Jiantao and Ramchandran, Kannan},
  journal={Advances in Neural Information Processing Systems},
  volume={34},
  pages={1325--1336},
  year={2021}
}

@inproceedings{laskey2017dart,
  title={Dart: Noise injection for robust imitation learning},
  author={Laskey, Michael and Lee, Jonathan and Fox, Roy and Dragan, Anca and Goldberg, Ken},
  booktitle={Conference on robot learning},
  pages={143--156},
  year={2017},
  organization={PMLR}
}

@inproceedings{cheng2024extreme,
  title={Extreme parkour with legged robots},
  author={Cheng, Xuxin and Shi, Kexin and Agarwal, Ananye and Pathak, Deepak},
  booktitle={2024 IEEE International Conference on Robotics and Automation (ICRA)},
  pages={11443--11450},
  year={2024},
  organization={IEEE}
}

@article{rohatgi2025computational,
  title={Computational-Statistical Tradeoffs at the Next-Token Prediction Barrier: Autoregressive and Imitation Learning under Misspecification},
  author={Rohatgi, Dhruv and Block, Adam and Huang, Audrey and Krishnamurthy, Akshay and Foster, Dylan J},
  journal={arXiv preprint arXiv:2502.12465},
  year={2025}
}

@article{ross2008online,
  title={Online planning algorithms for POMDPs},
  author={Ross, St{\'e}phane and Pineau, Joelle and Paquet, S{\'e}bastien and Chaib-Draa, Brahim},
  journal={Journal of Artificial Intelligence Research},
  volume={32},
  pages={663--704},
  year={2008}
}

@article{choudhury2025process,
  title={Process reward models for llm agents: Practical framework and directions},
  author={Choudhury, Sanjiban},
  journal={arXiv preprint arXiv:2502.10325},
  year={2025}
}

@article{arora2018hindsight,
  title={Hindsight is only 50/50: Unsuitability of MDP based approximate POMDP solvers for multi-resolution information gathering},
  author={Arora, Sankalp and Choudhury, Sanjiban and Scherer, Sebastian},
  journal={arXiv preprint arXiv:1804.02573},
  year={2018}
}

@article{daume2009search,
  title={Search-based structured prediction},
  author={Daum{\'e}, Hal and Langford, John and Marcu, Daniel},
  journal={Machine learning},
  volume={75},
  pages={297--325},
  year={2009},
  publisher={Springer}
}

@inproceedings{chang2015learning,
  title={Learning to search better than your teacher},
  author={Chang, Kai-Wei and Krishnamurthy, Akshay and Agarwal, Alekh and Daum{\'e} III, Hal and Langford, John},
  booktitle={International Conference on Machine Learning},
  pages={2058--2066},
  year={2015},
  organization={PMLR}
}

@incollection{littman1995learning,
  title={Learning policies for partially observable environments: Scaling up},
  author={Littman, Michael L and Cassandra, Anthony R and Kaelbling, Leslie Pack},
  booktitle={Machine Learning Proceedings 1995},
  pages={362--370},
  year={1995},
  publisher={Elsevier}
}

@inproceedings{luo2023perpetual,
  title={Perpetual humanoid control for real-time simulated avatars},
  author={Luo, Zhengyi and Cao, Jinkun and Kitani, Kris and Xu, Weipeng and others},
  booktitle={Proceedings of the IEEE/CVF International Conference on Computer Vision},
  pages={10895--10904},
  year={2023}
}

@article{weihs2021bridging,
  title={Bridging the imitation gap by adaptive insubordination},
  author={Weihs, Luca and Jain, Unnat and Liu, Iou-Jen and Salvador, Jordi and Lazebnik, Svetlana and Kembhavi, Aniruddha and Schwing, Alex},
  journal={Advances in Neural Information Processing Systems},
  volume={34},
  pages={19134--19146},
  year={2021}
}

@article{messikommer2024student,
  title={Student-Informed Teacher Training},
  author={Messikommer, Nico and Xing, Jiaxu and Aljalbout, Elie and Scaramuzza, Davide},
  journal={arXiv preprint arXiv:2412.09149},
  year={2024}
}

@article{mousa2025tar,
  title={TAR: Teacher-Aligned Representations via Contrastive Learning for Quadrupedal Locomotion},
  author={Mousa, Amr and Karavis, Neil and Caprio, Michele and Pan, Wei and Allmendinger, Richard},
  journal={arXiv preprint arXiv:2503.20839},
  year={2025}
}

@article{nguyen2022leveraging,
  title={Leveraging fully observable policies for learning under partial observability},
  author={Nguyen, Hai and Baisero, Andrea and Wang, Dian and Amato, Christopher and Platt, Robert},
  journal={arXiv preprint arXiv:2211.01991},
  year={2022}
}

@inproceedings{shenfeld2023tgrl,
  title={TGRL: An algorithm for teacher guided reinforcement learning},
  author={Shenfeld, Idan and Hong, Zhang-Wei and Tamar, Aviv and Agrawal, Pulkit},
  booktitle={International Conference on Machine Learning},
  pages={31077--31093},
  year={2023},
  organization={PMLR}
}

@inproceedings{walsman2022impossibly,
  title={Impossibly good experts and how to follow them},
  author={Walsman, Aaron and Zhang, Muru and Choudhury, Sanjiban and Fox, Dieter and Farhadi, Ali},
  booktitle={The Eleventh International Conference on Learning Representations},
  year={2022}
}

@inproceedings{warrington2021robust,
  title={Robust asymmetric learning in pomdps},
  author={Warrington, Andrew and Lavington, Jonathan W and Scibior, Adam and Schmidt, Mark and Wood, Frank},
  booktitle={International Conference on Machine Learning},
  pages={11013--11023},
  year={2021},
  organization={PMLR}
}

@inproceedings{mu2025should,
  title={When should we prefer state-to-visual dagger over visual reinforcement learning?},
  author={Mu, Tongzhou and Li, Zhaoyang and Strzelecki, Stanis{\l}aw Wiktor and Yuan, Xiu and Yao, Yunchao and Liang, Litian and Su, Hao},
  booktitle={Proceedings of the AAAI Conference on Artificial Intelligence},
  volume={39},
  number={14},
  pages={14637--14645},
  year={2025}
}

@inproceedings{song2023learning,
  title={Learning perception-aware agile flight in cluttered environments},
  author={Song, Yunlong and Shi, Kexin and Penicka, Robert and Scaramuzza, Davide},
  booktitle={2023 IEEE International Conference on Robotics and Automation (ICRA)},
  pages={1989--1995},
  year={2023},
  organization={IEEE}
}

@inproceedings{fu2023learning,
  title={Learning deep sensorimotor policies for vision-based autonomous drone racing},
  author={Fu, Jiawei and Song, Yunlong and Wu, Yan and Yu, Fisher and Scaramuzza, Davide},
  booktitle={2023 IEEE/RSJ International Conference on Intelligent Robots and Systems (IROS)},
  pages={5243--5250},
  year={2023},
  organization={IEEE}
}

@article{kumar2021rma,
  title={Rma: Rapid motor adaptation for legged robots},
  author={Kumar, Ashish and Fu, Zipeng and Pathak, Deepak and Malik, Jitendra},
  journal={arXiv preprint arXiv:2107.04034},
  year={2021}
}

@inproceedings{yang2023neural,
  title={Neural volumetric memory for visual locomotion control},
  author={Yang, Ruihan and Yang, Ge and Wang, Xiaolong},
  booktitle={Proceedings of the IEEE/CVF conference on computer vision and pattern recognition},
  pages={1430--1440},
  year={2023}
}

@article{hoeller2024anymal,
  title={Anymal parkour: Learning agile navigation for quadrupedal robots},
  author={Hoeller, David and Rudin, Nikita and Sako, Dhionis and Hutter, Marco},
  journal={Science Robotics},
  volume={9},
  number={88},
  pages={eadi7566},
  year={2024},
  publisher={American Association for the Advancement of Science}
}

@article{Pinto2017AsymmetricAC,
  title={Asymmetric Actor Critic for Image-Based Robot Learning},
  author={Lerrel Pinto and Marcin Andrychowicz and Peter Welinder and Wojciech Zaremba and P. Abbeel},
  journal={ArXiv},
  year={2017},
  volume={abs/1710.06542}
}

@article{Shridhar2021CLIPortWA,
  title={CLIPort: What and Where Pathways for Robotic Manipulation},
  author={Mohit Shridhar and Lucas Manuelli and Dieter Fox},
  journal={ArXiv},
  year={2021},
  volume={abs/2109.12098},
}

@inproceedings{
fujimoto2025towards,
title={Towards General-Purpose Model-Free Reinforcement Learning},
author={Scott Fujimoto and Pierluca D'Oro and Amy Zhang and Yuandong Tian and Michael Rabbat},
booktitle={The Thirteenth International Conference on Learning Representations},
year={2025}
}

@inproceedings{chen2019lbc,
  author    = {Chen, Dian and Zhou, Brady and Koltun, Vladlen and Kr\"ahenb\"uhl, Philipp},
  title     = {Learning by Cheating},
  booktitle = {Conference on Robot Learning (CoRL)},
  year      = {2019},
}

@inproceedings{liu2022partially,
  title={When is partially observable reinforcement learning not scary?},
  author={Liu, Qinghua and Chung, Alan and Szepesv{\'a}ri, Csaba and Jin, Chi},
  booktitle={Conference on Learning Theory},
  pages={5175--5220},
  year={2022},
  organization={PMLR}
}

@inproceedings{uehara2023computationally,
  title={Computationally efficient pac rl in pomdps with latent determinism and conditional embeddings},
  author={Uehara, Masatoshi and Sekhari, Ayush and Lee, Jason D and Kallus, Nathan and Sun, Wen},
  booktitle={International Conference on Machine Learning},
  pages={34615--34641},
  year={2023},
  organization={PMLR}
}

@inproceedings{
kwon2024rl,
title={{RL} in Latent {MDP}s is Tractable: Online Guarantees via Off-Policy Evaluation},
author={Jeongyeol Kwon and Shie Mannor and Constantine Caramanis and Yonathan Efroni},
booktitle={The Thirty-eighth Annual Conference on Neural Information Processing Systems},
year={2024}
}

@inproceedings{
zhan2023pac,
title={{PAC} Reinforcement Learning for Predictive State Representations},
author={Wenhao Zhan and Masatoshi Uehara and Wen Sun and Jason D. Lee},
booktitle={The Eleventh International Conference on Learning Representations },
year={2023}
}

@inproceedings{
song2024richobservation,
title={Rich-Observation Reinforcement Learning with Continuous Latent Dynamics},
author={Yuda Song and Lili Wu and Dylan J Foster and Akshay Krishnamurthy},
booktitle={Forty-first International Conference on Machine Learning},
year={2024},
}

@article{rohatgi2025necessary,
  title={Necessary and Sufficient Oracles: Toward a Computational Taxonomy For Reinforcement Learning},
  author={Rohatgi, Dhruv and Foster, Dylan J},
  journal={arXiv preprint arXiv:2502.08632},
  year={2025}
}

@article{mundhenk2000complexity,
  title={Complexity of finite-horizon Markov decision process problems},
  author={Mundhenk, Martin and Goldsmith, Judy and Lusena, Christopher and Allender, Eric},
  journal={Journal of the ACM (JACM)},
  volume={47},
  number={4},
  pages={681--720},
  year={2000},
  publisher={ACM New York, NY, USA}
}

@article{doi:10.1126/scirobotics.abk2822,
author = {Takahiro Miki  and Joonho Lee  and Jemin Hwangbo  and Lorenz Wellhausen  and Vladlen Koltun  and Marco Hutter },
title = {Learning robust perceptive locomotion for quadrupedal robots in the wild},
journal = {Science Robotics},
volume = {7},
number = {62},
pages = {eabk2822},
year = {2022},
doi = {10.1126/scirobotics.abk2822},
}

@inproceedings{zhuang2023robot,
  author    = {Zhuang, Ziwen and Fu, Zipeng and Wang, Jianren and Atkeson, Christopher and Schwertfeger, Sören and Finn, Chelsea and Zhao, Hang},
  title     = {Robot Parkour Learning},
  booktitle = {Conference on Robot Learning ({CoRL})},
  year      = {2023},
}

@article{Lee2020LearningQL,
  title={Learning quadrupedal locomotion over challenging terrain},
  author={Joonho Lee and Jemin Hwangbo and Lorenz Wellhausen and Vladlen Koltun and Marco Hutter},
  journal={Science Robotics},
  year={2020},
  volume={5},
}

@article{Makoviychuk2021IsaacGH,
  title={Isaac Gym: High Performance GPU-Based Physics Simulation For Robot Learning},
  author={Viktor Makoviychuk and Lukasz Wawrzyniak and Yunrong Guo and Michelle Lu and Kier Storey and Miles Macklin and David Hoeller and N. Rudin and Arthur Allshire and Ankur Handa and Gavriel State},
  journal={ArXiv},
  year={2021},
  volume={abs/2108.10470}
}

@article{Tassa2018DeepMindCS,
  title={DeepMind Control Suite},
  author={Yuval Tassa and Yotam Doron and Alistair Muldal and Tom Erez and Yazhe Li and Diego de Las Casas and David Budden and Abbas Abdolmaleki and Josh Merel and Andrew Lefrancq and Timothy P. Lillicrap and Martin A. Riedmiller},
  journal={ArXiv},
  year={2018},
  volume={abs/1801.00690}
}

@inproceedings{
foster2024is,
title={Is Behavior Cloning All You Need? Understanding Horizon in Imitation Learning},
author={Dylan J Foster and Adam Block and Dipendra Misra},
booktitle={The Thirty-eighth Annual Conference on Neural Information Processing Systems},
year={2024}}

@article{golowich2022learning,
  title={Learning in observable pomdps, without computationally intractable oracles},
  author={Golowich, Noah and Moitra, Ankur and Rohatgi, Dhruv},
  journal={Advances in neural information processing systems},
  volume={35},
  pages={1458--1473},
  year={2022}
}

@inproceedings{ross2010efficient,
  title={Efficient reductions for imitation learning},
  author={Ross, St{\'e}phane and Bagnell, Drew},
  booktitle={Proceedings of the thirteenth international conference on artificial intelligence and statistics},
  pages={661--668},
  year={2010},
  organization={JMLR Workshop and Conference Proceedings}
}

@article{kara2022near,
  title={Near optimality of finite memory feedback policies in partially observed Markov decision processes},
  author={Kara, Ali and Yuksel, Serdar},
  journal={Journal of Machine Learning Research},
  volume={23},
  number={11},
  pages={1--46},
  year={2022}
}

@article{papadimitriou1987complexity,
  title={The complexity of Markov decision processes},
  author={Papadimitriou, Christos H and Tsitsiklis, John N},
  journal={Mathematics of operations research},
  volume={12},
  number={3},
  pages={441--450},
  year={1987},
  publisher={INFORMS}
}

@inproceedings{golowich2023planning,
  title={Planning and learning in partially observable systems via filter stability},
  author={Golowich, Noah and Moitra, Ankur and Rohatgi, Dhruv},
  booktitle={Proceedings of the 55th Annual ACM Symposium on Theory of Computing},
  pages={349--362},
  year={2023}
}

@inproceedings{liu2023optimistic,
  title={Optimistic mle: A generic model-based algorithm for partially observable sequential decision making},
  author={Liu, Qinghua and Netrapalli, Praneeth and Szepesvari, Csaba and Jin, Chi},
  booktitle={Proceedings of the 55th Annual ACM Symposium on Theory of Computing},
  pages={363--376},
  year={2023}
}

@inproceedings{lee2023learning,
  title={Learning in pomdps is sample-efficient with hindsight observability},
  author={Lee, Jonathan and Agarwal, Alekh and Dann, Christoph and Zhang, Tong},
  booktitle={International Conference on Machine Learning},
  pages={18733--18773},
  year={2023},
  organization={PMLR}
}

@inproceedings{
cai2024provable,
title={Provable Partially Observable Reinforcement Learning with Privileged Information},
author={Yang Cai and Xiangyu Liu and Argyris Oikonomou and Kaiqing Zhang},
booktitle={The Thirty-eighth Annual Conference on Neural Information Processing Systems},
year={2024},
}

@inproceedings{efroni2022provable,
  title={Provable reinforcement learning with a short-term memory},
  author={Efroni, Yonathan and Jin, Chi and Krishnamurthy, Akshay and Miryoosefi, Sobhan},
  booktitle={International Conference on Machine Learning},
  pages={5832--5850},
  year={2022},
  organization={PMLR}
}

@article{mnih2013playing,
  title={Playing atari with deep reinforcement learning},
  author={Mnih, Volodymyr and Kavukcuoglu, Koray and Silver, David and Graves, Alex and Antonoglou, Ioannis and Wierstra, Daan and Riedmiller, Martin},
  journal={arXiv preprint arXiv:1312.5602},
  year={2013}
}

@inproceedings{agarwal2023legged,
  title={Legged locomotion in challenging terrains using egocentric vision},
  author={Agarwal, Ananye and Kumar, Ashish and Malik, Jitendra and Pathak, Deepak},
  booktitle={Conference on robot learning},
  pages={403--415},
  year={2023},
  organization={PMLR}
}

@article{block2023provable,
  title={Provable guarantees for generative behavior cloning: Bridging low-level stability and high-level behavior},
  author={Block, Adam and Jadbabaie, Ali and Pfrommer, Daniel and Simchowitz, Max and Tedrake, Russ},
  journal={Advances in Neural Information Processing Systems},
  volume={36},
  pages={48534--48547},
  year={2023}
}

@article{choudhury2018data,
  title={Data-driven planning via imitation learning},
  author={Choudhury, Sanjiban and Bhardwaj, Mohak and Arora, Sankalp and Kapoor, Ashish and Ranade, Gireeja and Scherer, Sebastian and Dey, Debadeepta},
  journal={The International Journal of Robotics Research},
  volume={37},
  number={13-14},
  pages={1632--1672},
  year={2018},
  publisher={SAGE Publications Sage UK: London, England}
}

@inproceedings{golowich2024exploration,
  title={Exploration is harder than prediction: Cryptographically separating reinforcement learning from supervised learning},
  author={Golowich, Noah and Moitra, Ankur and Rohatgi, Dhruv},
  booktitle={2024 IEEE 65th Annual Symposium on Foundations of Computer Science (FOCS)},
  pages={1953--1967},
  year={2024},
  organization={IEEE}
}

@article{ross2014reinforcement,
  title={Reinforcement and imitation learning via interactive no-regret learning},
  author={Ross, Stephane and Bagnell, J Andrew},
  journal={arXiv preprint arXiv:1406.5979},
  year={2014}
}

@article{rajaraman2020toward,
  title={Toward the fundamental limits of imitation learning},
  author={Rajaraman, Nived and Yang, Lin and Jiao, Jiantao and Ramchandran, Kannan},
  journal={Advances in Neural Information Processing Systems},
  volume={33},
  pages={2914--2924},
  year={2020}
}

@article{block2023butterfly,
  title={Butterfly effects of sgd noise: Error amplification in behavior cloning and autoregression},
  author={Block, Adam and Foster, Dylan J and Krishnamurthy, Akshay and Simchowitz, Max and Zhang, Cyril},
  journal={arXiv preprint arXiv:2310.11428},
  year={2023}
}

@article{pan2017agile,
  title={Agile autonomous driving using end-to-end deep imitation learning},
  author={Pan, Yunpeng and Cheng, Ching-An and Saigol, Kamil and Lee, Keuntaek and Yan, Xinyan and Theodorou, Evangelos and Boots, Byron},
  journal={arXiv preprint arXiv:1709.07174},
  year={2017}
}

@article{zhou2022horizon,
  title={Horizon-free reinforcement learning for latent markov decision processes},
  author={Zhou, Runlong and Wang, Ruosong and Du, Simon Shaolei},
  year={2022}
}

@article{burago1996complexity,
  title={On the complexity of partially observed Markov decision processes},
  author={Burago, Dima and De Rougemont, Michel and Slissenko, Anatol},
  journal={Theoretical Computer Science},
  volume={157},
  number={2},
  pages={161--183},
  year={1996},
  publisher={Elsevier}
}

@inproceedings{littman1994memoryless,
  title={Memoryless policies: Theoretical limitations and practical results},
  author={Littman, Michael L},
  booktitle={From Animals to Animats 3: Proceedings of the third international conference on simulation of adaptive behavior},
  volume={3},
  pages={238},
  year={1994},
  organization={MIT Press Cambridge, MA, USA}
}

@article{christiano2016transfer,
  title={Transfer from simulation to real world through learning deep inverse dynamics model},
  author={Christiano, Paul and Shah, Zain and Mordatch, Igor and Schneider, Jonas and Blackwell, Trevor and Tobin, Joshua and Abbeel, Pieter and Zaremba, Wojciech},
  journal={arXiv preprint arXiv:1610.03518},
  year={2016}
}

@article{seo2023regularized,
  title={Regularized behavior cloning for blocking the leakage of past action information},
  author={Seo, Seokin and Hwang, HyeongJoo and Yang, Hongseok and Kim, Kee-Eung},
  journal={Advances in Neural Information Processing Systems},
  volume={36},
  pages={2128--2153},
  year={2023}
}

@article{swamy2022sequence,
  title={Sequence model imitation learning with unobserved contexts},
  author={Swamy, Gokul and Choudhury, Sanjiban and Bagnell, J and Wu, Steven Z},
  journal={Advances in Neural Information Processing Systems},
  volume={35},
  pages={17665--17676},
  year={2022}
}

@inproceedings{sun2017deeply,
  title={Deeply aggrevated: Differentiable imitation learning for sequential prediction},
  author={Sun, Wen and Venkatraman, Arun and Gordon, Geoffrey J and Boots, Byron and Bagnell, J Andrew},
  booktitle={International conference on machine learning},
  pages={3309--3318},
  year={2017},
  organization={PMLR}
}

@article{eberhard2025partially,
  title={Partially Observable Reinforcement Learning with Memory Traces},
  author={Eberhard, Onno and Muehlebach, Michael and Vernade, Claire},
  journal={arXiv preprint arXiv:2503.15200},
  year={2025}
}

@article{mnih2015human,
  title={Human-level control through deep reinforcement learning},
  author={Mnih, Volodymyr and Kavukcuoglu, Koray and Silver, David and Rusu, Andrei A and Veness, Joel and Bellemare, Marc G and Graves, Alex and Riedmiller, Martin and Fidjeland, Andreas K and Ostrovski, Georg and others},
  journal={Nature},
  volume={518},
  number={7540},
  pages={529},
  year={2015},
  publisher={Nature Publishing Group}
}

@inproceedings{du2019provably,
  title={Provably efficient {RL} with rich observations via latent state decoding},
  author={Du, Simon and Krishnamurthy, Akshay and Jiang, Nan and Agarwal, Alekh and Dudik, Miroslav and Langford, John},
  booktitle={International Conference on Machine Learning},
  year={2019},
}

@inproceedings{krishnamurthy2016pac,
  title={{PAC} reinforcement learning with rich observations},
  author={Krishnamurthy, Akshay and Agarwal, Alekh and Langford, John},
  booktitle={Advances in Neural Information Processing Systems},
  year={2016}
}

@inproceedings{jiang2017contextual,
  title={Contextual decision processes with low {Bellman} rank are {PAC}-learnable},
  author={Jiang, Nan and Krishnamurthy, Akshay and Agarwal, Alekh and Langford, John and Schapire, Robert E},
  booktitle={International Conference on Machine Learning},
  year={2017}
}

@inproceedings{dann2018oracle,
  title={On oracle-efficient {PAC} {RL} with rich observations},
  author={Dann, Christoph and Jiang, Nan and Krishnamurthy, Akshay and Agarwal, Alekh and Langford, John and Schapire, Robert E},
  booktitle={Advances in Neural Information Processing Systems},
  year={2018}
}

@inproceedings{du2019latent,
  title={Provably efficient {RL} with rich observations via latent state decoding},
  author={Du, Simon and Krishnamurthy, Akshay and Jiang, Nan and Agarwal, Alekh and Dudik, Miroslav and Langford, John},
  booktitle={International Conference on Machine Learning},
  pages={1665--1674},
  year={2019},
  organization={PMLR}
}

@article{kearns2002near,
  title={Near-optimal reinforcement learning in polynomial time},
  author={Kearns, Michael and Singh, Satinder},
  journal={Machine learning},
  volume={49},
  pages={209--232},
  year={2002},
  publisher={Springer}
}

@article{uehara2022provably,
  title={Provably efficient reinforcement learning in partially observable dynamical systems},
  author={Uehara, Masatoshi and Sekhari, Ayush and Lee, Jason D and Kallus, Nathan and Sun, Wen},
  journal={Advances in Neural Information Processing Systems},
  volume={35},
  pages={578--592},
  year={2022}
}

@article{zhang2025statistical,
  title={Statistical Tractability of Off-policy Evaluation of History-dependent Policies in POMDPs},
  author={Zhang, Yuheng and Jiang, Nan},
  journal={arXiv preprint arXiv:2503.01134},
  year={2025}
}

@article{uehara2023future,
  title={Future-dependent value-based off-policy evaluation in pomdps},
  author={Uehara, Masatoshi and Kiyohara, Haruka and Bennett, Andrew and Chernozhukov, Victor and Jiang, Nan and Kallus, Nathan and Shi, Chengchun and Sun, Wen},
  journal={Advances in neural information processing systems},
  volume={36},
  pages={15991--16008},
  year={2023}
}

@inproceedings{ross2011reduction,
  title={A reduction of imitation learning and structured prediction to no-regret online learning},
  author={Ross, St{\'e}phane and Gordon, Geoffrey and Bagnell, Drew},
  booktitle={Proceedings of the fourteenth international conference on artificial intelligence and statistics},
  pages={627--635},
  year={2011},
  organization={JMLR Workshop and Conference Proceedings}
}

@article{brafman2002r,
  title={R-max-a general polynomial time algorithm for near-optimal reinforcement learning},
  author={Brafman, Ronen I and Tennenholtz, Moshe},
  journal={Journal of Machine Learning Research},
  volume={3},
  number={Oct},
  pages={213--231},
  year={2002}
}

@article{jin2020sample,
  title={Sample-Efficient Reinforcement Learning of Undercomplete {POMDPs}},
  author={Jin, Chi and Kakade, Sham and Krishnamurthy, Akshay and Liu, Qinghua},
  journal={Advances in Neural Information Processing Systems},
  volume={33},
  year={2020}
}

@inproceedings{misra2020kinematic,
  title={Kinematic state abstraction and provably efficient rich-observation reinforcement learning},
  author={Misra, Dipendra and Henaff, Mikael and Krishnamurthy, Akshay and Langford, John},
  booktitle={International conference on machine learning},
  year={2020},
}

@article{kwon2021rl,
  title={{RL} for latent {MDPs}: Regret guarantees and a lower bound},
  author={Kwon, Jeongyeol and Efroni, Yonathan and Caramanis, Constantine and Mannor, Shie},
  journal={Advances in Neural Information Processing Systems},
  volume={34},
  year={2021}
}

@inproceedings{mhammedi2023representation,
  title={Representation learning with multi-step inverse kinematics: An efficient and optimal approach to rich-observation rl},
  author={Mhammedi, Zakaria and Foster, Dylan J and Rakhlin, Alexander},
  booktitle={International Conference on Machine Learning},
  year={2023}
}

@inproceedings{zhang2022efficient,
	title={Efficient reinforcement learning in block {MDP}s: A model-free representation learning approach},
	author={Zhang, Xuezhou and Song, Yuda and Uehara, Masatoshi and Wang, Mengdi and Agarwal, Alekh and Sun, Wen},
	booktitle={International Conference on Machine Learning},
	year={2022},
}

@article{mhammedi2023efficient,
  title={Efficient Model-Free Exploration in Low-Rank {MDP}s},
  author={Mhammedi, Zakaria and Block, Adam and Foster, Dylan J and Rakhlin, Alexander},
  journal={Advances in Neural Information Processing Systems},
  year={2023}
}
\iclr{
\bibliographystyle{abbrv}
}
\newpage
\iclr{
\section*{NeurIPS Paper Checklist}
\begin{enumerate}

\item {\bf Claims}
    \item[] Question: Do the main claims made in the abstract and introduction accurately reflect the paper's contributions and scope?
    \item[] Answer: \answerYes{}
    \item[] Justification: We provide both theoretical and empirical evidence on how approximate decodability and belief contraction explain the algorithmic trade-off between expert distillation and RL.

\item {\bf Limitations}
    \item[] Question: Does the paper discuss the limitations of the work performed by the authors?
    \item[] Answer: \answerYes{}
    \item[] Justification: We include a discussion on the limitations of the work in the final section.

\item {\bf Theory assumptions and proofs}
    \item[] Question: For each theoretical result, does the paper provide the full set of assumptions and a complete (and correct) proof?
    \item[] Answer: \answerYes{}
    \item[] Justification: We explicitly state our assumptions and provide proofs to all of our theoretical results.

    \item {\bf Experimental result reproducibility}
    \item[] Question: Does the paper fully disclose all the information needed to reproduce the main experimental results of the paper to the extent that it affects the main claims and/or conclusions of the paper (regardless of whether the code and data are provided or not)?
    \item[] Answer: \answerYes{}
    \item[] Justification: We discuss all details of the experiments, including hyperparameters.

\item {\bf Open access to data and code}
    \item[] Question: Does the paper provide open access to the data and code, with sufficient instructions to faithfully reproduce the main experimental results, as described in supplemental material?
    \item[] Answer: \answerNo{}
    \item[] Justification: We did not provide the code in the supplemental material as the algorithms we implemented are standard ones, and we provide all hyperparameters in \cref{app:experiment}.

\item {\bf Experimental setting/details}
    \item[] Question: Does the paper specify all the training and test details (e.g., data splits, hyperparameters, how they were chosen, type of optimizer, etc.) necessary to understand the results?
    \item[] Answer: \answerYes{}{}
    \item[] Justification: We discussed how to generate the training and validation data, and our choice of optimizers.

\item {\bf Experiment statistical significance}
    \item[] Question: Does the paper report error bars suitably and correctly defined or other appropriate information about the statistical significance of the experiments?
    \item[] Answer: \answerYes{}
    \item[] Justification: All of our results include error bars.

\item {\bf Experiments compute resources}
    \item[] Question: For each experiment, does the paper provide sufficient information on the computer resources (type of compute workers, memory, time of execution) needed to reproduce the experiments?
    \item[] Answer: \answerYes{}
    \item[] Justification: We include our compute details in \cref{app:experiment}.

\item {\bf Code of ethics}
    \item[] Question: Does the research conducted in the paper conform, in every respect, with the NeurIPS Code of Ethics \url{https://neurips.cc/public/EthicsGuidelines}?
    \item[] Answer: \answerYes{}
    \item[] Justification: Our research conforms with the NeurIPS Code of Ethics.

\item {\bf Broader impacts}
    \item[] Question: Does the paper discuss both potential positive societal impacts and negative societal impacts of the work performed?
    \item[] Answer: \answerNA{}
    \item[] Justification: Our paper does not incur significant societal impacts.

\item {\bf Safeguards}
    \item[] Question: Does the paper describe safeguards that have been put in place for responsible release of data or models that have a high risk for misuse (e.g., pretrained language models, image generators, or scraped datasets)?
    \item[] Answer: \answerNA{}
    \item[] Justification: Our paper poses no such risks.

\item {\bf Licenses for existing assets}
    \item[] Question: Are the creators or original owners of assets (e.g., code, data, models), used in the paper, properly credited and are the license and terms of use explicitly mentioned and properly respected?
    \item[] Answer: \answerNA{}
    \item[] Justification: Our paper does not leverage existing dataset or codebases.

\item {\bf New assets}
    \item[] Question: Are new assets introduced in the paper well documented and is the documentation provided alongside the assets?
    \item[] Answer: \answerNA{}
    \item[] Justification: We do not release new assets.

\item {\bf Crowdsourcing and research with human subjects}
    \item[] Question: For crowdsourcing experiments and research with human subjects, does the paper include the full text of instructions given to participants and screenshots, if applicable, as well as details about compensation (if any)?
    \item[] Answer: \answerNA{}
    \item[] Justification: The paper does not involve crowdsourcing nor research with human subjects.

\item {\bf Institutional review board (IRB) approvals or equivalent for research with human subjects}
    \item[] Question: Does the paper describe potential risks incurred by study participants, whether such risks were disclosed to the subjects, and whether Institutional Review Board (IRB) approvals (or an equivalent approval/review based on the requirements of your country or institution) were obtained?
    \item[] Answer: \answerNA{}
    \item[] Justification: The paper does not involve crowdsourcing nor research with human subjects.

\item {\bf Declaration of LLM usage}
    \item[] Question: Does the paper describe the usage of LLMs if it is an important, original, or non-standard component of the core methods in this research? Note that if the LLM is used only for writing, editing, or formatting purposes and does not impact the core methodology, scientific rigorousness, or originality of the research, declaration is not required.
    \item[] Answer: \answerNA{}
    \item[] Justification: The core method development in this research does not involve LLMs as any important, original, or non-standard components.

\end{enumerate}

}

\newpage

 \renewcommand{\contentsname}{Contents of Appendix}
 \addtocontents{toc}{\protect\setcounter{tocdepth}{2}}
 {\hypersetup{hidelinks}
 \tableofcontents
 }

\appendix

\section{Additional Related Work}
\label{sec:related}
\subsection{Theoretical literature}

\paragraph{Planning and learning in POMDPs} It is well-known that the \emph{planning} problem in POMDPs (i.e. finding a near-optimal policy given the description of the POMDP) is computationally intractable \citep{papadimitriou1987complexity, littman1994memoryless,burago1996complexity,mundhenk2000complexity}, and the harder \emph{learning} problem (i.e. finding a near-optimal policy given interactive sample access to the POMDP) is also statistically intractable \citep{krishnamurthy2016pac}, without additional assumptions.
In light of these results, there has been recent interest in uncovering natural assumptions that allow statistically or computationally efficient algorithms. On the computational side, \cite{efroni2022provable} introduced the $L$-step decodability assumption, and under this assumption derived a learning algorithm with time complexity $\poly(X^L, A^L, H)$ via frame-stacking. Additionally, \cite{golowich2023planning,golowich2022learning} derived a quasi-polynomial time algorithm for learning in $\gamma$-observable POMDPs. Computationally efficient learning algorithms are also known for certain classes of POMDPs with deterministic dynamics \citep{jin2020sample,uehara2023computationally} and certain latent MDPs \citep{kwon2021rl,kwon2024rl}, which are a special case of POMDPs with fixed latent information.

On the statistical side, \cite{jin2020sample} derived a statistically efficient algorithm for POMDPs satisfying a weak observability condition. Recently, \citep{liu2022partially,uehara2022provably,liu2023optimistic,zhan2023pac} proposed statistically efficient algorithms for POMDPs or Predictive State Representations (PSR) satisfying certain low-rank conditions. More tangential to our work, there has also been increasing interest in off-policy evaluation in POMDPs \citep{uehara2023future,zhang2025statistical}.

\paragraph{Learning with privileged information in POMDPs} The most relevant theoretical works to ours are recent works that study the problem of learning POMDPs with latent state information (also called \emph{hindsight observability}) \citep{kwon2021rl,zhou2022horizon,lee2023learning,cai2024provable}. Of these, \cite{kwon2021rl,zhou2022horizon} are focused on a narrow yet interesting special case of POMDPs called \emph{latent MDPs}, where the unobserved data is fixed and low-dimensional. \cite{lee2023learning} show that learning in general POMDPs with latent state information is statistically tractable, in contrast with the situation without latent state information. \cite{cai2024provable} show that with latent state information the sample complexity of the algorithm for learning $\gamma$-observable POMDPs \citep{golowich2022learning} can be improved from quasi-polynomial to polynomial, though it is an open question whether this is possible without latent state information. Finally, as mentioned earlier, \cite{cai2024provable} showed that in perfectly decodable POMDPs (\cref{def:perfectly-decodable}), expert distillation yields a fully polynomial time algorithm for learning (for arbitrarily large window size $L$). Since there is a statistical lower bound of $\Omega(A^L)$ in the absence of latent state information \citep{efroni2022provable}, this yields a provable computational benefit of latent state information (and, in particular, for expert distillation), but only for perfectly decodable POMDPs.

Compared to the preceding theoretical works, our work seeks both theoretically \emph{and} empirically grounded understanding of the relative merits of expert distillation versus standard reinforcement learning. Among works with similar motivations or results, \cite{choudhury2018data} derive expressions for the output of expert distillation (analogous to \cref{lemma:tv-beltil-to-latent-main}, except they use a slightly different value-based distillation procedure rather than policy-based) and establish sub-optimality bounds for several imitation learning algorithms. However, they do not instantiate these bounds for concrete models, or theoretically contrast with reinforcement learning. \cite{sun2017deeply} establish provable benefits of imitating the optimal policy in a \emph{fully-observed} MDP (versus learning it via reinforcement learning), but they do not consider partial observability nor the ensuing error due to misspecification. \cite{swamy2022sequence} discuss a failure mode of expert distillation in POMDPs when using \emph{offline} imitation learning to distill the latent expert. The ``latching'' effect that they discuss is due to conditioning on previous actions (see also \cite{seo2023regularized}) --- a technical issue that corresponds to why we analyze $\forward$ with $L$ random actions --- though it is not clear whether this effect is related to the performance gaps between behavior cloning and $\Dagger$ in our locomotion experiments. Finally, several works \citep{arora2018hindsight,weihs2021bridging} give examples of a more fundamental failure mode of expert distillation: in general POMDPs, the optimal policy may need to take \emph{information-gathering actions}. The classical ``Tiger Door'' exemplifies this failure mode \citep{littman1995learning}.

\paragraph{Learning with rich observations} There has been extensive recent interest in reinforcement learning with \emph{rich observations} \citep{krishnamurthy2016pac}, i.e. where the observation space is too large to enumerate. This line of work has developed largely in parallel with the literature on partial observability, but it is motivated by similar applications as our work (e.g. robotics with image-based perception), and these works formalize the fundamental empirical challenge of representation learning, i.e. ``learning to see'' \citep{chen2019lbc}. The most well-studied model is the Block MDP \citep{dann2018oracle,du2019latent}, which corresponds to perfect decodability with $L=1$, but is studied in function approximation settings where the observation space is extremely large or infinite, since the problem is computationally easy if the observation space has polynomially-bounded cardinality. While the task of learning in Block MDPs is typically computationally intractable as it inherits the intractable of PAC learning  \citep{golowich2024exploration}, there is by now precise understanding of the computational complexity relative to supervised learning oracles \citep{misra2020kinematic,zhang2022efficient,mhammedi2023efficient,song2024richobservation,rohatgi2025necessary}.

As observed by \cite{cai2024provable}, there is also a provable computational benefit of latent state information in Block MDPs. Recent works \citep{golowich2022learning,rohatgi2025necessary} showed that in the absence of latent state information, learning in $\Phi$-decodable Block MDPs (where $\Phi$ is the function approximation class) is strictly harder than the supervised learning task of $\Phi$-decodable \emph{one-context regression}. In contrast, it is straightforward to see that with latent state information and a one-context regression oracle, the true decoding function $\phi^\star \in \Phi$ can be learned up to inverse-polynomial error (on average over any exploratory policy). This function, composed with the optimal latent policy, yields the optimal executable policy. \cite{cai2024provable} formally proved this result with a slightly different (multi-class classification rather than regression) supervised learning oracle.

\subsection{Empirical literature}

\paragraph{Applied methods that learn with privileged information} Privileged information has been widely used in training policies for real-world POMDPs, such as in robotics and autonomous driving. The most prominent and successful method is expert distillation \citep{pan2017agile,chen2019lbc,Lee2020LearningQL,doi:10.1126/scirobotics.abk2822,kumar2021rma,zhuang2023robot,yang2023neural,song2023learning,fu2023learning,hoeller2024anymal,cheng2024extreme}. First, one trains an expert policy with access to privileged information --- either the latent state in a simulator \citep{chen2019lbc}, or observation data from more expensive sensors that will not be available at deployment \citep{pan2017agile}. Second, one trains an executable policy by performing offline or online imitation learning with respect to the latent policy. \cite{pan2017agile} also observe empirical benefits of online imitation learning compared to offline, which is corroborated by our results. While there are also notable successes of using RL without privileged information \citep{agarwal2023legged,luo2023perpetual}, some of the previously-mentioned works observed that RL without privileged information failed to learn locomotion in their environment \citep{Lee2020LearningQL}.

Motivated by the theoretical failure modes of expert distillation in the prequel, there is also a line of work in the middle ground between expert distillation and RL without privileged information, that seeks to avoid these failure modes while also improving the convergence of standard RL. These hybrid methods include Asymmetric Actor-Critic \citep{Pinto2017AsymmetricAC}, and more broadly are described as \emph{asymmetric learning} \citep{Pinto2017AsymmetricAC,weihs2021bridging,warrington2021robust,nguyen2022leveraging,walsman2022impossibly,shenfeld2023tgrl,messikommer2024student,mousa2025tar,li2025guided}. While there is some evidence that an algorithm inspired by Asymmetric Actor-Critic may enjoy an improved statistical/computational trade-off for $\gamma$-observable POMDPs \citep{cai2024provable} (compared to the best-known algorithm that does not use privileged information \citep{golowich2022learning}), the theoretical foundations for these methods remain otherwise largely unexplored.

\paragraph{Learning with and without privileged information} In addition to the previously-mentioned ablation experiments \citep{Lee2020LearningQL}, recent work of \cite{mu2025should} conducted controlled comparisons between expert distillation and standard RL on simulated locomotion and manipulation tasks, with the goal of providing heuristic guidance on when to prefer expert distillation over RL without privileged information. They found that expert distillation converges faster. They also classified tasks as ``easy'' or ``hard'' based on the convergence speed of standard RL, and suggested that expert distillation performed better on the ``hard'' tasks.

\paragraph{Improvements to expert distillation} Recall that a key benefit of expert distillation for POMDPs with rich observations (e.g. as found in robotics with image-based perception) was that it avoids performing reinforcement learning on the high-dimensional and complex observation space. In contrast, most of the previously-mentioned works on asymmetric learning use exactly such an algorithm (e.g., as the ``Actor'' component in Asymmetric Actor-Critic). An exception is the method of \cite{warrington2021robust}, which is a variant of expert distillation that iteratively refines the expert with the goal of decreasing misspecification. Our smoothed distillation method (\cref{sec:smooth}) can be thought of as a more lightweight approach that refines the expert in one shot. As mentioned in \cref{sec:smooth}, it is also similar (though not identical) to several noise injection methods in imitation learning \cite{laskey2017dart,block2023provable}.

\section{Additional Preliminaries}
\label{sec:additional}
\subsection{Belief states}

The following operators describe how a belief state evolves as more information is revealed.

\begin{definition}[Belief state update \citep{golowich2023planning}]
  \label{def:bel-update}
For each $h \in \{1,\dots,H\}$, the \emph{Bayes operator} is $\BB_h: \Delta(\MS_h) \times \cX_h \to \Delta(\MS_h)$ defined by \[\BB_h(b; x_h)(s_h) := \frac{\BO_h(x_h\mid{}s_h)b(s_h)}{\sum_{z_h \in \MS_h} \BO_h(x_h\mid{}z_h)b(z_h)}.\]
If the denominator is zero, choose an arbitrary fixed distribution; this only defines the operator on observations outside the support of its input prior.
For each $h \in \{2,\dots,H\}$, the \emph{belief update operator} $\BU_h: \Delta(\MS_{h-1}) \times \MA_{h-1} \times \cX_h \to \Delta(\MS_h)$, is defined by $\BU_h(b;a_{h-1},x_h) :=  \BB_h(\bbP_h(a_{h-1}) \cdot b;x_h)$
where $\bbP_h(a)$ denotes the real-valued $|\MS_{h}| \times |\MS_{h-1}|$ matrix of latent transition probabilities from step $h-1$ to step $h$ under action $a_{h-1}$.
\end{definition}

The following definition of an approximate belief state is analogous to the inductive definition of a true belief state (\cref{def:true-belief-state}); the only difference is that it updates based on a window of the $L$ most recent observations and actions $(x_{h-L+1:h},a_{h-L:h-1})$ rather than the entire history, and is additionally parametrized by a distribution $\cD$ (which represents the prior on the latent state at step $h-L$).

\begin{definition}\label{def:apx-belief}
For a window length $L>0$, any $h > L$, and prior $\cD \in \Delta(\MS_{h-L})$, the \emph{approximate belief state} is inductively defined as
\[\belapx_h(x_{h-L+1:h}, a_{h-L:h-1};\cD) := \BU_h(\belapx_{h-1}(x_{h-L+1:h-1}, a_{h-L:h-2};\cD),a_{h-1},x_h)\]
where for $L = 0$, $\belapx_h(\emptyset;\cD) := \cD$. For $h \leq L$, the approximate belief state is defined to coincide with the true belief state.
\end{definition}

\subsection{Decodability and $\gamma$-Observability}

For completeness, we include the definition of (perfect) $L$-step decodability from \citep{efroni2022provable}.

\begin{definition}[$L$-step decodable model \citep{efroni2022provable}]\label{def:perfectly-decodable}
    A POMDP is said to be $L$-step decodable if, for each timestep $h \in [H]$, there exists a deterministic mapping $\phi_h: \cX^{h-L+1:h} \times \cA^{h-L+1:h-1} \to \mathcal{S}_h$ such that for any admissible trajectory $\tau = (s,x,a)_{1:h}$ (i.e., a trajectory that occurs with positive probability under the uniformly random policy), we have $s_h = \phi_h\prn*{x_{h-L+1:h},a_{h-L+1:h-1}}$.
\end{definition}

Next, we introduce relevant definitions and results relating to $\gamma$-observable POMDPs from \citep{golowich2023planning,golowich2022learning}.

\begin{definition}[$\gamma$-observability \citep{golowich2023planning}]
    \label{def:observable}
    Let $\gamma \in (0,1]$. A POMDP is $\gamma$-observable if for any $h \in [H]$ and distributions $b, b' \in \Delta(\cS_h)$, it holds that
    \begin{align*}
        \nrm*{\bbO_h^{\top} b - \bbO_h^{\top} b'}_1 \geq \gamma \nrm*{b - b'}_1,
    \end{align*}
    where $\bbO_h \in \RR^{\cS_h \times \cX_h}$ is the observation matrix defined by $(\bbO_h)_{sx} := \bbO_h(x\mid{}s)$.
    \end{definition}

The algorithm of \cite{golowich2022learning} for learning $\gamma$-observable POMDPs in quasi-polynomial time requires bounding a slightly generalized version of the belief contraction error defined in \cref{def:belief-contraction-main}. We state this version below.

\begin{definition}[Generalized belief contraction \citep{golowich2022learning}]\label{def:gen-belief-contraction}
Let $\epsilon,\phi \in (0,1)$ and $L \in \NN$. We say that a POMDP $\cP$ satisfies $(\epsilon; \phi, L)$-belief contraction if the following property holds. Let $\pi$ be an executable policy, let $h \in \{L+1,\dots,H\}$, and let $\cD,\cD' \in \Delta(\cS_{h-L})$. If $\norm{\frac{\cD'}{\cD}}_\infty \leq 1/\phi$, then for any fixed history $(x_{1:h-L}, a_{1:h-L-1})$ it holds that
\[\EE_{s_{h-L} \sim \cD'} \EE^\pi\left[\norm{\belapx_h(x_{h-L+1:h}, a_{h-L:h-1}; \cD') - \belapx_h(x_{h-L+1:h}, a_{h-L:h-1}; \cD)}_1\right] \leq \epsilon\]
where the inner expectation is over partial trajectories $(x_{h-L+1:h}, a_{h-L:h-1})$ drawn from $\cP$ by initializing to latent state $s_{h-L}$ at step $h-L$, and sampling action $a_k \sim \pi(x_{1:k},a_{1:k-1})$ at each $h-L \leq k < h$.
\end{definition}

For context, \cite[Theorem B.2]{golowich2022learning} shows that $\gamma$-observability implies $(\epsilon;\phi,L)$-belief contraction for $L=O(\gamma^{-4}\log(1/(\epsilon\phi)))$. When $\phi\leq 1/S$, generalized belief contraction also bounds the error in \cref{def:belief-contraction-main}: take $\cD'=\belief_{h-L}(x_{1:h-L},a_{1:h-L-1})$ and $\cD=\unif(\cS_{h-L})$, so that $\norm{\cD'/\cD}_\infty\leq S\leq1/\phi$, and average over the initial history.

\begin{theorem}[Rewards depending on state and action]
\label{thm:golowich}
There is a universal constant $C^\star>0$ with the following property.
Given $\epsilon,\beta\in(0,1)$, $\gamma\in(0,1]$, $L\in\NN$, and a
$\gamma$-observable POMDP $\cP$, set
\[
 \phi:=\frac{\gamma\epsilon}{C^\star H^5S^{7/2}X^2}.
\]
If $\cP$ satisfies $(\epsilon;\phi,L)$-belief contraction, then
\texttt{BaSeCAMP} with empirical reward estimates returns an executable
policy $\pihat$ satisfying
\[
 J(\pistar)-J(\pihat)
 \leq\epsilon\cdot\poly(S,X,H,\gamma^{-1})
\]
with probability at least $1-\beta$, in time
$\poly((XA)^L,H,S,\epsilon^{-1},\gamma^{-1},\log(\beta^{-1}))$.
\end{theorem}

\begin{proof}
The rewards in our model are $R_h(s_h,a_h)\in[0,1]$, observed during
training. The policies below depend only on observations and actions.
We retain the transition estimation and exploration of
\citet{golowich2022learning}, and replace its recovery of an
observation's reward from one sample by an empirical mean for each
window and action. Throughout this proof, numbered results attributed
to that paper refer to its full version, including the appendix.

\noindent\textbf{Exploration and sampling.}
The prescribed window length in that analysis is used through its
Theorem B.2. Its conclusion is precisely the generalized contraction
assumption here. This also suffices in its Lemma F.3: the truncated
processes there have the same dynamics as $\cP$ throughout the window,
and histories already in the absorbing state contribute zero.
Consequently its exploration analysis applies with our $L$.
One may take $L\leq H$, since longer windows contain the full history.
An initial dummy stage and a terminal dummy stage reconcile the
observation/action timing conventions; both have singleton state,
observation, and action spaces and zero reward. These change the horizon by at
most a factor of three, absorbed by increasing $C^\star$. In particular,
our assumption holds at the threshold required by the cited analysis:
contraction at $\phi$ implies contraction at every larger threshold.
We use the original time indices below.

At exploration round $k$ and step $h$, write $\mu^{k,h}$ for its cover
policy, and let $d_t^\mu$ denote the latent-state distribution at step
$t$ under policy $\mu$. Write $\nu^{k,h}$ for the sampling policy which follows
$\mu^{k,h}$ through step $h-L-1$ and then chooses uniformly random
actions. For $h\leq L$ it chooses uniformly from the start.
Let $z_h=(x_{h-L+1:h},a_{h-L:h-1})$, with the full history used when
$h\leq L$, and set
\[
 \widetilde b_{k,h}(z_h)
 :=\belapx_h(z_h;d_{h-L}^{\mu^{k,h}}),\qquad
 \widetilde r_{k,h}(z_h,a)
 :=\sum_s\widetilde b_{k,h}(z_h)(s)R_h(s,a).
\]
For $h\leq L$, the first expression means the true belief.
By \cref{lemma:belapx-cond-prob,lemma:belief-is-cond-prob}, the conditional state distribution
under $\nu^{k,h}$ given $z_h$ is $\widetilde b_{k,h}(z_h)$.
The independent uniform action $a_h$ does not change it. Hence the
conditional mean of the observed reward given $(z_h,a_h)=(z,a)$ is
exactly $\widetilde r_{k,h}(z,a)$.

Use fictitious zero rewards during the cover construction, storing the
actual rewards separately. The barycentric-spanner routine uses its
own synthetic observation rewards, so this leaves exploration unchanged.
Afterwards, in each empirical MDP $\widehat M_k$, assign
$\widehat r_{k,h}(z,a)$ to be the average of the rewards from the
$\nu^{k,h}$ sampling batch matching $(z,a)$ if there are at least $N_1$ such samples, and zero
otherwise. Include $h=H$: collect $a_H$ and its reward even though
no subsequent transition is needed.

\noindent\textbf{Estimation and coverage bounds.}
Let $K$ be the number of exploration rounds, which is $O(HS)$.
In addition to the sample sizes of the cited paper, require
\[
 N_1\geq c\epsilon^{-2}
       \log\!\left(\frac{KH A(XA)^L}{\beta}\right)
\]
for a sufficiently large universal $c$, increasing $N_0$ according to
their Lemma E.1 whenever $N_1$ increases. Conditional on earlier rounds,
the samples in each batch are independent. Hoeffding's inequality,
conditional on the number and window/action labels of the matching
samples, and a union bound give, simultaneously over all rounds,
steps, and pairs with at least $N_1$ samples,
\[
 |\widehat r_{k,h}(z,a)-\widetilde r_{k,h}(z,a)|\leq\epsilon.
\]
Use the failure probability $\beta/4$ for this event
and $\beta/4$ for the transition/count event of their Lemma E.1.
The latter event ensures that every action has at least $N_1$ samples
at each window outside
\[
 \mathcal Z^{\mathrm{low}}_{k,h}
 :=\{z:\Pr^{\nu^{k,h}}(z_h=z)\leq\zeta\}.
\]
The same count bound applies at $h=H$.

For any executable policy $\pi$, define, when $h>L$,
\[
 U_{k,h}:=\{s:d_{h-L}^{\mu^{k,h}}(s)<\phi\},\qquad
 u_{k,h}^{\pi}:=\Pr^\pi(s_{h-L}\in U_{k,h}),
\]
and set $u_{k,h}^{\pi}=0$ for $h\leq L$.
Take the transition accuracy $\theta=\epsilon$ and
$\zeta=\epsilon\phi/(A^{2L}X^L)$, as in Section C.1 of
\citet{golowich2022learning}. Its Lemmas H.2 and F.4 give, respectively,
\begin{align}
 \EE^\pi\!\left[\norm{\bel_h-\widetilde b_{k,h}(z_h)}_1\right]
 &\leq\epsilon+6u_{k,h}^{\pi},\label{eq:reward-extension-belief}\\
 \Pr^\pi(z_h\in\mathcal Z^{\mathrm{low}}_{k,h})
 &\leq\epsilon+u_{k,h}^{\pi}.\label{eq:reward-extension-low}
\end{align}
The first bound uses only our generalized contraction assumption.
Moreover, its Lemma I.3 and the coverage conclusion (63) of Lemma I.4
ensure that some round $k$ satisfies, simultaneously for every $\pi,h$,
\[
 u_{k,h}^{\pi}\leq SH\xi,
 \qquad
 \xi=c_0X^2H^3\sqrt S\,\epsilon/\gamma,
\]
where $c_0$ is a universal constant. Here $\xi$ is the exploration
parameter called $\delta$ in that paper, unrelated to the perturbation
level of a Block MDP. These results concern transitions and cover
policies; the fictitious zero rewards satisfy their reward-estimation
conditions exactly.

\noindent\textbf{Value comparison for state/action rewards.}
Fix such a round $k$. At full history
$\tau_h=(x_{1:h},a_{1:h-1})$, the conditional mean reward is
$\overline r_h(\tau_h,a)=\sum_s\bel_h(\tau_h)(s)R_h(s,a)$.
Since $0\leq R_h(s,a)\leq1$,
\[
 |\overline r_h(\tau_h,a)-\widetilde r_{k,h}(z_h,a)|
 \leq\tfrac12\norm{\bel_h(\tau_h)-\widetilde b_{k,h}(z_h)}_1.
\]
Together with \cref{eq:reward-extension-belief,eq:reward-extension-low}
and reward concentration, this gives
\[
 e_{r,h}^{\pi}:=
 \EE^\pi|\overline r_h(\tau_h,a_h)-\widehat r_{k,h}(z_h,a_h)|
 \leq\tfrac52\epsilon+4u_{k,h}^{\pi}.
\]
Write $P_h(\cdot\mid\tau_h,a_h)$ for the true next-observation law
and $\widehat P_{k,h}(\cdot\mid z_h,a_h)$ for its empirical estimate,
including the algorithm's absorbing observation. The intermediate
transition law obtained from $\widetilde b_{k,h}$ differs from the
true law by at most the belief $\ell_1$ error. By their Lemma E.1 it
differs from the empirical law by at most
$\epsilon+2\mathbf1\{z_h\in\mathcal Z^{\mathrm{low}}_{k,h}\}$.
Consequently, for $h<H$,
\[
 e_{P,h}^{\pi}:=
 \EE^\pi\norm{P_h(\cdot\mid\tau_h,a_h)
            -\widehat P_{k,h}(\cdot\mid z_h,a_h)}_1
 \leq4\epsilon+8u_{k,h}^{\pi}.
\]

For the value comparison, initialize $\widehat M_k$ with the true law
of $x_1$. This distribution is used only in the analysis: backward
dynamic programming gives a policy optimal from every initial window
and does not require its distribution. View both processes on full
histories and telescope their Bellman recursions. Since the remaining
value in $\widehat M_k$ lies in
$[0,H-h]$, for every executable policy $\pi$ we obtain
\begin{align*}
 |J(\pi)-J_{\widehat M_k}(\pi)|
 &\leq\sum_{h=1}^H e_{r,h}^{\pi}
       +\frac12\sum_{h=1}^{H-1}(H-h)e_{P,h}^{\pi}\\
 &\leq4H^2(\epsilon+SH\xi)
 =:\Delta
 =\epsilon\cdot\poly(S,X,H,\gamma^{-1}).
\end{align*}
This derivation allows rewards depending on both $z_h$ and $a_h$;
it does not require them to be functions of the current observation.
If $\epsilon$ is large enough that the final polynomial bound exceeds
$1$, the normalized reward assumption gives the theorem trivially.
Otherwise the parameters above satisfy the small-accuracy conditions
of the cited exploration analysis, after increasing universal constants.

For each round, compute an optimal policy $\widehat\pi_k$ of its
empirical MDP by dynamic programming. A Markov policy is optimal there
even among full-history policies, so the good round satisfies
\[
 J(\pistar)-J(\widehat\pi_k)\leq2\Delta.
\]
Evaluate all candidates on fresh episodes to accuracy $\epsilon$ and
return the one with largest empirical value. With probability at
least $1-\beta/2$ these estimates are all accurate, giving regret at
most $2\Delta+2\epsilon$. This evaluation uses
$O(\epsilon^{-2}\log(K/\beta))$ episodes per candidate, since trajectory
rewards lie in $[0,1]$. The additional reward sampling and dynamic
programming are polynomial in the stated quantities, and none of the
executed policies conditions on reward history. Combining the failure
probabilities proves the claim.
\end{proof}

\section{Technical Lemmas}

\begin{lemma}[Data processing inequality]\label{lemma:dpi}
Let $\MS,\MT$ be finite sets, let $p,q \in \Delta(\MS)$ be distributions, and let $K: \MS \to \Delta(\MT)$ be a conditional distribution function. Then
\[\TV(K \circ p, K \circ q) \leq \TV(p, q).\]
Similarly, if $p \ll q$, then
\[\norm{\frac{K \circ p}{K \circ q}}_\infty \leq \norm{\frac{p}{q}}_\infty.\]
\end{lemma}

\begin{proof}
The first inequality follows from the fact that total variation distance is an $f$-divergence. The second inequality can be directly checked: for all $y \in \MT$,
\[(K\circ p)(y) = \sum_{x\in\MS} K(y\mid{} x) p(x) \leq \norm{\frac{p}{q}}_\infty \sum_{x\in\MS} K(y\mid{} x) q(x) = \norm{\frac{p}{q}}_\infty (K \circ q)(y)\]
as needed.
\end{proof}

Recall that a policy is \emph{executable} if the action distribution at any step is determined by the action/observation history (note that a latent policy therefore may not be executable). The following lemma, which was implicitly used in prior work \citep{golowich2023planning}, verifies under any executable policy, the conditional distribution of the latent state given the history is the true belief state.\footnote{Note that this is not true for all policies, since a latent policy could reveal the latent state (or, more generally, bias the conditional distribution) by its choice of action. This issue underpins the ``latching'' effect observed in behavior cloning of privileged experts \citep{swamy2022sequence}.}

\begin{lemma}\label{lemma:belief-is-cond-prob}
Fix any step $h \in [H]$ and executable policy $\pi$. Then
\[\bbP^\pi[s_h \mid{} x_{1:h}, a_{1:h-1}] = \bel_h(x_{1:h}, a_{1:h-1})(s_h)\]
and, if $h>1$,
\[\bbP^\pi[x_h \mid{} x_{1:h-1},a_{1:h-1}] = (\BO_h^\t \cdot \bbP_h(a_{h-1}) \cdot \bel_{h-1}(x_{1:h-1},a_{1:h-2}))(x_h)\]
Both identities hold whenever the conditioning history has positive probability under $\bbP^\pi$, for every latent state $s_h$ and observation $x_h$, respectively.
\end{lemma}

\begin{proof}
We prove the first claim by induction on $h$. It is clear that $\bbP^\pi[s_1\mid{}x_1] \propto \BO_1(x_1\mid{}s_1)\bbP_1(s_1) \propto \BB_1(\bbP_1; x_1)(s_1) = \bel_1(x_1)(s_1)$, where proportionality is up to factors independent of $s_1$. Since $\bbP^\pi[\cdot\mid{}x_1]$ and $\bel_1(x_1)$ are distributions, it follows from the proportionality that they are equal. Now fix any $h \in \{2,\dots,H\}$ and assume the claim holds for $h-1$. Let $(s_{1:h}, x_{1:h}, a_{1:h-1})$ be a random trajectory drawn from $\bbP^\pi$, i.e. generated via $s_h \sim \bbP_h(s_{h-1},a_{h-1})$, $x_h \sim \BO_h(s_h)$, and $a_h \sim \pi(x_{1:h}, a_{1:h-1})$ (since we assumed that $\pi$ is executable, the action distribution does not directly depend on $s_{1:h}$). Then,
\begin{align*}
\bbP^\pi[s_h \mid{} x_{1:h}, a_{1:h-1}]
&\propto \sum_{s_{1:h-1}} \bbP^\pi[s_{1:h}, x_{1:h}, a_{1:h-1}] \\
&= \sum_{s_{1:h-1}} \bbP^\pi[s_{1:h-1}, x_{1:h-1}, a_{1:h-2}]\pi_{h-1}(a_{h-1}\mid x_{1:h-1},a_{1:h-2}) \\
&\hspace{2em}\cdot\bbP_h(s_h\mid s_{h-1},a_{h-1})\BO_h(x_h\mid s_h) \\
&\propto \sum_{s_{1:h-1}} \bbP^\pi[s_{1:h-1}, x_{1:h-1}, a_{1:h-2}] \bbP[s_h\mid{}s_{h-1},a_{h-1}]\BO_h(x_h\mid{}s_h) \\
&= \BO_h(x_h\mid{}s_h)\sum_{s_{h-1}} \bbP[s_h\mid{}s_{h-1},a_{h-1}] \sum_{s_{1:h-2}} \bbP^\pi[s_{1:h-1}, x_{1:h-1}, a_{1:h-2}] \\
&\propto \BO_h(x_h\mid{}s_h)\sum_{s_{h-1}} \bbP[s_h\mid{}s_{h-1},a_{h-1}] \bbP^\pi[s_{h-1} \mid{} x_{1:h-1}, a_{1:h-2}] \\
&= \BO_h(x_h\mid{}s_h)\sum_{s_{h-1}} \bbP[s_h\mid{}s_{h-1},a_{h-1}] \bel_{h-1}(x_{1:h-1},a_{1:h-2})(s_{h-1}) \\
 &\propto \bel_h(x_{1:h}, a_{1:h-1})(s_h)
\end{align*}
where the penultimate equality uses the induction hypothesis and the final equality uses \cref{eq:belief-defn}. This proves the first claim. To prove the second claim, observe that by a similar argument to above, for any $h>1$,
\[\bbP^\pi[s_h\mid{}x_{1:h-1},a_{1:h-1}] \propto \sum_{s_{h-1}} \bbP[s_h\mid{}s_{h-1},a_{h-1}] \bel_{h-1}(x_{1:h-1},a_{1:h-2})(s_{h-1})\]
so that $\bbP^\pi[s_h\mid{}x_{1:h-1},a_{1:h-1}] = (\bbP_h(a_{h-1}) \cdot \bel_{h-1}(x_{1:h-1},a_{1:h-2}))(s_h)$. But then
\begin{align*}
\bbP^\pi[x_h\mid{}x_{1:h-1},a_{1:h-1}]
&\propto \sum_{s_{1:h}} \bbP^\pi[s_{1:h},x_{1:h},a_{1:h-1}] \\
&= \sum_{s_h} \BO_h(x_h\mid{}s_h) \sum_{s_{1:h-1}} \bbP^\pi[s_{1:h},x_{1:h-1},a_{1:h-1}] \\
&\propto \sum_{s_h} \BO_h(x_h\mid{}s_h) \bbP^\pi[s_h\mid{}x_{1:h-1},a_{1:h-1}].
\end{align*}
Therefore $\bbP^\pi[x_h \mid{} x_{1:h-1}, a_{1:h-1}] = (\BO_h^\t \cdot \bbP_h(a_{h-1}) \cdot \bel_{h-1}(x_{1:h-1},a_{1:h-2}))(x_h)$ as claimed.
\end{proof}

We will also need the following variant of \cref{lemma:belief-is-cond-prob}, which shows how approximate belief states arise as conditional probability distributions:

\begin{lemma}\label{lemma:belapx-cond-prob}
Fix any $h \in [H]$, $L \in \{0,\dots,h-1\}$, and executable policy $\pi$ where $\pi_{h-t}(\cdot\mid{}x_{1:h-t},a_{1:h-t-1})$ is determined by $(x_{h-L+1:h-t},a_{h-L:h-t-1})$ for all $t \in [L]$. Then, almost surely under $\bbP^\pi$,
\[\bbP^\pi[s_h \mid{} x_{h-L+1:h}, a_{h-L:h-1}] = \belapx_h(x_{h-L+1:h},a_{h-L:h-1}; d^\pi_{h-L})(s_h).\]
\end{lemma}

\begin{proof}
We induct on $L$. If $L = 0$ then, for any $h \in [H]$, by definition $\bbP^\pi[s_h] = d^\pi_h(s_h) = \belapx_h(\emptyset; d^\pi_h)(s_h)$. Fix any $L>0$ and suppose the claim holds for $L-1$ (for all $h > L-1$). Then for any $h>L$,
\begin{align*}
&\bbP^\pi[s_h \mid{} x_{h-L+1:h}, a_{h-L:h-1}]\\
&\propto \sum_{s_{h-L:h-1}} \bbP^\pi[s_{h-L:h}, x_{h-L+1:h}, a_{h-L:h-1}] \\
&= \sum_{s_{h-L:h-1}} \bbP^\pi[s_{h-L:h-1}, x_{h-L+1:h-1}, a_{h-L:h-2}] \pi(a_{h-1} \mid{} x_{h-L+1:h-1}, a_{h-L:h-2}) \bbP_h[s_h \mid{} s_{h-1}, a_{h-1}] \BO_h(x_h \mid{} s_h) \\
&\propto \sum_{s_{h-L:h-1}} \bbP^\pi[s_{h-L:h-1}, x_{h-L+1:h-1}, a_{h-L:h-2}] \bbP_h[s_h \mid{} s_{h-1}, a_{h-1}] \BO_h(x_h \mid{} s_h) \\
&= \BO_h(x_h\mid{} s_h) \sum_{s_{h-1}} \bbP_h[s_h \mid{} s_{h-1}, a_{h-1}] \bbP^\pi[s_{h-1}, x_{h-L+1:h-1}, a_{h-L:h-2}] \\
&\propto \BO_h(x_h\mid{} s_h) \sum_{s_{h-1}} \bbP_h[s_h \mid{} s_{h-1}, a_{h-1}] \bbP^\pi[s_{h-1} \mid{} x_{h-L+1:h-1}, a_{h-L:h-2}] \\
&= \BO_h(x_h\mid{} s_h) \sum_{s_{h-1}} \bbP_h[s_h \mid{} s_{h-1}, a_{h-1}] \belapx_{h-1}(x_{h-L+1:h-1}, a_{h-L:h-2};d^\pi_{h-L})(s_{h-1}) \\
&\propto \belapx_h(x_{h-L+1:h},a_{h-L:h-1}; d^\pi_{h-L})(s_h)
\end{align*}
by the induction hypothesis and the definition of $\belapx_h(x_{h-L+1:h},a_{h-L:h-1}; d^\pi_{h-L})$.
\end{proof}

The following fact is well-known.

\begin{lemma}\label{lemma:mdp_upper_pomdp}
Let $\pi^\ast$ be the optimal policy of the POMDP, and let $\pi^\latent$ be the optimal policy of the MDP. Then we have that
\begin{align*}
    J(\pi^\ast) \leq J(\pi^\latent).
\end{align*}
\end{lemma}

\begin{proof}
Write $Q_h^{\cP}$ and $Q_h^{\cM}$ for the optimal POMDP and MDP
Q-value functions, respectively. For $h<H$, the POMDP Bellman equation
at an admissible history is
\begin{align*}
 Q^{\cP}_h((x_{1:h},a_{1:h-1}),a_h)
 &=\sum_{s_h}\belief_h(x_{1:h},a_{1:h-1})(s_h)R_h(s_h,a_h)\\
 &\quad+\sum_{x_{h+1}}P(x_{h+1}\mid x_{1:h},a_{1:h})
       \max_{a_{h+1}}Q^{\cP}_{h+1}((x_{1:h+1},a_{1:h}),a_{h+1}),
\end{align*}
where
\[
 P(x_{h+1}\mid x_{1:h},a_{1:h})
 =\sum_{s_h,s_{h+1}}\belief_h(x_{1:h},a_{1:h-1})(s_h)
   \bbP_{h+1}(s_{h+1}\mid s_h,a_h)\bbO_{h+1}(x_{h+1}\mid s_{h+1}).
\]
The history $(x_{1:h},a_{1:h-1})$ can be summarized by its belief state.
Thus, for an admissible belief $b_h\in\Delta(\cS_h)$, we write
\begin{align*}
 Q^{\cP}_h(b_h,a_h)
 &=\sum_{s_h}b_h(s_h)R_h(s_h,a_h)\\
 &\quad+\sum_{x_{h+1}}P(x_{h+1}\mid b_h,a_h)
       \max_{a_{h+1}}Q^{\cP}_{h+1}(b_{h+1},a_{h+1}),
 \qquad h<H,
\end{align*}
where $b_{h+1}:=\BU_{h+1}(b_h;a_h,x_{h+1})$. Similarly, define
\[
 Q^{\cM}_h(b_h,a_h):=\sum_{s_h}b_h(s_h)Q^{\cM}_h(s_h,a_h).
\]
At the final step, both Q-values equal the expected immediate reward.
For optimal executable and latent policies, respectively,
\[
 J(\pi^\ast)=\En_{x_1}\brk*{\max_{a_1}Q^{\cP}_1(\belief_1(x_1),a_1)},
 \qquad
 J(\pi^\latent)=\En_{s_1}\brk*{\max_{a_1}Q^{\cM}_1(s_1,a_1)}.
\]

We prove by backward induction that, for every $h\in[H]$, admissible
belief $b_h$, and action $a_h$,
\[
 Q^{\cM}_h(b_h,a_h)\geq Q^{\cP}_h(b_h,a_h).
\]
For $h=H$,
\[
 Q^{\cM}_H(b_H,a_H)=\sum_{s_H}b_H(s_H)R_H(s_H,a_H)
 =Q^{\cP}_H(b_H,a_H).
\]
Now fix $h<H$ and assume the claim at step $h+1$. Then
\begin{align*}
 Q^{\cP}_h(b_h,a_h)
 &=\sum_{s_h}b_h(s_h)R_h(s_h,a_h)
   +\sum_{x_{h+1}}P(x_{h+1}\mid b_h,a_h)
       \max_{a_{h+1}}Q^{\cP}_{h+1}(b_{h+1},a_{h+1})\\
 &\leq\sum_{s_h}b_h(s_h)R_h(s_h,a_h)
   +\sum_{x_{h+1}}P(x_{h+1}\mid b_h,a_h)
       \max_{a_{h+1}}Q^{\cM}_{h+1}(b_{h+1},a_{h+1})\\
 &\leq\sum_{s_h}b_h(s_h)R_h(s_h,a_h)\\
 &\quad+\sum_{x_{h+1}}P(x_{h+1}\mid b_h,a_h)
       \sum_{s_{h+1}}b_{h+1}(s_{h+1})
       \max_{a_{h+1}}Q^{\cM}_{h+1}(s_{h+1},a_{h+1}).
\end{align*}
For any function $f$ of the next state, the Bayes update gives
\begin{align*}
 &\sum_{x_{h+1}}P(x_{h+1}\mid b_h,a_h)
       \sum_{s_{h+1}}b_{h+1}(s_{h+1})f(s_{h+1})\\
 &\quad=\sum_{x_{h+1},s_{h+1}}\bbO_{h+1}(x_{h+1}\mid s_{h+1})
       \sum_{s_h}\bbP_{h+1}(s_{h+1}\mid s_h,a_h)b_h(s_h)f(s_{h+1})\\
 &\quad=\sum_{s_h,s_{h+1}}\bbP_{h+1}(s_{h+1}\mid s_h,a_h)
       b_h(s_h)f(s_{h+1}).
\end{align*}
Observations of zero probability contribute zero to these sums.
Applying this identity with $f(s)=\max_aQ^{\cM}_{h+1}(s,a)$ yields
\begin{align*}
 Q^{\cP}_h(b_h,a_h)
 &\leq\sum_{s_h}b_h(s_h)R_h(s_h,a_h)\\
 &\quad+\sum_{s_h,s_{h+1}}\bbP_{h+1}(s_{h+1}\mid s_h,a_h)b_h(s_h)
       \max_{a_{h+1}}Q^{\cM}_{h+1}(s_{h+1},a_{h+1})\\
 &=\sum_{s_h}b_h(s_h)\left(R_h(s_h,a_h)
   +\sum_{s_{h+1}}\bbP_{h+1}(s_{h+1}\mid s_h,a_h)
       \max_{a_{h+1}}Q^{\cM}_{h+1}(s_{h+1},a_{h+1})\right)\\
 &=\sum_{s_h}b_h(s_h)Q^{\cM}_h(s_h,a_h)
 =Q^{\cM}_h(b_h,a_h).
\end{align*}
Finally, conditioning the initial state on $x_1$ and using the comparison
just proved,
\begin{align*}
 J(\pi^{\latent})
 &=\En_{x_1}\brk*{\sum_{s_1}\belief_1(x_1)(s_1)
                         \max_{a_1}Q^{\cM}_1(s_1,a_1)}\\
 &\geq\En_{x_1}\brk*{\max_{a_1}Q^{\cM}_1(\belief_1(x_1),a_1)}\\
 &\geq\En_{x_1}\brk*{\max_{a_1}Q^{\cP}_1(\belief_1(x_1),a_1)}
 =J(\pi^\ast).
\end{align*}
\end{proof}

We will use the following martingale bound to analyze decodability error in the perturbed Block MDP.

\begin{lemma}\label{lemma:martingale-bound-improved}
Fix $\epsilon\in(0,1)$, $S>0$, and an integer $L\geq0$.
Let $(X_n,\mathcal F_n)_{n=0}^L$ be a non-negative supermartingale with
$\EE[X_0]\leq S$ and
\[
 \Pr[X_{n+1}>\epsilon X_n\mid X_n]\leq\epsilon
\]
almost surely for all $0\leq n<L$. Define
\[
 \rho(\epsilon):=\sqrt{\epsilon}
 \left(1-\epsilon+\sqrt{1-\epsilon+\epsilon^2}\right).
\]
Then $\rho(\epsilon)<1$ and
\[
 \EE[\min(X_L,S)]\leq S\rho(\epsilon)^L.
\]
In particular,
\[
 \EE[\min(X_L,S)]\leq S\min\{1,(2\sqrt{\epsilon})^L\}.
\]
\end{lemma}

\begin{proof}
For $L=0$, the claimed bound follows from
$\EE[\min(X_0,S)]\leq\EE[X_0]\leq S$. Suppose henceforth that $L\geq1$.
We first prove the one-step square-root contraction
\begin{equation}\label{eq:martingale-sharp-sqrt}
 \EE[\sqrt{X_{n+1}}\mid X_n]
 \leq\rho(\epsilon)\sqrt{X_n}
 \qquad(0\leq n<L).
\end{equation}

\noindent\textbf{Step 1: condition on the current value.}
Since $X_n$ is $\mathcal F_n$-measurable, the supermartingale assumption
and the tower property give
\[
 \EE[X_{n+1}\mid X_n]
 =\EE\big[\EE[X_{n+1}\mid\mathcal F_n]\mid X_n\big]
 \leq X_n.
\]
Fix $n<L$. For almost every value $x$ of $X_n$, let $Y$ have the
conditional law of $X_{n+1}$ given $X_n=x$. Thus
\[
 \EE[Y]\leq x,
 \qquad \Pr[Y>\epsilon x]\leq\epsilon.
\]
All expectations and probabilities involving $Y$ below refer to this
conditional law. If $x=0$, nonnegativity forces $Y=0$ almost surely,
so \cref{eq:martingale-sharp-sqrt} holds. We therefore assume $x>0$.
Set
\[
 \beta:=\frac{1-\epsilon+\epsilon^2}{\epsilon}>1,
 \qquad (1-\epsilon)\epsilon+\epsilon\beta=1.
\]

\par\medskip\noindent\textbf{Step 2: compare capped expectations.}
For every $t\geq0$, we claim that
\begin{equation}\label{eq:martingale-sharp-capped}
 \EE[\min(Y,t)]
 \leq(1-\epsilon)\min(\epsilon x,t)
       +\epsilon\min(\beta x,t).
\end{equation}
There are three cases.
\begin{enumerate}
\item If $t\leq\epsilon x$, the right-hand side equals $t$, and
$\min(Y,t)\leq t$.
\item If $\epsilon x<t<\beta x$, write
$r:=\Pr[Y>\epsilon x]\leq\epsilon$. On the event $Y\leq\epsilon x$,
we have $\min(Y,t)\leq\epsilon x$; on its complement, we have
$\min(Y,t)\leq t$. Consequently,
\begin{align*}
 \EE[\min(Y,t)]
 &\leq(1-r)\epsilon x+rt\\
 &=\epsilon x+r(t-\epsilon x)\\
 &\leq\epsilon x+\epsilon(t-\epsilon x)\\
 &=(1-\epsilon)\epsilon x+\epsilon t,
\end{align*}
which equals the right-hand side of \cref{eq:martingale-sharp-capped}.
\item If $t\geq\beta x$, the right-hand side equals
$(1-\epsilon)\epsilon x+\epsilon\beta x=x$, while
$\EE[\min(Y,t)]\leq\EE[Y]\leq x$.
\end{enumerate}
This proves \cref{eq:martingale-sharp-capped} for every $t\geq0$.

\par\medskip\noindent\textbf{Step 3: pass to square roots.}
For every $y\geq0$,
\[
 \sqrt y=\frac14\int_0^\infty\min(y,t)t^{-3/2}\,dt.
\]
Indeed, for $y>0$, splitting the integral at $y$ gives
$\int_0^y t^{-1/2}\,dt+y\int_y^\infty t^{-3/2}\,dt=4\sqrt y$;
the case $y=0$ is immediate. Integrate
\cref{eq:martingale-sharp-capped} against $\frac14t^{-3/2}\,dt$.
Tonelli's theorem applies because the integrands are nonnegative, giving
\begin{align*}
 \EE[\sqrt Y]
 &\leq(1-\epsilon)\sqrt{\epsilon x}
       +\epsilon\sqrt{\beta x}\\
 &=\sqrt{\epsilon}
   \left(1-\epsilon+\sqrt{1-\epsilon+\epsilon^2}\right)\sqrt x\\
 &=\rho(\epsilon)\sqrt x.
\end{align*}
This establishes \cref{eq:martingale-sharp-sqrt}.

\par\medskip\noindent\textbf{Step 4: iterate and clip.}
Taking expectations in \cref{eq:martingale-sharp-sqrt} and iterating,
\[
 \EE[\sqrt{X_L}]\leq\rho(\epsilon)^L\EE[\sqrt{X_0}].
\]
For every $y\geq0$, $\min(y,S)\leq\sqrt{Sy}$.
Jensen's inequality and $\EE[X_0]\leq S$ therefore imply
\begin{align*}
 \EE[\min(X_L,S)]
 &\leq\sqrt S\,\EE[\sqrt{X_L}]\\
 &\leq\sqrt S\,\rho(\epsilon)^L\EE[\sqrt{X_0}]\\
 &\leq\sqrt S\,\rho(\epsilon)^L\sqrt{\EE[X_0]}\\
 &\leq S\rho(\epsilon)^L.
\end{align*}
Finally, let $Z$ equal $\epsilon$ with probability $1-\epsilon$ and
$\beta$ with probability $\epsilon$. Then $\EE[Z]=1$ and
$\EE[\sqrt Z]=\rho(\epsilon)$. Since $Z$ takes two distinct values
with positive probabilities, strict concavity of the square root gives
$\rho(\epsilon)<\sqrt{\EE[Z]}=1$.
Also $1-\epsilon\leq1$ and $1-\epsilon+\epsilon^2\leq1$, so
$\rho(\epsilon)\leq2\sqrt\epsilon$. These two inequalities imply
the final bound in the statement.
\end{proof}

\section{Omitted Proofs for Perturbed Block MDP}

Below, we restate the definition of a $\delta$-perturbed Block MDP. We then prove our main theoretical results for the perturbed Block MDP. In \cref{subsec:belief-contraction}, we prove \cref{thm:belief-contraction-main} (the belief contraction result, restated as \cref{thm:belief-contraction}). In \cref{subsec:decodability}, we prove \cref{prop:det-vinf-decay-main} (the decodability result for deterministic dynamics, restated as \cref{prop:vinf-decay}). Finally, in \cref{subsec:misspec-lb}, we prove \cref{prop:stoch-decodability-lb-main} (the misspecification lower bound for stochastic dynamics, restated as \cref{prop:stoch-decodability-lb}).

\begin{definition}\label{ass:perturbed-block}
Fix a parameter $\delta\in[0,1]$. A POMDP $\cP$ is a \emph{$\delta$-perturbed Block MDP} if, for each $h \in [H]$, there are $\BOT_h,E_h: \cS_h\to\Delta(\cX_h)$ such that $\BOT_h: \cS_h\to\Delta(\cX_h)$ satisfies the \emph{block} property \citep{du2019latent}, i.e. $\BOT_h(\cdot\mid{}s_h), \BOT_h(\cdot\mid{}s_h')$ have disjoint supports for all $s_h \neq s'_h$, and moreover the emission distribution $\BO_h$ at step $h$ can be decomposed as follows:
\arxiv{\[\BO_h(x_h\mid{}s_h) = (1-\delta) \BOT_h(x_h\mid{} s_h) + \delta E_h(x_h\mid{}s_h).\]}\iclr{$\BO_h(x_h\mid{}s_h) = (1-\delta) \BOT_h(x_h\mid{} s_h) + \delta E_h(x_h\mid{}s_h).$}

For notational convenience, for each $x \in \cX_h$ let $\phi(x) \in \cS_h$ be the unique state for which
$\BOT_h(x\mid{}\phi(x)) > 0$ (or arbitrary, if no such state exists).
\end{definition}

\begin{definition}
For any $h \in [H]$, $b \in \Delta(\MS_h)$, and $x_h \in \MO_h$, we write
\[\BO_h(x_h\mid{}b) := \sum_{z_h\in\MS_h} \BO_h(x_h\mid{}z_h) b(z_h).\]
We similarly define $E_h(x_h\mid{}b)$ and $\BOT_h(x_h\mid{}b)$.
\end{definition}

Notice that $\BOT_h(x_h\mid{}b) = b(\phi(x_h)) \BOT_h(x_h\mid{}\phi(x_h))$.

\subsection{Belief Contraction}\label{subsec:belief-contraction}

In this section, we prove \cref{thm:belief-contraction}, a slight generalization of \cref{thm:belief-contraction-main}. The proof broadly follows the proof of belief contraction for $\gamma$-observable POMDPs \citep{golowich2023planning} (of which $\delta$-perturbed Block MDPs with $\delta<1/2$ are a special case --- see \cref{remark:perturbed-block-golowich}), but since we require a stronger bound, we must modify the argument.

The basic idea (and main technical difficulty) in \cite{golowich2023planning} is to identify a monotonic transform of an $f$-divergence that multiplicatively contracts in expectation under the Bayes operator (\cref{def:bel-update}). In their case, they show that $\sqrt{\Dkl{\BB_h(b;x_h)}{\BB_h(b';x_h)}}$ contracts by roughly a constant factor (relative to $\sqrt{\Dkl{b}{b'}}$) in expectation over $x_h \sim \BO_h(\cdot\mid{}b)$. That is, updating the true belief and approximate belief by an observation drawn from the true belief tends to decrease the $\mathsf{KL}$-divergence. Updating the two beliefs by applying a transition matrix cannot increase the $\mathsf{KL}$-divergence since it is an $f$-divergence, so an iterative argument (alternating between observation updates and transition updates) proves exponential contraction of the belief error.

However, we would like to prove contraction by $\poly(\delta)$ per step, and the following example seems to present an obstacle to proving such contraction via $\mathsf{KL}$-divergence. It also presents an obstacle to directly analyzing $\TV$-distance.

\begin{example}[Failure of contraction of $\TV$ and $\mathsf{KL}$]\label{example:tv-not-contract}
Fix $0<\delta\leq1/2$. Let $\cS = \cX = \{0,1\}$ and let $\BO: \cS \to \Delta(\cX)$ be defined by $\BO(x\mid{}x) = 1-\delta$. Define $b = (1,0)$ and $b' = (\delta^2, 1-\delta^2)$. Then it holds almost surely over $x \sim \BO(\cdot\mid{}b)$ that:
\begin{itemize}
\item $\TV(\BB(b;x),\BB(b';x)) \geq 1-\delta$ even though $\TV(b,b') \leq 1$.
\item $\Dkl{\BB(b;x)}{\BB(b';x)} \geq \log(1/\delta)$ even though $\Dkl{b}{b'} \leq 2\log(1/\delta)$.
\end{itemize}
\end{example}

To control both discrepancies, we study the product of total variation and the maximum density ratio. Its square root contracts by a factor $\sqrt{8\delta}$ in expectation under the Bayes operator (\cref{lemma:dpr-bayes-update}), which allows us to iterate the one-step bound directly.

\begin{definition}
For distributions $b,b' \in \Delta(\MS)$ with $b \ll b'$, we define
\[\Dpr{b}{b'} := \TV(b,b') \cdot \norm{\frac{b}{b'}}_\infty.\]
\end{definition}

This quantity upper bounds $\TV$-distance and is at most $\norm{b/b'}_\infty$. Both factors satisfy the data processing inequality (\cref{lemma:dpi}), so their product and its square root cannot increase under a common transition kernel, including a stochastic one.

\begin{lemma}\label{lemma:inf-bayes-update}
Fix $h\in[H]$, $b,b'\in\Delta(\MS_h)$ with $b\ll b'$, and an
observation $x$ with $\BO_h(x\mid b)>0$. Then
$\BO_h(x\mid b')>0$, both posteriors are well-defined, and
\[
 \norm{\frac{\BB_h(b;x)}{\BB_h(b';x)}}_\infty
 \leq \frac{\BO_h(x\mid b')}{\BO_h(x\mid b)}
       \norm{\frac b{b'}}_\infty.
\]
\end{lemma}
\begin{proof}
Let $\kappa:=\norm{b/b'}_\infty$, $P:=\BO_h(x\mid b)$, and
$Q:=\BO_h(x\mid b')$. The pointwise inequality $b\leq\kappa b'$
gives $0<P\leq\kappa Q$, so $Q>0$ and $\kappa Q/P\geq1$.
For every state $s$, Bayes' rule gives
\[
 \BB_h(b;x)(s)
 =\frac{\BO_h(x\mid s)b(s)}P
 \leq\kappa\frac QP\,\BB_h(b';x)(s).
\]
Taking the maximum density ratio proves the claim, including coordinates
where both posterior probabilities vanish under the convention $0/0=1$.
Only an inequality is asserted: an observation with zero likelihood at a
state maximizing $b/b'$ can remove that maximizer from the posterior support.
\end{proof}

\begin{lemma}\label{lemma:dpr-bayes-update}
Suppose $\cP$ is a $\delta$-perturbed Block MDP, with $0\leq\delta\leq1$.
Fix $h\in[H]$ and $b,b'\in\Delta(\MS_h)$ with $b\ll b'$.
For $x\sim\BO_h(\cdot\mid b)$,
\[
 \EE\left[\sqrt{\Dpr{\BB_h(b;x)}{\BB_h(b';x)}}\right]
 \leq\sqrt{8\delta}\,\sqrt{\Dpr{b}{b'}}.
\]
Here and in the proof, $\TV(p,q)=\frac12\norm{p-q}_1$.
\end{lemma}
\begin{proof}
If $b=b'$, the posteriors coincide almost surely and the claim is immediate.
Otherwise, define
\[
 T:=\TV(b,b')>0,\qquad \kappa:=\norm{b/b'}_\infty,
 \qquad w(s):=\frac{|b(s)-b'(s)|}{2T}.
\]
Thus $w$ is a probability distribution and $\Dpr{b}{b'}=T\kappa$.
For an observation with $P(x):=\BO_h(x\mid b)>0$, write
$Q(x):=\BO_h(x\mid b')>0$, $b_x:=\BB_h(b;x)$, and
$b'_x:=\BB_h(b';x)$. All expressions below are restricted to these
observations, which have total probability one under $P$.

The differences $b_x(s)-b'_x(s)$ sum to zero. Consequently, the absolute
difference at $\phi(x)$ is at most the sum of the absolute differences
at the other states, and hence
\[
 \TV(b_x,b'_x)\leq
 \sum_{s\ne\phi(x)}|b_x(s)-b'_x(s)|.
\]
For $s\ne\phi(x)$, the block property gives
$\BO_h(x\mid s)=\delta E_h(x\mid s)$. Suppressing the argument $x$
in $P$ and $Q$, we obtain
\begin{align*}
 Q\TV(b_x,b'_x)
 &\leq\delta\sum_{s\ne\phi(x)}E_h(x\mid s)
       \left|\frac QP b(s)-b'(s)\right|\\
 &\leq\delta\left(2T E_h(x\mid w)
       +\frac{E_h(x\mid b)}P|Q-P|\right)\\
 &\leq2\delta T\left(E_h(x\mid w)
       +\frac{E_h(x\mid b)\BO_h(x\mid w)}P\right).
\end{align*}
The last step uses
$|Q-P|\leq\sum_s\BO_h(x\mid s)|b(s)-b'(s)|
=2T\BO_h(x\mid w)$.
By \cref{lemma:inf-bayes-update},
$\norm{b_x/b'_x}_\infty\leq\kappa Q/P$. Therefore
\[
 \frac{\Dpr{b_x}{b'_x}}{\Dpr{b}{b'}}
 \leq2\delta\bigl(U(x)+V(x)W(x)\bigr),
\]
where
\[
 U(x):=\frac{E_h(x\mid w)}{P(x)},\qquad
 V(x):=\frac{E_h(x\mid b)}{P(x)},\qquad
 W(x):=\frac{\BO_h(x\mid w)}{P(x)}.
\]
Since $E_h(\cdot\mid w)$, $E_h(\cdot\mid b)$, and
$\BO_h(\cdot\mid w)$ are distributions, summing over $P(x)>0$ gives
$\EE_P[U]\leq1$, $\EE_P[V]\leq1$, and $\EE_P[W]\leq1$.
Using $\sqrt{u+v}\leq\sqrt u+\sqrt v$, Jensen's inequality, and
Cauchy--Schwarz, we conclude that
\begin{align*}
 \EE_P\left[\sqrt{\frac{\Dpr{b_x}{b'_x}}{\Dpr{b}{b'}}}\right]
 &\leq\sqrt{2\delta}
       \left(\EE_P[\sqrt U]+\EE_P[\sqrt{VW}]\right)\\
 &\leq\sqrt{2\delta}
       \left(\sqrt{\EE_P[U]}+
       \sqrt{\EE_P[V]\EE_P[W]}\right)
 \leq\sqrt{8\delta}.
\end{align*}
Multiplying by $\sqrt{\Dpr{b}{b'}}$ proves the claim.
\end{proof}

The following result combines square-root contraction under an observation update with data processing under the transition kernel. It follows from \cref{lemma:dpr-bayes-update} and the data processing inequality.

\begin{corollary}\label{cor:dpr-update}
Suppose $\cP$ is a $\delta$-perturbed Block MDP, with $0\leq\delta\leq1$.
Fix $h\in\{2,\ldots,H\}$, $b,b'\in\Delta(\MS_{h-1})$ with
$b\ll b'$, and $a_{h-1}\in\MA_{h-1}$. For
$x_h\sim(\BO_h)^\top\bbP_h(a_{h-1})b$,
\[
 \EE\left[\sqrt{
   \Dpr{\BU_h(b;a_{h-1},x_h)}{\BU_h(b';a_{h-1},x_h)}}\right]
 \leq\sqrt{8\delta}\,\sqrt{\Dpr{b}{b'}}.
\]
The latent transition kernel may be stochastic.
\end{corollary}
\begin{proof}
Put $K:=\bbP_h(a_{h-1})$. The data processing inequalities in
\cref{lemma:dpi} give $Kb\ll Kb'$ and
\[
 \Dpr{Kb}{Kb'}
 =\TV(Kb,Kb')\norm{Kb/Kb'}_\infty
 \leq\TV(b,b')\norm{b/b'}_\infty
 =\Dpr{b}{b'}.
\]
Apply \cref{lemma:dpr-bayes-update} to $Kb,Kb'$ and use
$\BU_h(b;a_{h-1},x_h)=\BB_h(Kb;x_h)$ to obtain
\[
 \EE\left[\sqrt{
   \Dpr{\BU_h(b;a_{h-1},x_h)}{\BU_h(b';a_{h-1},x_h)}}\right]
 \leq\sqrt{8\delta}\,\sqrt{\Dpr{Kb}{Kb'}}
 \leq\sqrt{8\delta}\,\sqrt{\Dpr{b}{b'}}.
\]
\end{proof}

We now prove our main belief contraction result by iterating \cref{cor:dpr-update}. We first compare two beliefs initialized from arbitrary priors, then specialize to the true belief conditioned on the preceding history. The main-text result, \cref{thm:belief-contraction-main}, follows by converting total variation to $\ell_1$ distance.

\begin{theorem}\label{thm:belief-contraction}
Suppose $\cP$ is a $\delta$-perturbed Block MDP, with $0\leq\delta\leq1$.
Fix an executable policy $\pi$, an integer $L\geq1$, indices
$1\leq h<h+L\leq H$, and distributions
$\cD',\cD\in\Delta(\MS_h)$ with $\cD'\ll\cD$.
For any fixed partial history $\tau=(x_{1:h},a_{1:h-1})$,
\begin{equation}\label{eq:belief-contraction}
\begin{aligned}
 &\EE_{s_h\sim\cD'}\EE^\pi\Bigl[
 \TV\bigl(\belapx_{h+L}(x_{h+1:h+L},a_{h:h+L-1};\cD'),\,
          \belapx_{h+L}(x_{h+1:h+L},a_{h:h+L-1};\cD)\bigr)
 \,\Bigm|\,s_h\Bigr] \\
 &\qquad\leq\min\left\{1,\;
 (8\delta)^{L/2}\sqrt{\Dpr{\cD'}{\cD}}\right\}.
\end{aligned}
\end{equation}
Here the inner expectation is over the continuation under $\pi$ with the
environment initialized at $s_h$ at step $h$, holding $\tau$ fixed and
without resampling $x_h$.

As a consequence, for any partial history $\tau$ of positive probability
under $\pi$ with $\belief_h(\tau)\ll\cD$,
\begin{align*}
 &\EE^\pi\Bigl[
 \TV\bigl(\belief_{h+L}(x_{1:h+L},a_{1:h+L-1}),\,
          \belapx_{h+L}(x_{h+1:h+L},a_{h:h+L-1};\cD)\bigr)
 \,\Bigm|\,\tau\Bigr] \\
 &\qquad\leq\min\left\{1,\;
 (8\delta)^{L/2}\sqrt{\Dpr{\belief_h(\tau)}{\cD}}\right\}.
\end{align*}
\end{theorem}
\begin{proof}
Fix the initial history $\tau$ and generate a continuation by sampling
$s_h\sim\cD'$ and then rolling out $\pi$ as in the statement.
For $0\leq t\leq L$, abbreviate the two beliefs by
\[
 u_t:=\belapx_{h+t}(x_{h+1:h+t},a_{h:h+t-1};\cD'),\qquad
 v_t:=\belapx_{h+t}(x_{h+1:h+t},a_{h:h+t-1};\cD).
\]
Let $\mathcal F_t$ be generated by the observations
$x_{h+1:h+t}$ and actions $a_{h:h+t-1}$ in this continuation.
Then $u_0=\cD'$ and $v_0=\cD$.
The conditional-belief identity in \cref{lemma:belief-is-cond-prob},
applied to this initialized process, shows that $u_t$ is the conditional
law of $s_{h+t}$ given $\mathcal F_t$. Since $\pi$ is executable,
conditioning additionally on the chosen action $a_{h+t}$ does not change
that conditional law. Thus the next observation has conditional law
$(\BO_{h+t+1})^\top\bbP_{h+t+1}(a_{h+t})u_t$.

Domination is preserved by both the transition and the posterior update,
so $u_t\ll v_t$ almost surely for every $t$. For
$Z_t:=\sqrt{\Dpr{u_t}{v_t}}$, \cref{cor:dpr-update},
first conditioned on $\mathcal F_t,a_{h+t}$ and then averaged over the
action, gives
\[
 \EE[Z_{t+1}\mid\mathcal F_t]
 \leq\sqrt{8\delta}\,Z_t.
\]
Iteration yields
\[
 \EE[Z_L]\leq(8\delta)^{L/2}
                  \sqrt{\Dpr{\cD'}{\cD}}.
\]
For any $p\ll q$, $\TV(p,q)\leq1$ and
$\norm{p/q}_\infty\geq1$, so
$\TV(p,q)^2\leq\Dpr{p}{q}$. Hence
$\TV(u_L,v_L)\leq Z_L$ and $\TV(u_L,v_L)\leq1$, proving
\cref{eq:belief-contraction}.

For the second claim, set $\cD'=\belief_h(\tau)$.
By \cref{lemma:belief-is-cond-prob}, this is the conditional law of $s_h$
given $\tau$, so the initialized continuation has the same law as the
original process conditioned on $\tau$. The recursive belief definitions
identify $u_L$ with the true belief at step $h+L$. Substituting into
\cref{eq:belief-contraction} proves the claim.
\end{proof}

The initial factor satisfies
$\Dpr{\cD'}{\cD}=\TV(\cD',\cD)\norm{\cD'/\cD}_\infty
\leq\norm{\cD'/\cD}_\infty$.
In particular, for $\cD=\unif(\MS_h)$ it is at most $S$.
Using $\norm{p-q}_1=2\TV(p,q)$ and averaging the conditional bound over
the initial history therefore gives
\[
 \epsilon^{\contraction}_{h+L}(\pi;L)
 \leq\min\left\{2,\;2\sqrt S(8\delta)^{L/2}\right\}.
\]

\begin{remark}\label{remark:perturbed-block-golowich}
For $0\leq\delta<1/2$, any $\delta$-perturbed Block MDP is $\gamma$-observable
with $\gamma = 1-2\delta$ (\cref{def:observable}): for any $h\in[H]$, we have $\BO_h = (1-\delta)\BOT_h + \delta E_h$, and thus given any $b,b'$, we have
\begin{align*}
    \norm{\BO^\top_h b - \BO^\top_h b'}_1 &=
    \norm{(1-\delta)\BOT^\top_h(b-b') + \delta E_h^\top(b-b')}_1 \\
    &\geq (1-\delta)\norm{\BOT^\top_h(b-b')}_1 - \delta\norm{E_h^\top(b-b')}_1 \\
    &\geq (1-\delta)\norm{b-b'}_1 - \delta\norm{b-b'}_1 \\
    &= (1-2\delta)\norm{b-b'}_1
\end{align*}
because $\BOT_h$ satisfies the block property and the stochastic kernel $E_h$ contracts $\ell_1$ distance. It was shown in \cite[Theorem 4.7]{golowich2023planning} that, for any $\gamma$-observable POMDP $\cP$, the belief contraction error can be bounded as $\epsilon_h^{\contraction}(\pi;L) \leq (1 - \gamma^4/2^{40})^L \cdot \bigoh(S)$. However, substituting in $\gamma := 1-2\delta$, we see that due to the constant factor of $2^{40}$, this bound does not asymptotically improve as $\delta$ decreases --- indeed, it is never better than $(1 - 1/2^{40})^L \cdot \bigoh(S)$ --- and moreover is vacuous for $L = \littleoh(\log S)$.
\end{remark}

\subsection{Approximate Decodability}\label{subsec:decodability}

In this section we prove \cref{prop:det-vinf-decay-main}, restated below as \cref{prop:vinf-decay}, which states that for $\delta$-perturbed Block MDPs with \emph{deterministic} latent transitions, the decodability error decays exponentially. This argument uses \cref{lemma:martingale-bound-improved} together with the fact that deterministic transitions cannot increase decodability error. The restriction to deterministic transitions is essential here; the data processing argument for belief contraction in \cref{thm:belief-contraction} also allows stochastic transitions.

For notational convenience, we make the following definition of the ``$\ell_\infty$ variance'' $\Vinf(b)$ for a given distribution $b$.

\begin{definition}
For any set $\MS$ and distribution $b \in \Delta(\MS)$, define $\Vinf(b) := 1 - \norm{b}_\infty$.
\end{definition}

The following lemma shows that the $\ell_\infty$ variance contracts by $\poly(\delta)$ with high probability under the Bayes operator.

\begin{lemma}\label{lemma:vinf-bayes-update}
Let $\delta\in[0,1]$, and suppose that $\cP$ is a $\delta$-perturbed Block MDP with deterministic latent transitions (but potentially stochastic initial state). Fix $h \in [H]$ and $b \in \Delta(\MS_h)$. Draw $x \sim \BO_h(\cdot\mid{}b)$. Then $\EE[\Vinf(\BB_h(b;x))] \leq \min(\Vinf(b),\delta)$ and
\[\Pr[\Vinf(\BB_h(b;x)) > \delta^{1/3} \Vinf(b)] \leq 3\delta^{1/3}.\]
\end{lemma}

\begin{proof}
Pick any $s^\star \in \MS_h$ such that $b_{s^\star} = \norm{b}_\infty$, and hence $\sum_{s \in \MS_h \setminus \{s^\star\}} b(s) = \Vinf(b)$. Then
\begin{align}
\EE[\Vinf(\BB_h(b;x))]
&= \sum_{x \in \cX_h} \BO_h(x \mid{} b) \left(1 - \max_{s \in \MS_h} \BB_h(b;x)(s)\right) \nonumber\\
&= \sum_{x \in \cX_h} \min_{s \in \MS_h} \left(\BO_h(x \mid{} b) - b(s) \BO_h(x \mid{} s)\right) \nonumber\\
&= \sum_{x \in \cX_h} \min_{s \in \MS_h} \sum_{s' \in \MS_h \setminus \{s\}} b(s') \BO_h(x \mid{} s') \label{eq:vinf-expr-init}
\end{align}
where the second equality is by the definition $\BB_h(b;x)(s) := \frac{b(s) \BO_h(x\mid{} s)}{\BO_h(x \mid{} b)}$ (\cref{def:bel-update}). First, \cref{eq:vinf-expr-init} implies that
\[\EE[\Vinf(\BB_h(b;x))] \leq \sum_{x \in \cX_h} \sum_{s' \in \MS_h \setminus \{s^\star\}} b(s') \BO_h(x\mid{}s') = \sum_{s' \in \MS_h \setminus \{s^\star\}} b(s') = \Vinf(b),\]
where the first equality uses the fact that $\BO_h(\cdot\mid{}s')$ is a distribution for any fixed $s'$. Next, \cref{eq:vinf-expr-init} implies that
\begin{align}
\EE[\Vinf(\BB_h(b;x))]
&\leq \sum_{x \in \cX_h} \sum_{s' \in \MS_h \setminus \{\phi(x)\}} b(s') \BO_h(x \mid{} s') \nonumber \\
&= \delta \sum_{x \in \cX_h} \sum_{s' \in \MS_h \setminus \{\phi(x)\}} b(s') E_h(x \mid{} s') \nonumber \\
&= \delta \sum_{s' \in \MS_h} b(s') \sum_{x \in \cX_h:\phi(x)\neq s'} E_h(x \mid{} s') \nonumber \\
&\leq \delta \sum_{s' \in \MS_h} b(s') \nonumber \\
&\leq \delta. \label{eq:vinf-exp-delta}
\end{align}
This, together with the preceding bound, proves the first claim of the lemma. If $\delta=0$ or $\Vinf(b)=0$, this bound gives $\Vinf(\BB_h(b;x))=0$ almost surely, proving the second claim as well. Hence assume that $\delta>0$ and $\Vinf(b)>0$. Now consider the event that $\phi(x) = s^\star$. Since $\BO_h(x\mid{}b) \geq (1-\delta)b(\phi(x))\BOT_h(x\mid{}\phi(x))$, we have
\begin{equation} \Pr[\phi(x) \neq s^\star] = 1 - \sum_{x\in\cX_h:\phi(x)=s^\star}\BO_h(x \mid{} b) \leq 1 - (1-\delta) b(s^\star) = 1 - (1-\delta)(1-\Vinf(b)) \leq \delta + \Vinf(b).\label{eq:vinf-bad-event}\end{equation}
Moreover, by an analogous calculation as \cref{eq:vinf-expr-init},
\begin{align*}
\EE[\Vinf(\BB_h(b;x)) \mathbbm{1}[\phi(x) = s^\star]]
&= \sum_{x\in\cX_h:\phi(x)=s^\star} \BO_h(x \mid{} b) \left(1 - \max_{s \in \MS_h} \BB_h(b;x)(s)\right) \\
&= \sum_{x\in\cX_h:\phi(x)=s^\star} \min_{s \in \MS_h}\left(\BO_h(x\mid{}b) - b(s)\BO_h(x\mid{}s)\right) \\
&\leq \sum_{x\in\cX_h:\phi(x)=s^\star} \BO_h(x\mid{}b) - b(s^\star) \BO_h(x\mid{}s^\star) \\
&= \sum_{x\in\cX_h:\phi(x)=s^\star} \sum_{s \in \MS_h \setminus \{s^\star\}} b(s) \BO_h(x\mid{}s) \\
&= \delta \sum_{x\in\cX_h:\phi(x)=s^\star}\sum_{s \in \MS_h \setminus \{s^\star\}} b(s) E_h(x\mid{}s) \\
&\leq \delta \Vinf(b).
\end{align*}
It follows that
\begin{align*}
\Pr[\Vinf(\BB_h(b;x)) > \delta^{1/3} \Vinf(b)]
&\leq \Pr[\Vinf(\BB_h(b;x)) \mathbbm{1}[\phi(x)=s^\star] > \delta^{1/3} \Vinf(b)] + \Pr[\phi(x) \neq s^\star] \\
&\leq \delta^{2/3} + \delta + \Vinf(b)
\end{align*}
by Markov's inequality and \cref{eq:vinf-bad-event}. To conclude, we distinguish two cases. If $\Vinf(b) \leq \delta^{1/3}$, then we get
\[\Pr[\Vinf(\BB_h(b;x)) > \delta^{1/3} \Vinf(b)] \leq \delta^{2/3} + \delta + \delta^{1/3} \leq 3\delta^{1/3}\]
as needed. Otherwise, $\Vinf(b) > \delta^{1/3}$, so
\[\Pr[\Vinf(\BB_h(b;x)) > \delta^{1/3} \Vinf(b)] \leq \Pr[\Vinf(\BB_h(b;x)) > \delta^{2/3}] \leq \delta^{1/3}\]
by Markov's inequality and \cref{eq:vinf-exp-delta}. This completes the proof.
\end{proof}

Using \cref{lemma:vinf-bayes-update} and the assumption of deterministic latent dynamics, it is straightforward to show that the $\ell_\infty$ variance contracts by $\poly(\delta)$ with high probability under the \emph{belief update} operator:

\begin{corollary}\label{cor:vinf-belief-update}
Let $\delta\in[0,1]$, and suppose that $\cP$ is a $\delta$-perturbed Block MDP with deterministic latent transitions (but potentially stochastic initial state). Fix $h \in \{2,\ldots,H\}$ and $b \in \Delta(\MS_{h-1})$. For any action $a_{h-1} \in \MA_{h-1}$, with expectation over $x_h \sim (\BO_h)^\t \bbP_h(a_{h-1}) \cdot b$, it holds that $\EE[\Vinf(\BU_h(b;a_{h-1},x_h))] \leq \Vinf(b)$ and
\[\Pr[\Vinf(\BU_h(b;a_{h-1},x_h)) > \delta^{1/3} \Vinf(b)] \leq 3\delta^{1/3}.\]
\end{corollary}

\begin{proof}
From \cref{def:bel-update}, we have for any $x_h$ that $\BU_h(b;a_{h-1},x_h) = \BB_h(\bbP_h(a_{h-1}) \cdot b; x_h)$. Applying \cref{lemma:vinf-bayes-update} to the distribution $\bbP_h(a_{h-1}) \cdot b$ (observe that $x_h$ is indeed distributed according to $\BO_h(\cdot \mid{} \bbP_h(a_{h-1}) \cdot b)$), we get $\EE[\Vinf(\BU_h(b;a_{h-1}, x_h))] \leq \Vinf(\bbP_h(a_{h-1}) \cdot b)$ and
\[\Pr[\Vinf(\BU_h(b;a_{h-1},x_h)) > \delta^{1/3} \Vinf(\bbP_h(a_{h-1})\cdot b)] \leq 3\delta^{1/3}.\]
To complete the proof, it suffices to show that $\Vinf(\bbP_h(a_{h-1}) \cdot b) \leq \Vinf(b)$. Indeed, since the transitions are deterministic, the matrix $\bbP_h(a_{h-1}) \in \RR^{|\MS_h| \times |\MS_{h-1}|}$ satisfies that every column is a standard basis vector. Identify any $s^\star \in \MS_{h-1}$ with $b_{s^\star} = \norm{b}_\infty$. Then there is some $s_h \in \MS_h$ with $\bbP_h(a_{h-1})_{s_h, s^\star} = \bbP_h(s_h \mid{} s^\star, a_{h-1}) = 1$. But then the entry of $\bbP_h(a_{h-1}) \cdot b$ indexed by $s_h$ is at least $b_{s^\star}$. So indeed $\Vinf(\bbP_h(a_{h-1}) \cdot b) \leq \Vinf(b)$.
\end{proof}

We can now prove the following restatement of \cref{prop:det-vinf-decay-main}.

\begin{proposition}\label{prop:vinf-decay}
There is a universal constant $C_{\ref{prop:vinf-decay}}>1$ such that the following holds.
Let $\delta\in[0,1]$, and suppose that $\cP$ is a $\delta$-perturbed Block MDP
with deterministic latent transitions (but potentially stochastic initial
state). Fix any executable policy $\pi$ and index $h\in[H]$. Then
\[
 \EE^\pi\!\left[\Vinf\bigl(\bel_h(x_{1:h},a_{1:h-1})\bigr)\right]
 \leq\min\left\{\delta,(C_{\ref{prop:vinf-decay}}\delta)^{(h-1)/6}\right\}.
\]
\end{proposition}
\begin{proof}
Set
$X_t:=\Vinf(\bel_t(x_{1:t},a_{1:t-1}))$ and let
$\mathcal F_t:=\sigma(x_{1:t},a_{1:t-1})$ for $1\leq t\leq h$.
Then $0\leq X_t\leq1$. For $2\leq t\leq h$, condition on
$\mathcal F_{t-1}$ and the next action $a_{t-1}$.
Since $\pi$ is executable, \cref{cor:vinf-belief-update} applies to the
conditional observation law. Averaging over $a_{t-1}$ gives
\[
 \EE[X_t\mid\mathcal F_{t-1}]\leq X_{t-1},\qquad
 \Pr[X_t>\delta^{1/3}X_{t-1}\mid\mathcal F_{t-1}]
 \leq3\delta^{1/3}.
\]
Thus $(X_t,\mathcal F_t)_{t=1}^h$ is a nonnegative supermartingale,
and the same probability bound holds conditionally on $X_{t-1}$.
Moreover, \cref{lemma:vinf-bayes-update} gives $\EE[X_1]\leq\delta$,
so $\EE[X_h]\leq\delta$.

Set $C_{\ref{prop:vinf-decay}}=C_{\ref{prop:det-vinf-decay-main}}:=1728$. The case $h=1$ follows from
$\EE[X_1]\leq\min\{\delta,1\}$.
If $\delta=0$, then $\EE[X_h]\leq\delta=0$ proves the claim.
If $\delta\geq1/27$, then $(C_{\ref{prop:vinf-decay}}\delta)^{(h-1)/6}\geq1$, and
$\EE[X_h]\leq\min\{\delta,1\}$ again proves the claim.
For $0<\delta<1/27$, take $\epsilon:=3\delta^{1/3}\in(0,1)$.
Since $\epsilon\geq\delta^{1/3}$, the preceding tail bound implies
\[
 \Pr[X_t>\epsilon X_{t-1}\mid X_{t-1}]\leq\epsilon.
\]
Apply \cref{lemma:martingale-bound-improved} to
$(X_1,\ldots,X_h)$ with $S=1$ and $h-1$ steps. Because $X_h\leq1$,
\[
 \EE[X_h]=\EE[\min(X_h,1)]
 \leq(2\sqrt\epsilon)^{h-1}
 =(2\sqrt3\,\delta^{1/6})^{h-1}
 =(C_{\ref{prop:vinf-decay}}\delta)^{(h-1)/6}.
\]
Combining this with $\EE[X_h]\leq\delta$ proves the result.
\end{proof}

\subsection{Misspecification Lower Bound for Stochastic Dynamics}\label{subsec:misspec-lb}

In this section we prove the following restatement of \cref{prop:stoch-decodability-lb-main}, which shows that in a $\delta$-perturbed Block MDP with general (stochastic) latent transitions, the misspecification of the optimal latent policy with respect to the class of executable policies can be as large as $\Omega(\delta H)$ (for $\delta \leq 1/H$). This implies an analogous lower bound on decodability error, i.e. it cannot improve exponentially as $h$ increases, unlike the case of deterministic latent transitions. Moreover, it shows a fundamental source of (horizon-dependent) error that is not mitigated by increasing the frame-stack $L$: since the following bound applies to all executable policies, it also applies to the class of $L$-step executable policies for any $L$.

\begin{proposition}\label{prop:stoch-decodability-lb}
Let $\delta\in(0,1]$ and $H\in \NN$. There is a $\delta$-perturbed Block MDP $\cP$ with horizon $H$ such that the optimal latent policy $\pilat$ satisfies
\[\min_{\pi\in\Pi} \TV(\bbP^{\pilat}, \bbP^\pi) \geq \Omega(\min(1, \delta H))\]
where $\Pi$ is the class of executable policies.
\end{proposition}

\begin{proof}
Set $q:=\min\{\delta,1/2\}$. We define $\cP$ as follows. For all $h \in [H]$, define the latent state space and observation space to be $\cS_h := \cX_h := \{0,1\}$; also define $\cA_h := \{0,1\}$. Let the initial distribution and latent transition dynamics at each step be uniformly random (independent of the previous state and action). For each $h \in [H]$, define the reward function $R_h: \cS_h \times \cA_h \to [0,1]$ by $R_h(s,a) = \frac{1}{H}\mathbbm{1}[s=a]$. Define the observation distribution $\BO_h: \cS_h \to \Delta(\cX_h)$ by $\BO_h(s\mid{}s) = 1-q$ and $\BO_h(1-s\mid{}s) = q$.

This is a $\delta$-perturbed Block MDP: take $\BOT_h(s\mid s)=1$ and
\[
 E_h(s\mid s)=1-q/\delta,\qquad E_h(1-s\mid s)=q/\delta.
\]
Since $0<q\leq\delta$, these are valid distributions and satisfy
$\BO_h=(1-\delta)\BOT_h+\delta E_h$.
Under the trajectory distribution $\bbP^{\pilat}$ induced by the optimal latent policy $\pilat$, it holds that $a_h = s_h$ for all $h \in [H]$ with probability $1$. For any executable policy $\pi$, step $h$, and positive-probability history $\tau_{1:h-1} = (s_{1:h-1}, x_{1:h-1}, a_{1:h-1})$, the state $s_h$ is uniform and independent of this history. Given $x_h$, its largest posterior mass is $1-q$, since $q\leq1/2$. Consequently,
\[
 \Pr^\pi[a_h=s_h\mid\tau_{1:h-1}]\leq1-q,
 \qquad
 \Pr^\pi[\forall h\in[H]:a_h=s_h]\leq(1-q)^H.
\]
Comparing the probability of the event $\{a_h=s_h\text{ for all }h\}$ gives
\[
 \TV(\bbP^\pi,\bbP^{\pilat})\geq1-(1-q)^H.
\]
If $\delta\leq1/2$, then $q=\delta$ and
$1-(1-q)^H\geq1-e^{-\delta H}\geq\frac12\min\{1,\delta H\}$.
If $\delta>1/2$, then $q=1/2$ and
$1-(1-q)^H\geq1/2\geq\frac12\min\{1,\delta H\}$.
This proves the claimed bound for every executable policy.
\end{proof}

\section{Omitted Proofs for Expert Distillation}
\subsection{Misspecification Bounds for Composed Policies}

In this section we prove upper bounds on the misspecification of a latent policy (with respect to certain executable policies obtained by \emph{composing} the latent with some belief state) in terms of instance-dependent error metrics. The first main result is \cref{lemma:tv-beltil-to-latent} (a restatement of \cref{lemma:tv-beltil-to-latent-main}), where the upper bound is in terms of decodability error (\cref{def:decodability-error}) and error in approximating the true belief state. The second main result is \cref{lemma:smoothness-to-latent}, where the decodability error term is replaced by \emph{action-prediction error} (\cref{def:action-prediction}).

To prove \cref{lemma:tv-beltil-to-latent}, it is convenient to first analyze the policy that samples a state from the true belief state induced by the current history, and then samples an action from $\pilat$ accordingly. The key technical observation, below, encapsulates the intuition that resampling a near-deterministic random variable is likely to yield the same realization.

\begin{lemma}\label{lemma:resampling-tv}
Let $\bbP_{Y,\Ynuis}$ be a joint distribution over random variables $(Y,\Ynuis)$. Let $\bbP_{Z|Y}$ be a conditional distribution. Define $\bbQ_{Z}$ as follows. Resample $Y' \sim \bbP_{Y}$ and then sample $Z \sim \bbP_{Z|Y'}$. Then
\[\TV(\bbP_{Y,\Ynuis}\bbP_{Z|Y}, \bbP_{Y,\Ynuis}\bbQ_{Z}) \leq 2\Vinf(\bbP_Y).\]
Additionally,
\[\TV(\bbP_{Y,\Ynuis}\bbP_{Z|Y}, \bbP_{Y,\Ynuis}\bbQ_{Z}) \leq 2\Vinf(\bbP_Z).\]
\end{lemma}

\begin{proof}
Consider the process where we draw $(Y,\Ynuis) \sim \bbP_{Y,\Ynuis}$, $Z \sim \bbP_{Z|Y}$, and $Y' \sim \bbP_Y$. If $Y' = Y$ then we set $Z' = Z$; otherwise we sample $Z' \sim \bbP_{Z|Y'}$. Then $(Y,\Ynuis,Z)$ is distributed according to $\bbP_{Y,\Ynuis}\bbP_{Z|Y}$, and $(Y,\Ynuis, Z')$ is distributed according to $\bbP_{Y,\Ynuis}\bbQ_{Z}$. Thus, we have defined a coupling. Moreover,
\begin{align*}
\Pr[(Y,\Ynuis,Z)\neq (Y,\Ynuis,Z')]
&\leq \Pr[Y \neq Y'] \\
&\leq 2\Pr[Y \neq \argmax_y \bbP_Y(y)] \\
&= 2 \Vinf(\bbP_Y)
\end{align*}
as needed for the first claim. For the second claim, fix a mode $z^\star\in\argmax_z\bbP_Z(z)$. Both $Z$ and $Z'$ have marginal law $\bbP_Z$, so
\begin{align*}
\Pr[(Y,\Ynuis,Z)\neq (Y,\Ynuis,Z')]
&= \Pr[Z \neq Z'] \\
&\leq \Pr[Z\neq z^\star]+\Pr[Z'\neq z^\star] \\
&= 2 \Vinf(\bbP_Z)
\end{align*}
as needed.
\end{proof}

\begin{lemma}\label{lemma:tv-bel-to-latent}
Let $\pilat \in \Pilat$ be a latent (Markovian) policy and define the executable policy $\pi$ by
\[\pi(x_{1:h},a_{1:h-1}) := \pilat \circ \bel_h(x_{1:h},a_{1:h-1}).\]
Then
\[\TV(\bbP^{\pilat}, \bbP^{\pi}) \leq 2\sum_{h=1}^H \EE^{\pi}[\Vinf(\bel_h(x_{1:h}, a_{1:h-1}))].\]
\end{lemma}

\begin{proof}
For each $1 \leq h \leq H+1$ let $\pi \circ_h \pilat$ denote the policy that follows $\pi$ for the first $h-1$ actions and subsequently follows $\pilat$. Since $\pi \circ_1 \pilat = \pilat$ and $\pi \circ_{H+1} \pilat = \pi$, it suffices to show that, for each $h \in [H]$,
\[\TV(\bbP^{\pi \circ_h \pilat}, \bbP^{\pi \circ_{h+1} \pilat}) \leq 2\EE^\pi[\Vinf(\bel_h(x_{1:h},a_{1:h-1}))].\]
Observe that both distributions have the same conditional distributions over $(s_{h+1:H},x_{h+1:H},a_{h+1:H})$ given $(s_{1:h},x_{1:h},a_{1:h})$. By this fact and the data processing inequality,
\[\TV(\bbP^{\pi \circ_h \pilat}, \bbP^{\pi \circ_{h+1} \pilat}) = \TV(\bbP^{\pi \circ_h \pilat}_{X,Y,\Ynuis,Z}, \bbP^{\pi \circ_{h+1} \pilat}_{X,Y,\Ynuis,Z})\]
where $X = (x_{1:h},a_{1:h-1})$, $Y = s_h$, $\Ynuis = s_{1:h-1}$, and $Z = a_h$. Both distributions have the same marginal over $(X,Y,\Ynuis)$. The distribution of $Y|X$ is precisely $\bel_h(x_{1:h},a_{1:h-1})$ by \cref{lemma:belief-is-cond-prob} and the fact that $\pi$ is executable. In the former distribution, $a_h$ is generated by sampling from $\pilat(s_h)$. In the latter distribution, $a_h$ is generated by sampling $s'_h \sim \bel_h(x_{1:h},a_{1:h-1})$ and then sampling from $\pilat(s'_h)$. Thus, the conditions of \cref{lemma:resampling-tv} are met (after conditioning out $X$), and we get
\[\TV(\bbP^{\pi \circ_h \pilat}, \bbP^{\pi \circ_{h+1} \pilat}) \leq 2\EE^\pi[\Vinf(\bel_h(x_{1:h},a_{1:h-1}))]\]
as needed.
\end{proof}

We now prove the following restatement of \cref{lemma:tv-beltil-to-latent-main}.

\begin{lemma}\label{lemma:tv-beltil-to-latent}
Let $\pilat \in \Pilat$ be a latent (Markovian) policy and let $\beltil_{1:H}$ be a collection of functions $\beltil_h: \cX^h \times \cA^{h-1} \to \Delta(\cS_h)$. Define the executable policy $\pitil$ by
\[\pitil(x_{1:h},a_{1:h-1}) := \pilat \circ \beltil_h(x_{1:h},a_{1:h-1}).\]
Then
\[\TV(\bbP^{\pilat}, \bbP^{\pitil}) \leq \sum_{h=1}^H \EE^{\pi}\left[2\Vinf(\bel_h(x_{1:h}, a_{1:h-1})) + \norm{\bel_h(x_{1:h}, a_{1:h-1}) - \beltil_h(x_{1:h}, a_{1:h-1})}_1\right]\]
and
\begin{align*}
\TV(\bbP^{\pilat}, \bbP^{\pitil})
&\leq \sum_{h=1}^H \EE^{\pi}\left[2\Vinf(\bel_h(x_{1:h}, a_{1:h-1}))\right] \\
&\quad+\sum_{h=1}^H\EE^{\pitil}\left[\norm{\bel_h(x_{1:h}, a_{1:h-1}) - \beltil_h(x_{1:h}, a_{1:h-1})}_1\right]
\end{align*}
where $\pi(x_{1:h},a_{1:h-1}) := \pilat \circ \bel_h(x_{1:h},a_{1:h-1})$.
\end{lemma}

\begin{proof}
We can couple $\bbP^\pi$ and $\bbP^{\pitil}$ so that at any step $h$, if the trajectories have thus far been the same sequence $(x_{1:h},a_{1:h-1})$, then the probability that they choose different actions $a_h$ is at most \[\norm{\bel_h(x_{1:h}, a_{1:h-1}) - \beltil_h(x_{1:h}, a_{1:h-1})}_1.\]
By a standard hybrid argument, it follows that
\[\TV(\bbP^\pi,\bbP^{\pitil}) \leq \sum_{h=1}^H \EE^\pi\left[\norm{\bel_h(x_{1:h}, a_{1:h-1}) - \beltil_h(x_{1:h}, a_{1:h-1})}_1\right]\]
and also
\[\TV(\bbP^\pi,\bbP^{\pitil}) \leq \sum_{h=1}^H \EE^{\pitil}\left[\norm{\bel_h(x_{1:h}, a_{1:h-1}) - \beltil_h(x_{1:h}, a_{1:h-1})}_1\right].\]
Combining with \cref{lemma:tv-bel-to-latent} completes the proof.
\end{proof}

The preceding lemma used the intuition that if the latent state is near-deterministic when conditioned on the observation/action history, then resampling it is unlikely to change it. The following lemma uses the intuition that if the \emph{action} is near-deterministic (which is likely when the action-prediction is small) when conditioned on the action/observation history, then resampling the latent state --- and subsequently resampling the action conditioned on this resampled latent state --- is unlikely to change the action, though it may have changed the state.

\begin{lemma}\label{lemma:smoothness-to-latent}
    Let $\pilat \in \Pilat$ be a latent (Markovian) policy and let $\beltil_{1:H}$ be a collection of functions $\beltil_h: \cX^h \times \cA^{h-1} \to \Delta(\cS_h)$. Define the executable policy $\pitil$ by
    \[\pitil(x_{1:h},a_{1:h-1}) := \pilat \circ \beltil_h(x_{1:h},a_{1:h-1}).\]
    Then
    \[\TV(\bbP^{\pilat}, \bbP^{\pitil}) \leq \sum_{h=1}^H \EE^{\pi}\left[2 \epsilon^{\asmooth;\pi^{\latent}}_h(\pi) + \norm{\pilat \circ \bel_h(x_{1:h}, a_{1:h-1}) - \pilat \circ \beltil_h(x_{1:h}, a_{1:h-1})}_1\right]\]
    where $\pi(x_{1:h},a_{1:h-1}) := \pilat \circ \bel_h(x_{1:h},a_{1:h-1})$.
    \end{lemma}

\begin{proof}
As with the proof of \cref{lemma:tv-beltil-to-latent}, it suffices to show that
\[\TV(\bbP^{\pilat}, \bbP^{\pi}) \leq \sum_{h=1}^H \EE^{\pi}\left[2 \epsilon^{\asmooth;\pi^{\latent}}_h(\pi)\right].\]
The proof is identical to that of \cref{lemma:tv-beltil-to-latent} except for using the second claim of \cref{lemma:resampling-tv} instead of the first.
\end{proof}

\subsection{Analysis of $\forward$ for $L$-step Executable Policies}
\label{sec:forward-population}

In this section we describe the (slightly modified) $\forward$ imitation learning algorithm \citep{ross2010efficient}, applied to the problem of distilling an expert latent policy $\pilat$ to an $L$-step executable policy $\pihat \in \Pi^L$. We first formally derive the expression for the policy learned in the infinite-sample limit (\cref{lemma:forward-bayes}), and then prove \cref{thm:perturbed-mdp-il-guarantee} (the regret bound for this policy, restated as \cref{thm:perturbed-mdp-il-guarantee-main}). Finally, we prove a finite-sample guarantee for the modified $\forward$ algorithm (\cref{thm:perturbed-mdp-il-finite-sample}), using the same ideas together with a finite-sample analysis of empirical conditional action distributions.

\paragraph{$\forward$ with $L$-step random actions} For a latent Markovian policy $\pilat \in \Pilat$ and an integer $L\geq1$, the (modified) $\forward$ algorithm computes an $L$-step executable policy $\pihat = \pihat_{1:H}$ as follows. For $h = 1,\dots, H$ (with an empty learned prefix and uniform actions from the start when $h\leq L$):
\begin{enumerate}
\item Draw $n$ fresh trajectories $\tau^i = (s_{1:H}^i, x_{1:H}^i, a_{1:H}^i)$ from the policy $\pihat_{1:h-L-1} \circ_{h-L} \Unif(\cA)$ (which follows $\pihat$ until step $h-L-1$ and subsequently plays uniformly random actions).
\item For each trajectory, independently sample an expert action $y_h^i\sim\pilat_h(\cdot\mid s_h^i)$, and compute
\[\pihat_h := \argmax_{\pi_h \in \Pi^L_h} \frac{1}{n} \sum_{i=1}^n \log \pi_h(y_h^i \mid{} x^i_{\max(1,h-L+1):h}, a^i_{\max(1,h-L):h-1}).\]
\end{enumerate}

For a deterministic expert, $y_h^i$ is its prescribed action. In population objectives below, $\log\pi_h(\pilat_h(s_h)\mid z_h)$ denotes the conditional expectation of $\log\pi_h(y_h\mid z_h)$ over $y_h\sim\pilat_h(\cdot\mid s_h)$. Above, $\Pi^L_h$ is the set of $L$-step conditional distributions $\pi_h: \cX^{h-L+1:h} \times \cA^{h-L:h-1} \to \Delta(\cA)$. Note that the standard $\forward$ algorithm is identical except that it draws data from $\pihat_{1:h-1}$ at step $h$. The modified algorithm is simpler to analyze since the random actions do not bias the conditional distribution of the latent state at step $h$ given the partial history $(x_{h-L+1:h},a_{h-L:h-1})$, so this distribution can be directly related to an approximate belief state with appropriate prior (\cref{lemma:belapx-cond-prob}).

For notational convenience, let $\pitil$ denote the policy obtained from the above algorithm in the infinite-sample limit, i.e. at step $h$ we define
\[\pitil_h := \argmax_{\pi_h \in \Pi^L_h} \EE^{\pitil_{1:h-L-1}\circ_{h-L}\Unif(\cA)}\left[ \log \pi_h(\pilat_h(s_h) \mid{} x_{\max(1,h-L+1):h}, a_{\max(1,h-L):h-1})\right].\]

The following lemma gives a closed-form expression for this policy.

\begin{lemma}\label{lemma:forward-bayes}
For each $h\in[H]$, the following identity holds almost surely under the sampling policy $\pitil_{1:h-L-1}\circ_{h-L}\Unif(\cA)$:
\[\pitil_h(\cdot \mid{} x_{h-L+1:h}, a_{h-L:h-1}) =\begin{cases}
\pilat_h \circ \belapx_h(x_{h-L+1:h},a_{h-L:h-1}; d^{\pitil}_{h-L}) & \text{ if } h > L \\
\pilat_h \circ \bel_h(x_{1:h},a_{1:h-1}) &
\text{ otherwise}
\end{cases}
\]
Values at windows of zero sampling probability may be chosen arbitrarily.
\end{lemma}

\begin{proof}
We have for each $h>L$ that
\begin{align*}
\pitil_h
&=\argmax_{\pi_h \in \Pi_h^L} \EE^{\pitil_{1:h-L-1} \circ_{h-L} \Unif(\cA)}\left[ \log \pi_h(\pilat_h(s_h) \mid{} x_{h-L+1:h}, a_{h-L:h-1})\right].
\end{align*}
Since $\Pi_h^L$ contains every conditional action distribution and population-level maximum likelihood minimizes KL divergence, we get, at every window of positive sampling probability,
\begin{align*}
\pitil_h(a_h \mid{} x_{h-L+1:h}, a_{h-L:h-1})
&= \sum_{s_h \in \cS_h} \pilat_h(a_h \mid{} s_h) \cdot \bbP^{\pitil_{1:h-L-1} \circ_{h-L} \Unif(\cA)}[s_h \mid{} x_{h-L+1:h}, a_{h-L:h-1}]  \\
&= \sum_{s_h \in \cS_h} \pilat_h(a_h \mid{} s_h) \cdot \belapx_h(x_{h-L+1:h},a_{h-L:h-1}; d^{\pitil_{1:h-L-1} \circ_{h-L} \Unif(\cA)}_{h-L})(s_h)\\
&= \sum_{s_h \in \cS_h} \pilat_h(a_h \mid{} s_h) \cdot \belapx_h(x_{h-L+1:h},a_{h-L:h-1}; d^{\pitil}_{h-L})(s_h)
\end{align*}
where the second equality uses \cref{lemma:belapx-cond-prob}. The application of this lemma uses the fact that $\pitil_{1:h-L-1} \circ_{h-L} \Unif(\cA)$ plays actions at steps $h-L,\dots,h-1$ that are uniformly random. This proves the lemma in the case $h>L$. The case $h \leq L$ is analogous but uses \cref{lemma:belief-is-cond-prob} instead of \cref{lemma:belapx-cond-prob}. Histories under $\pitil$ have density ratio at most $A^L$ relative to the corresponding sampling policy before action $a_h$, so the identity also holds almost surely under $\pitil$.
\end{proof}

\begin{theorem}
\label{thm:perturbed-mdp-il-guarantee}
Suppose that the POMDP $\cP$ is a $\delta$-perturbed Block MDP with deterministic transitions, and fix an integer $L\geq1$. Let $\pilat\in\Pilat$ be a latent (Markovian) policy, and let $\pitil$ be the policy computed by $\forward$ with $L$-step random actions, in the infinite-sample limit. Then
\[
 J(\pilat)-J(\pitil)
 \leq\TV(\bbP^{\pilat},\bbP^{\pitil})
 \leq O(\delta)+2H\sqrt S\,(8\delta)^{L/2}.
\]
\end{theorem}

\begin{proof}
Write $C=C_{\ref{prop:vinf-decay}}$. If $\delta\geq1/(64C)$, the claim follows from $\TV\leq1\leq64C\delta$. Hence assume $\delta<1/(64C)$.

First, for any executable policy $\mu$, \cref{prop:vinf-decay} gives
\begin{equation}\label{eq:distillation-decodability-sum}
\begin{aligned}
 2\sum_{h=1}^H\EE^\mu[\Vinf(\bel_h(x_{1:h},a_{1:h-1}))]
 &\leq2\sum_{h=1}^H\min\{\delta,(C\delta)^{(h-1)/6}\}\\
 &\leq12\delta+\frac{2C\delta}{1-(C\delta)^{1/6}}
 \leq O(\delta).
\end{aligned}
\end{equation}
The last step uses $(C\delta)^{1/6}<1/2$.

Next, fix any executable policy $\mu$ and $h>L$, and set $t=h-L$.
Let $\tau_t=(x_{1:t},a_{1:t-1})$, $b_t=\bel_t(\tau_t)$, and
$d_t=d_t^\mu$ be the marginal law of $s_t$ under $\mu$.
By \cref{lemma:belief-is-cond-prob}, $\EE^\mu[b_t(s)]=d_t(s)$.
Consequently $b_t\ll d_t$ almost surely. On this event, set
$\kappa_t=\norm{b_t/d_t}_\infty$.
The positive-support ratios have maximum at least $1$, so the convention $0/0=1$ does not change this maximum. Thus
\[
 \EE^\mu[\kappa_t]
 \leq\sum_{s:\,d_t(s)>0}\frac{\EE^\mu[b_t(s)]}{d_t(s)}
 =\lvert\operatorname{supp}(d_t)\rvert\leq S.
\]
Apply the conditional assertion of \cref{thm:belief-contraction} at time $t$, with initial beliefs $b_t$ and $d_t$. Since
$\Dpr{b_t}{d_t}\leq\kappa_t$ and $\norm{p-q}_1=2\TV(p,q)$, averaging over $\tau_t$ and using Jensen's inequality yields
\begin{equation}\label{eq:distillation-belief-error}
\begin{aligned}
 &\EE^\mu\left[
 \norm{\bel_h(x_{1:h},a_{1:h-1})
 -\belapx_h(x_{h-L+1:h},a_{h-L:h-1};d_{h-L}^\mu)}_1\right]\\
 &\qquad\leq2(8\delta)^{L/2}\EE^\mu[\sqrt{\kappa_t}]
 \leq2\sqrt S\,(8\delta)^{L/2}.
\end{aligned}
\end{equation}

Define $\pi_h=\pilat_h\circ\bel_h$. By \cref{lemma:tv-beltil-to-latent} and the expression for $\pitil$ in \cref{lemma:forward-bayes},
\begin{align*}
 \TV(\bbP^{\pilat},\bbP^{\pitil})
 &\leq2\sum_{h=1}^H\EE^\pi[\Vinf(\bel_h(x_{1:h},a_{1:h-1}))]\\
 &\quad+\sum_{h=L+1}^H\EE^{\pitil}\left[
 \norm{\bel_h(x_{1:h},a_{1:h-1})
 -\belapx_h(x_{h-L+1:h},a_{h-L:h-1};d_{h-L}^{\pitil})}_1\right]\\
 &\leq O(\delta)+2H\sqrt S\,(8\delta)^{L/2},
\end{align*}
where we used \cref{eq:distillation-decodability-sum} with $\mu=\pi$ and \cref{eq:distillation-belief-error} with $\mu=\pitil$.
If $L\geq H$, the second sum is empty.
The value bound follows from the normalization of trajectory rewards.
\end{proof}

\subsection{Finite-sample guarantee}

\begin{theorem}
\label{thm:perturbed-mdp-il-finite-sample}
There is a universal constant $C_{\ref{thm:perturbed-mdp-il-finite-sample}}>0$ with the following property. Fix an integer $L\geq1$, a latent policy $\pilat\in\Pilat$, $\delta,\epopt>0$, and $\eta\in(0,1)$. Suppose that the POMDP $\cP$ is a $\delta$-perturbed Block MDP with deterministic transitions. If
\[
 n\geq C_{\ref{thm:perturbed-mdp-il-finite-sample}}
 X^L A^{3L+1}H^2\epopt^{-2}\log(2Hn/\eta),
\]
the output $\pihat$ of the modified $\forward$ algorithm with $n$ samples per step satisfies, with probability at least $1-\eta$,
\[
 J(\pilat)-J(\pihat)
 \leq\TV(\bbP^{\pilat},\bbP^{\pihat})
 \leq O(\delta)+2H\sqrt S\,(8\delta)^{L/2}+\epopt.
\]
Moreover, $\pihat$ can be computed in time $\poly(n,H,X^L,A^L)$.
\end{theorem}

\begin{proof}
At each step $h$, the fitted policy is the empirical distribution of the expert action labels for each observed window; at windows with no samples, choose any fixed action distribution. These counts can be computed in the stated time.

For the statistical bound, fix $h$ and condition on all training data from earlier steps. Then the roll-in policy $\pihat_{1:h-L-1}$ is fixed, and the $n$ fresh trajectories and sampled expert labels at step $h$ are independent. Write
$\nu_h=\pihat_{1:h-L-1}\circ_{h-L}\Unif(\cA)$ and
$z_h=(x_{\max(1,h-L+1):h},a_{\max(1,h-L):h-1})$.
For $h\leq L$, $\nu_h$ plays uniformly random actions from the start.

Let $g_{\pi_h}$ be the joint law of $(z_h,y_h)$ obtained by drawing $z_h$ under $\nu_h$ and then $y_h\sim\pi_h(\cdot\mid z_h)$, and let $g^\star$ instead draw $y_h\sim\pilat_h(\cdot\mid s_h)$. By the same argument as in \cref{lemma:forward-bayes}, $g^\star=g_{\pihat_h^\star}$, where
\[
 \pihat_h^\star(\cdot\mid z_h)=
 \begin{cases}
 \pilat_h\circ\belapx_h(z_h;d_{h-L}^{\pihat}) & h>L,\\
 \pilat_h\circ\bel_h(x_{1:h},a_{1:h-1}) & h\leq L.
 \end{cases}
\]
Here $d_{h-L}^{\pihat}$ depends only on the already fitted prefix, so it is fixed under this conditioning.

Let $\widehat g$ be the empirical joint law of the $n$ pairs $(z_h,y_h)$. At each window, $\widehat g$ and $g_{\pihat_h}$ have the same conditional action law; their window marginals are the empirical and true laws, respectively. Therefore, by marginalization and the triangle inequality,
\[
 \TV(g^\star,g_{\pihat_h})
 \leq\TV(g^\star,\widehat g)+\TV(\widehat g,g_{\pihat_h})
 \leq2\TV(g^\star,\widehat g).
\]
There are at most $D:=X^L A^{L+1}$ window/action pairs. Applying Hoeffding's inequality to every subset of this alphabet and taking a union bound over at most $2^D$ subsets gives, with conditional probability at least $1-\eta/H$,
\[
 \TV(g^\star,\widehat g)
 \leq\sqrt{\frac{D\log2+\log(2H/\eta)}{2n}}.
\]
Consequently, on this event,
\begin{align*}
 \EE^{\nu_h}[\TV(\pihat_h^\star(\cdot\mid z_h),\pihat_h(\cdot\mid z_h))]
 &=\TV(g^\star,g_{\pihat_h})\\
 &\leq O\left(\sqrt{\frac{X^L A^{L+1}\log(2Hn/\eta)}{n}}\right).
\end{align*}
The distribution of the history $(s_{1:h},x_{1:h},a_{1:h-1})$ before action $a_h$ under the learned policy is dominated by its distribution under $\nu_h$ with density ratio at most $A^L$: the policies agree before the random-action window, which contains at most $L$ actions. Consequently,
\begin{equation}\label{eq:opt-error-bound}
 \EE^{\pihat}[\TV(\pihat_h^\star(\cdot\mid z_h),\pihat_h(\cdot\mid z_h))]
 \leq O\left(A^L\sqrt{\frac{X^L A^{L+1}\log(2Hn/\eta)}{n}}\right)
 \leq\frac{\epopt}{H}.
\end{equation}
The last inequality follows from the sample assumption for a sufficiently large universal constant. Averaging over earlier training data and taking a union bound gives \cref{eq:opt-error-bound} simultaneously for all $h$ with probability at least $1-\eta$.

Fix a realization of the learned policy on this event. Write $C=C_{\ref{prop:vinf-decay}}$. If $\delta\geq1/(64C)$, the conclusion follows from $\TV\leq1\leq64C\delta$, so assume $\delta<1/(64C)$.
At full history $\tau_h=(x_{1:h},a_{1:h-1})$, abbreviate
$b_h=\bel_h(\tau_h)$ and
\[
 \bar b_h=
 \begin{cases}
 \belapx_h(z_h;d_{h-L}^{\pihat}) & h>L,\\
 b_h & h\leq L.
 \end{cases}
\]
Conditionally on $\tau_h$ under $\pihat$, the latent state has law $b_h$ by \cref{lemma:belief-is-cond-prob}. Data processing gives
$\TV(\pilat_h(\cdot\mid s),\pilat_h\circ b_h)\leq1-b_h(s)$, and hence
\[
 \EE^{\pihat}\left[
 \TV(\pilat_h(\cdot\mid s_h),\pilat_h\circ b_h)
 \mid\tau_h\right]
 \leq1-\sum_s b_h(s)^2\leq2\Vinf(b_h).
\]
A hybrid argument for trajectory laws, followed by the triangle inequality and data processing, now gives
\begin{align*}
 \TV(\bbP^{\pilat},\bbP^{\pihat})
 &\leq\sum_{h=1}^H\EE^{\pihat}
 [\TV(\pilat_h(\cdot\mid s_h),\pihat_h(\cdot\mid z_h))]\\
 &\leq2\sum_{h=1}^H\EE^{\pihat}[\Vinf(b_h)]
 +\sum_{h=L+1}^H\EE^{\pihat}[\norm{b_h-\bar b_h}_1]\\
 &\quad+\sum_{h=1}^H\EE^{\pihat}
 [\TV(\pihat_h^\star(\cdot\mid z_h),\pihat_h(\cdot\mid z_h))]\\
 &\leq O(\delta)+2H\sqrt S\,(8\delta)^{L/2}+\epopt.
\end{align*}
For the last step, apply \cref{eq:distillation-decodability-sum,eq:distillation-belief-error} with $\mu=\pihat$, and sum \cref{eq:opt-error-bound}. If $L\geq H$, the belief-error sum is empty. The value bound again follows from the normalization of trajectory rewards.
\end{proof}

\section{Omitted Proofs for Reinforcement Learning}
In this section we prove \cref{cor:golowich-perturbed-mdp-main}, stated formally below.

\begin{corollary}\label{cor:golowich-perturbed-mdp}
There is a reinforcement learning algorithm that, for any $\delta,\beta\in(0,1/3)$, $L\in\NN$, and $\delta$-perturbed Block MDP $\cP$, learns an executable policy $\pi^{\rl}$ satisfying
\[
J(\pistar)-J(\pi^{\rl})
\leq \min\left\{1,\,(8\delta)^{L/3}\cdot\poly(S,X,H)\right\}
\]
with probability at least $1-\beta$, using time and samples at most
\[
(XA/\delta)^{\bigoh(L)}\cdot\poly(H,S,\log(1/\beta)).
\]
\end{corollary}

\begin{proof}
Fix $\gamma=1/3$. Since $\delta<1/3$, \cref{remark:perturbed-block-golowich} implies $\gamma$-observability. Increase the universal constant $C^\star$ in \cref{thm:golowich} to at least $1$, if necessary, and write
\[
K:=3C^\star H^5S^{7/2}X^2,
\qquad
\epsilon_L:=(4K)^{1/3}(8\delta)^{L/3}.
\]
If $\epsilon_L\geq1$, return any executable policy. Its regret is at most $1$ by the normalization of trajectory rewards, which proves the claimed bound after choosing the displayed polynomial to dominate $(4K)^{1/3}$. This case includes every $\delta\geq1/8$.

Suppose henceforth that $\epsilon_L<1$, and set $\epsilon:=\epsilon_L$ and $\phi:=\epsilon_L/K$. Then $\epsilon,\phi\in(0,1)$, and $\phi$ is exactly the value required by \cref{thm:golowich}. Apply \cref{thm:belief-contraction} to the process initialized at an arbitrary distribution $\cD'$ at step $h-L$, holding the earlier history fixed. For any comparison prior $\cD$ satisfying $\norm{\cD'/\cD}_\infty\leq1/\phi$, the expected $\ell_1$ distance between the two resulting beliefs is at most
\[
2(8\delta)^{L/2}\sqrt{\norm{\cD'/\cD}_\infty}
\leq 2(8\delta)^{L/2}\phi^{-1/2}
=2\sqrt{K}(8\delta)^{L/2}\epsilon_L^{-1/2}
=\epsilon_L.
\]
The last equality follows from $\epsilon_L^{3/2}=2\sqrt{K}(8\delta)^{L/2}$. This verifies every prior pair and fixed history in \cref{def:gen-belief-contraction}, so $\cP$ satisfies $(\epsilon_L;\phi,L)$-belief contraction.

Applying \cref{thm:golowich} gives regret at most $\epsilon_L\cdot\poly(S,X,H)$ with probability at least $1-\beta$. Substitution of $\epsilon_L$ yields the claimed exponent $L/3$; the universal bound of $1$ supplies the cap. Its running time is
\[
\poly\big((XA)^L,H,S,\epsilon_L^{-1},\log(1/\beta)\big)
\leq (XA/\delta)^{\bigoh(L)}\cdot\poly(H,S,\log(1/\beta)),
\]
because $\epsilon_L^{-1}=(4K)^{-1/3}(8\delta)^{-L/3}\leq\delta^{-L/3}$. The number of samples is bounded by the same running-time bound.
\end{proof}

\section{A Motivating Toy Model for Smoothing}\label{app:smoothing}

Adding to the discussion from \cref{sec:smooth}, we informally discuss a theoretical toy model in which smoothing the expert may decrease (a metric version of) action-prediction error and thus improve final performance. Consider a horizon-$1$ POMDP where the state and action spaces have metrics $d_S$ and $d_A$, and the reward function $R$ is binary-valued. For each latent state $s$, let $G(s)$ be the set of ``good'' actions, i.e. $G(s) := \{a\in\cA: R(s,a)=1\}$, and let $D(s)$ be the diameter of $G(s)$. Assume these sets are nonempty and have uniformly bounded diameter. Suppose that the following additional assumptions hold:
\begin{enumerate}
\item The map $s \mapsto G(s)$ is Lipschitz from $d_S$ to the Hausdorff distance induced by $d_A$, with a constant independent of $\epsilon$. Thus, perturbing $s$ by $\epsilon$ changes $G(s)$ by $O(\epsilon)$ in Hausdorff distance.
\item Under any observation, the posterior on states is ``$\veps$-local'' with respect to $d_S$, i.e. contained in an $\epsilon$-ball.
\end{enumerate}
With no smoothing, the optimal expert may, for each $s$, play an arbitrary action in $G(s)$, so the ``metric'' action-prediction error (i.e. expected dispersion of actions played by the expert, conditioned on an observation)
is bounded by $O(\epsilon)+\sup_s D(s)$ and need not be small when $\epsilon$ is small. Now suppose motor noise perturbs the chosen action within an $\eta$-radius ball under $d_A$. Such noise can favor actions whose entire $\eta$-neighborhood lies inside $G(s)$: these actions retain reward $1$ under every allowed perturbation and are therefore optimal whenever they exist. This may reduce variation in the expert's selected actions when nearby states have suitably overlapping robust-action sets and the expert chooses consistently among optimal actions. The assumptions above alone do not ensure these additional properties, so they do not establish a general improvement in action-prediction error or distillation performance.

If $\eta$ becomes too large, some states may have no action whose entire noise neighborhood lies in $G(s)$. The optimal actions may then change substantially across states, potentially increasing action-prediction error and reducing reward.

\section{Supplemental Materials for Experiments}\label{app:experiment}
In \cref{app:exp-mis-dec} we present details for our empirical test of whether perfect decodability holds in image-based locomotion tasks. In \cref{app:omitted-figures} we present figures omitted from the main body of the paper. In \cref{app:hyperparameters} we include hyperparameter choices and details about compute resources.

\horizon{
\subsection{Misspecification of Decodability in Practice}\label{app:exp-mis-dec}
\begin{figure}[h]
    \centering
    \includegraphics[width=0.32\textwidth]{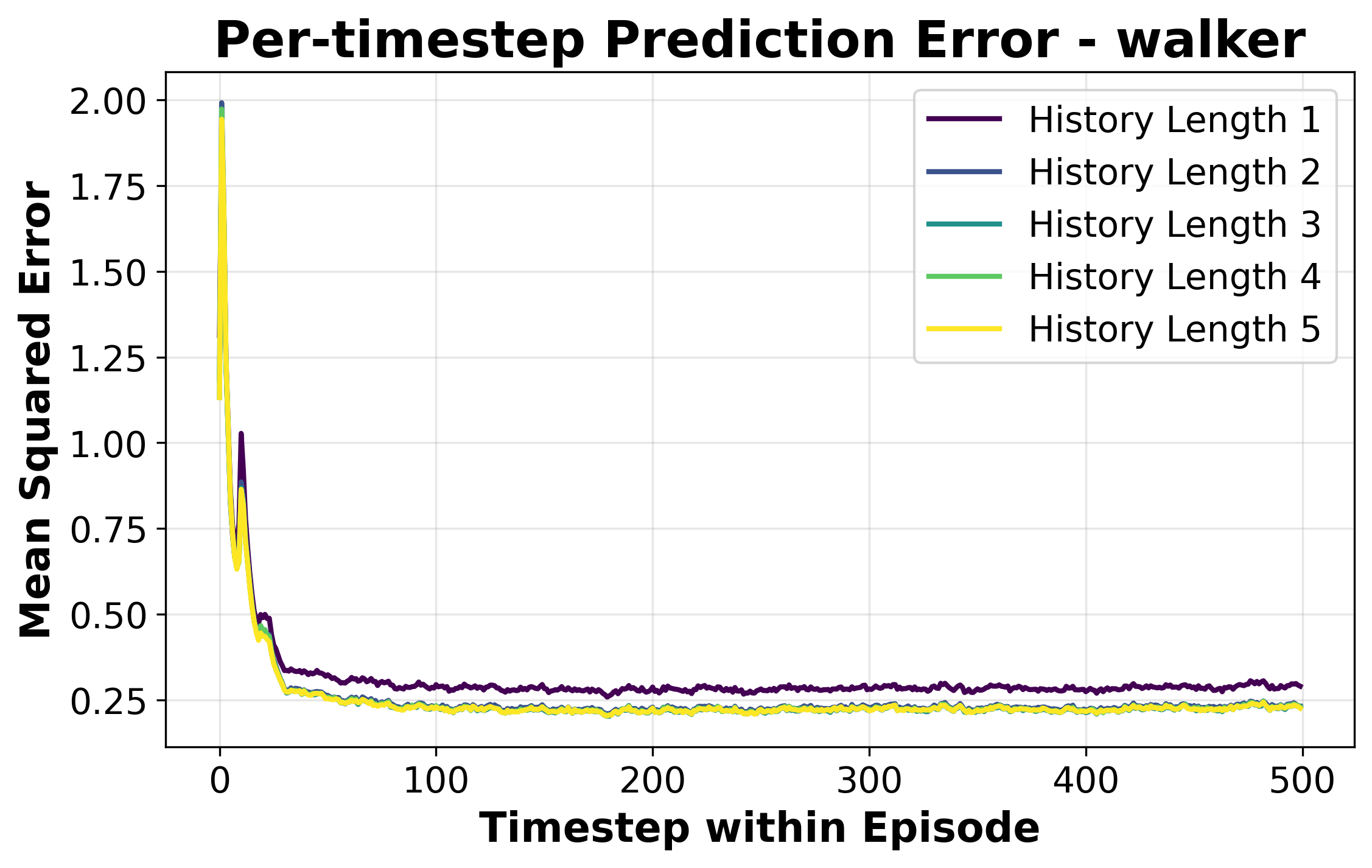}
    \includegraphics[width=0.32\textwidth]{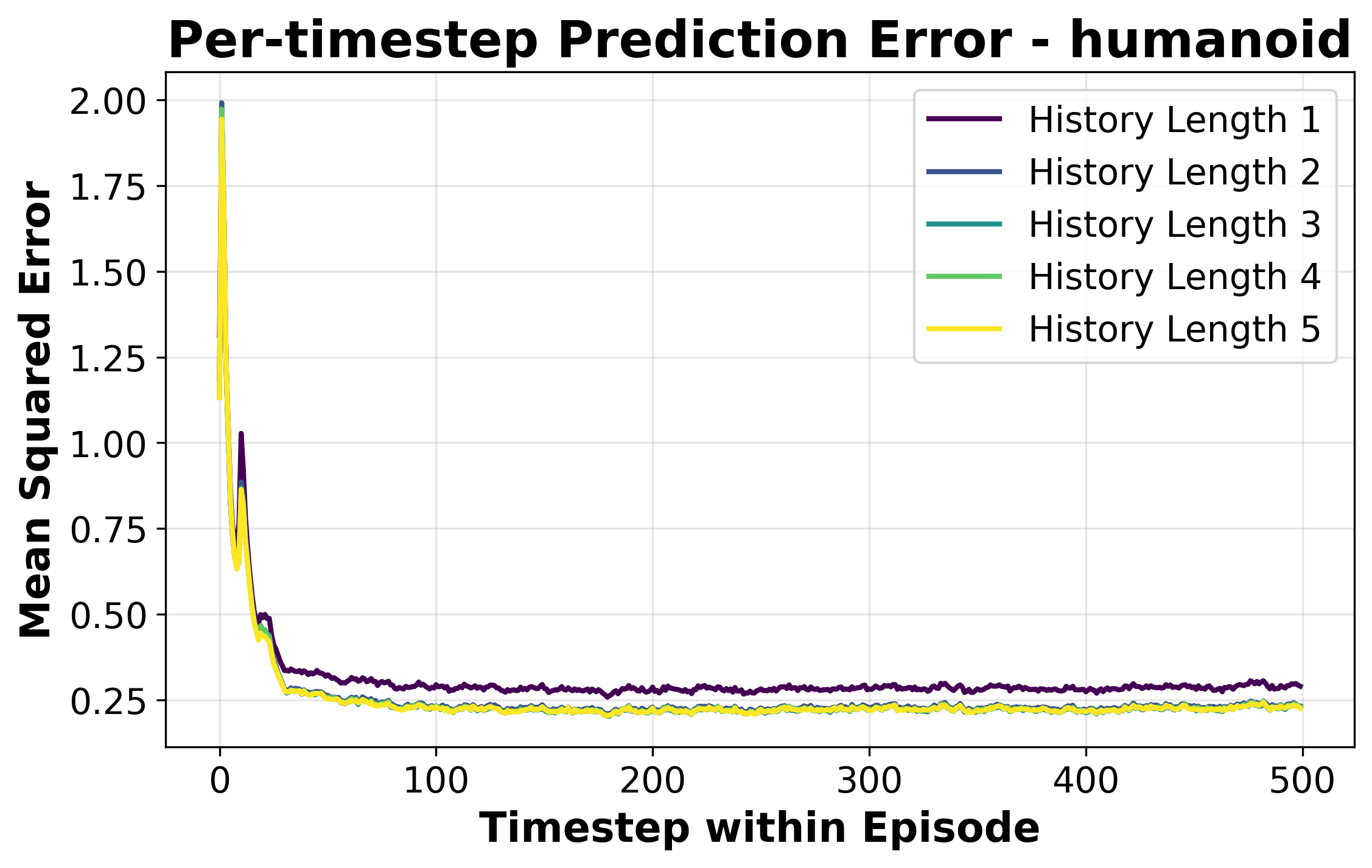}
    \includegraphics[width=0.32\textwidth]{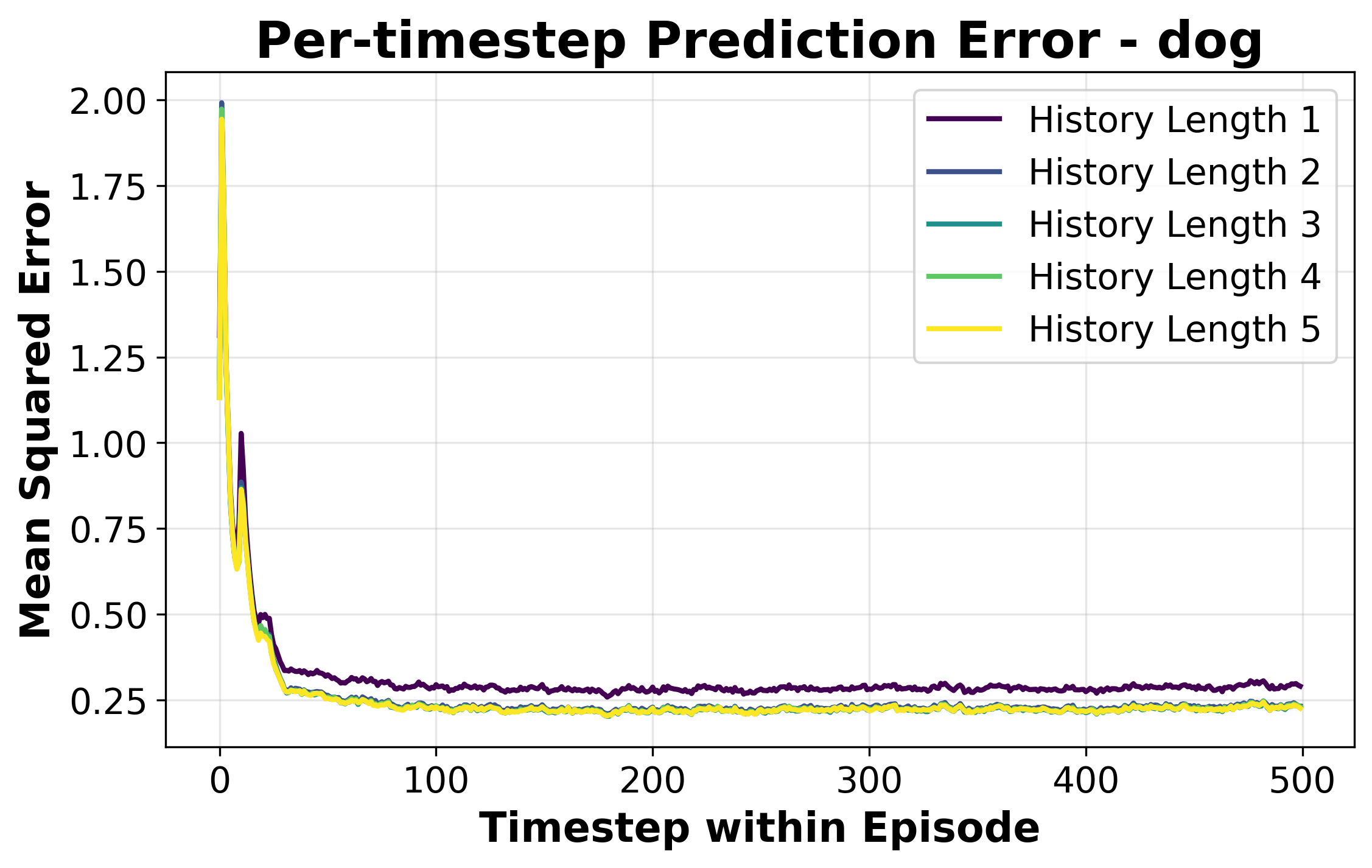}
    \caption{Per timestep validation error of the state-prediction model for different choices of frame-stack $L \in \{1,2,3,4,5\}$. The model is trained to predict the state given the most recent $L$ observations. We present the average error across 5 random seeds.}
    \label{fig:decodable}
\end{figure}

For each of the three tasks (\texttt{walker-run}, \texttt{humanoid-walk}, and \texttt{dog-walk}), we collect $2000$ trajectories from the expert latent policy, and for each $L \in \{1,2,3,4,5\}$ we train a model directly mapping from $L$ observations $x_{h-L+1:h}$ to the state $s_h$, by minimizing mean-squared error. We then evaluate the validation error of the model on $100$ trajectories collected from the same policy and plot the per-timestep error in \pref{fig:decodable}. We normalize the states with the trajectory mean and standard deviation of the combined dataset. The prediction error may reflect limited model capacity, finite data or optimization error, truncation of the observation history, or inherent ambiguity about the latent state. We observe small error at later timesteps but larger error at initial timesteps. For $h\leq L$, no observations are omitted by truncation, so this pattern provides evidence of early-state ambiguity. However, the experiment does not exclude approximation or optimization error at early timesteps, or information contained in past actions, and therefore does not establish failure of perfect decodability from observation/action histories.

\paragraph{Ablation} To investigate whether the error at initial timesteps is due to parameter sharing, we also tried to train non-stationary models (i.e., one model for each timestep) or using weighted loss with higher weights for the initial timesteps, but neither approach significantly changed the results.

\paragraph{Conclusions} We conjecture that the higher error at early timesteps is due to a nearly uniform distribution for the initial state distribution, where states that induce occlusion may be quite likely. In later time-steps, the observations along the expert trajectory are in a more stable regime and thus may introduce less occlusion (and hence be more likely to correspond to a unique latent state).} Finally, we observe that even at later time-steps there is still significant prediction error. This indicates difficulty learning a perfect decoder with the models and data used here, and it roughly corresponds with task difficulty (\texttt{humanoid-walk} and \texttt{dog-walk} are harder than \texttt{walker-run}).

\paragraph{Source of the error} One may naturally wonder if the error is caused by unpredictable state components that are irrelevant to decision-making (one example could be the absolute position of the agent, but in the environments that we test on, absolute coordinates are  actually not part of the latent space). If this is the case, then the error would not negatively impact the performance of the policy distilled from the latent expert.

One piece of evidence against this hypothesis is that the action-prediction error is indeed non-trivial for our tasks of interest (c.r. \pref{fig:error_smooth}), which suggests that the error is not only due to irrelevant components. In addition, we take a more detailed look at the state prediction error and identify the components that contribute most to the error. With the same setup as in \pref{fig:decodable}, we compute the mean squared error for each coordinate of the state, averaged over the whole trajectories, and present the top 5 and bottom 5 coordinates in \pref{tab:top_5}. We see that the coordinates that contribute most to the error are mostly angular velocities of limbs, which are indeed hard to predict from images. On the other hand, the coordinates that contribute least to the error are mostly joint angles or balance point coordinates, which are easier to predict from images. From first principles, all of these coordinates are crucial for the optimal policy, providing additional confirmation that the error is not only due to irrelevant components.

\begin{table}[h]
    \centering
        \caption{(Left) Top 5 and (right) bottom 5 coordinates contributing to state prediction error.}
        \centering
      \begin{tabular}{c|c}
        \toprule
       Coordinates & Error \\
        \midrule
        left\_ankle\_y angular velocity & 0.82 \\
        left\_hip\_x angular velocity & 0.81 \\
        right\_shoulder1 angular velocity & 0.69 \\
        right\_shoulder2 angular velocity & 0.67 \\
        right\_ankle\_y angular velocity & 0.64 \\
      \bottomrule
      \end{tabular}
      \hspace{1cm}
      \begin{tabular}{c|c}
        \toprule
       Coordinates & Error \\
        \midrule
      left\_elbow joint angle & 0.006 \\
      balance point z & 0.020 \\
      balance point y & 0.034 \\
      balance point x & 0.044 \\
      left\_knee joint angle & 0.045 \\
      \bottomrule
      \end{tabular}\\
      \label{tab:top_5}
\end{table}

\nothorizon{
\subsection{Error compounding is algorithm-dependent}\label{app:horizon}
The horizon dependence of the error in imitation learning has received intensive theoretical \citep{rajaraman2020toward,foster2024is,rohatgi2025computational} and empirical \citep{ross2010efficient,laskey2017dart,block2023butterfly} study. It is widely believed that online methods such as $\Dagger$ outperform offline methods such as behavior cloning due to being able to \emph{recover} from mistakes \citep{ross2014reinforcement,rajaraman2020toward} and thereby avoid error compounding over the horizon. Does this gap manifest in expert distillation for POMDPs?
In \pref{fig:horizon}, we vary the horizon $H$ from $50$ to $450$, and measure the sub-optimality of both offline and online expert distillation. Notably we first normalize the trajectory return by the horizon so that trajectory reward is between $0$ and $1$ to match our theoretical setting.  Motivated by the accumulation of per-step errors in our upper bounds, we further normalize the suboptimality by the mean validation action-prediction error. We show the comparison between the normalized suboptimality and the horizon in \pref{fig:horizon}, and we see a strong linear relationship in BC, but weaker in $\Dagger$, likely due to recoverability \citep{ross2011reduction,rajaraman2021value,foster2024is}. This contrasts with empirical results of \cite{foster2024is}: they perform \emph{well-specified} behavior cloning in similar tasks, and find little horizon dependence. Our results therefore suggest that misspecification, rather than sampling error, may be a more critical source of horizon dependence. \loose}

\subsection{Imitating a Smoother Expert under Deterministic Latent Dynamics}\label{app:smooth}

In \pref{sec:smooth}, we showed that under stochastic latent dynamics, imitating a smoother expert (which is trained under higher motor noise) can lead to better performance. A natural question is whether this phenomenon also exists under deterministic latent dynamics. Under deterministic latent dynamics, \pref{thm:perturbed-mdp-il-guarantee} guarantees that imitating the non-smoothed expert achieves $O(\delta)$ suboptimality with enough framestack. That said, it remains unclear if smoothing the expert can help improve the performance in practice. To answer this question, we conduct experiments in the same setup as in \pref{sec:smooth}, but with motor noise $\sigma=0$ when performing expert distillation. We vary the motor noise level used to train the latent expert over $\{0.1,0.2,0.3,0.4,0.5\}$, and we use a framestack of size $3$.

We present the results in \pref{fig:smooth_c0}. We see that imitating a smoother expert does not lead to better performance in this case, and the best performance is achieved by imitating the non-smoothed expert (c.r., \pref{fig:det-il-rl}). This corroborates our theoretical findings.
\begin{figure*}[h]
    \centering
    \includegraphics[width=0.4\textwidth]{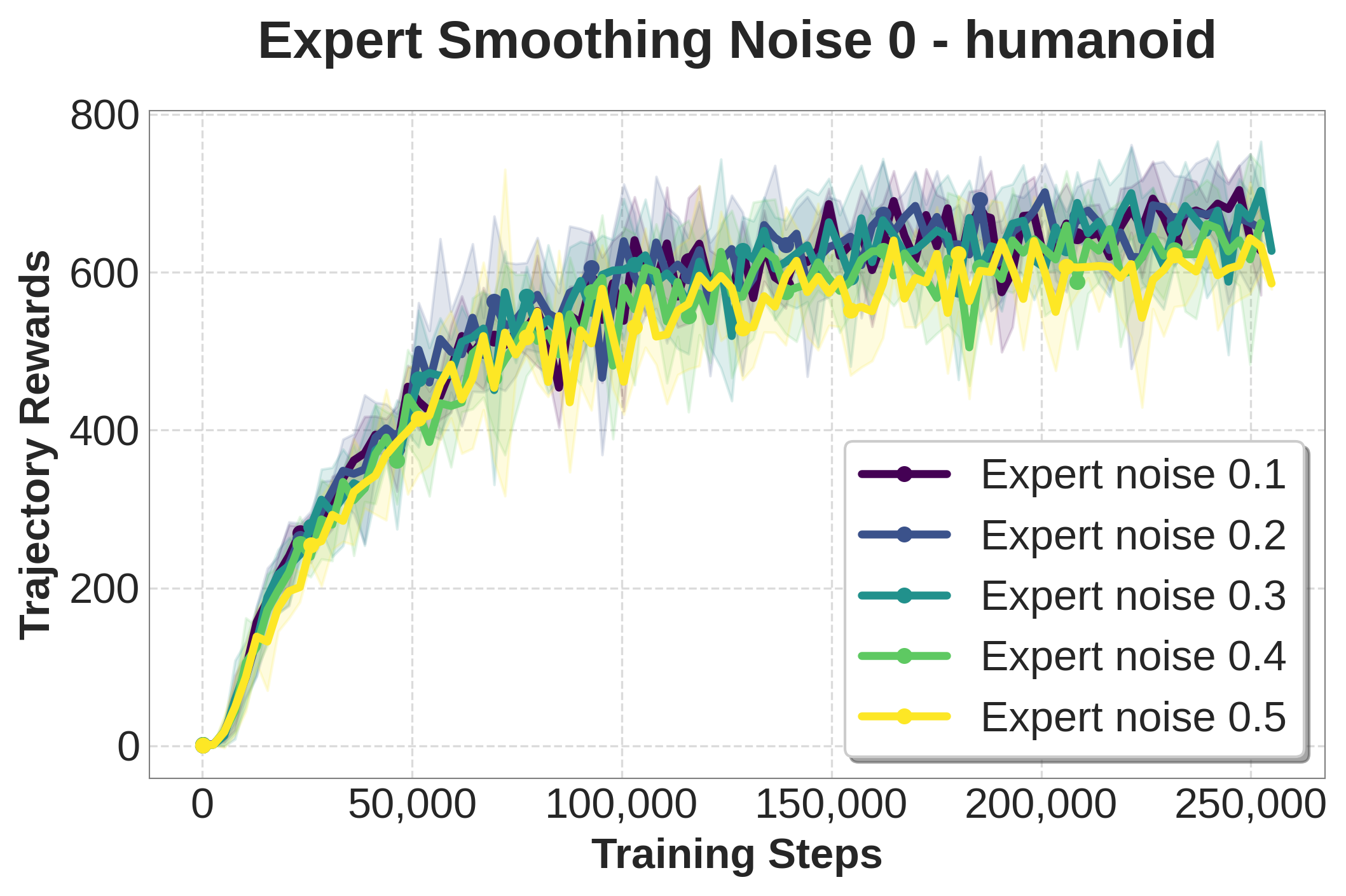}
    \includegraphics[width=0.4\textwidth]{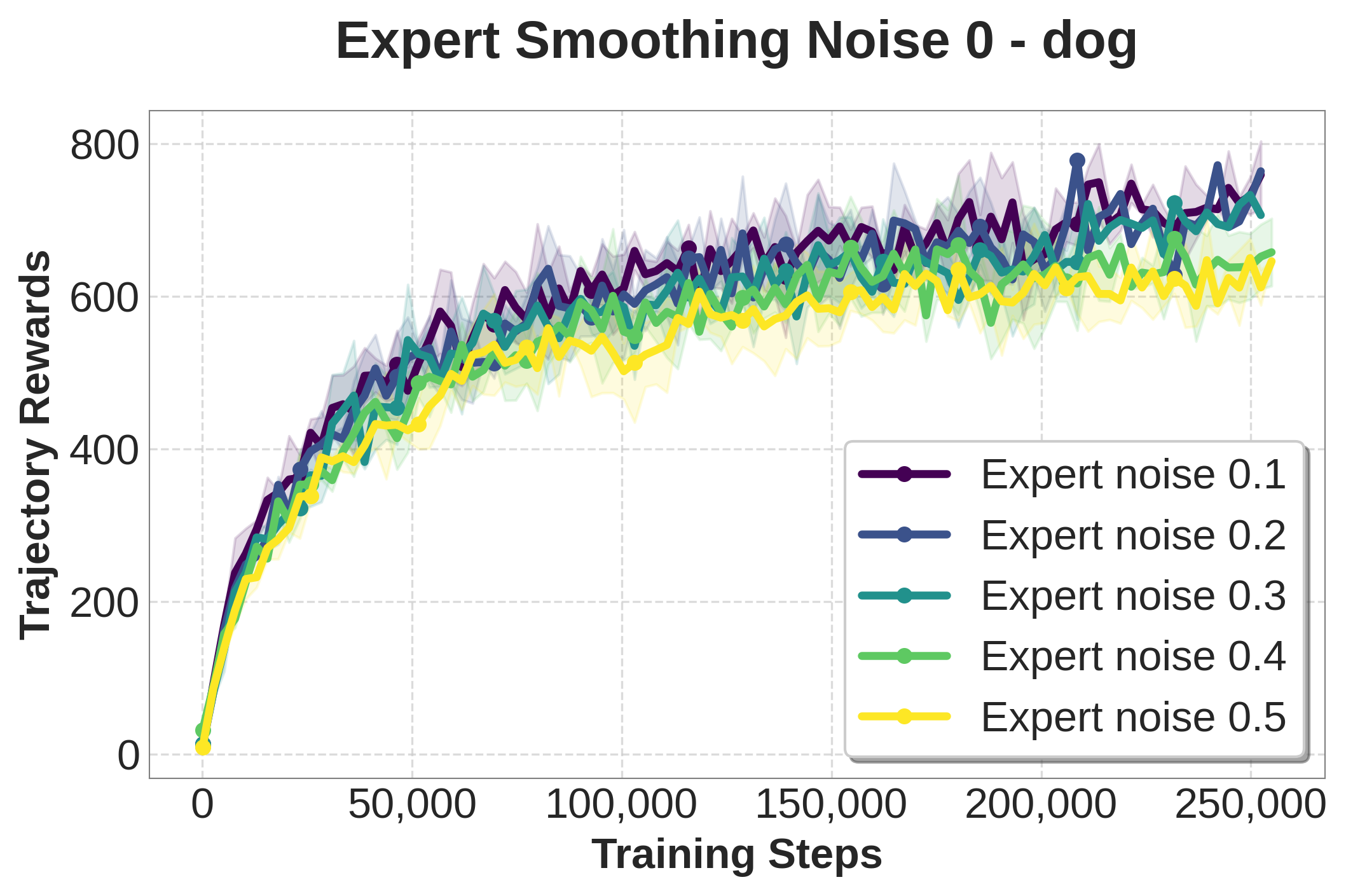}
    \caption{Performance of $\Dagger$ on the validation dataset for the \texttt{humanoid-walk} and \texttt{dog-walk} environments with motor noise $\sigma=0$, as the noise level for the training environment (i.e. the environment in which the latent expert was trained) varies over $\{0.1,0.2,0.3,0.4,0.5\}$. }
    \label{fig:smooth_c0}
\end{figure*}

\newpage
\subsection{Omitted Figures}\label{app:omitted-figures}
\subsubsection{Belief contraction with/without motor noise}
\iclr{
\begin{figure*}[h]
\centering
\includegraphics[width=0.3\textwidth]{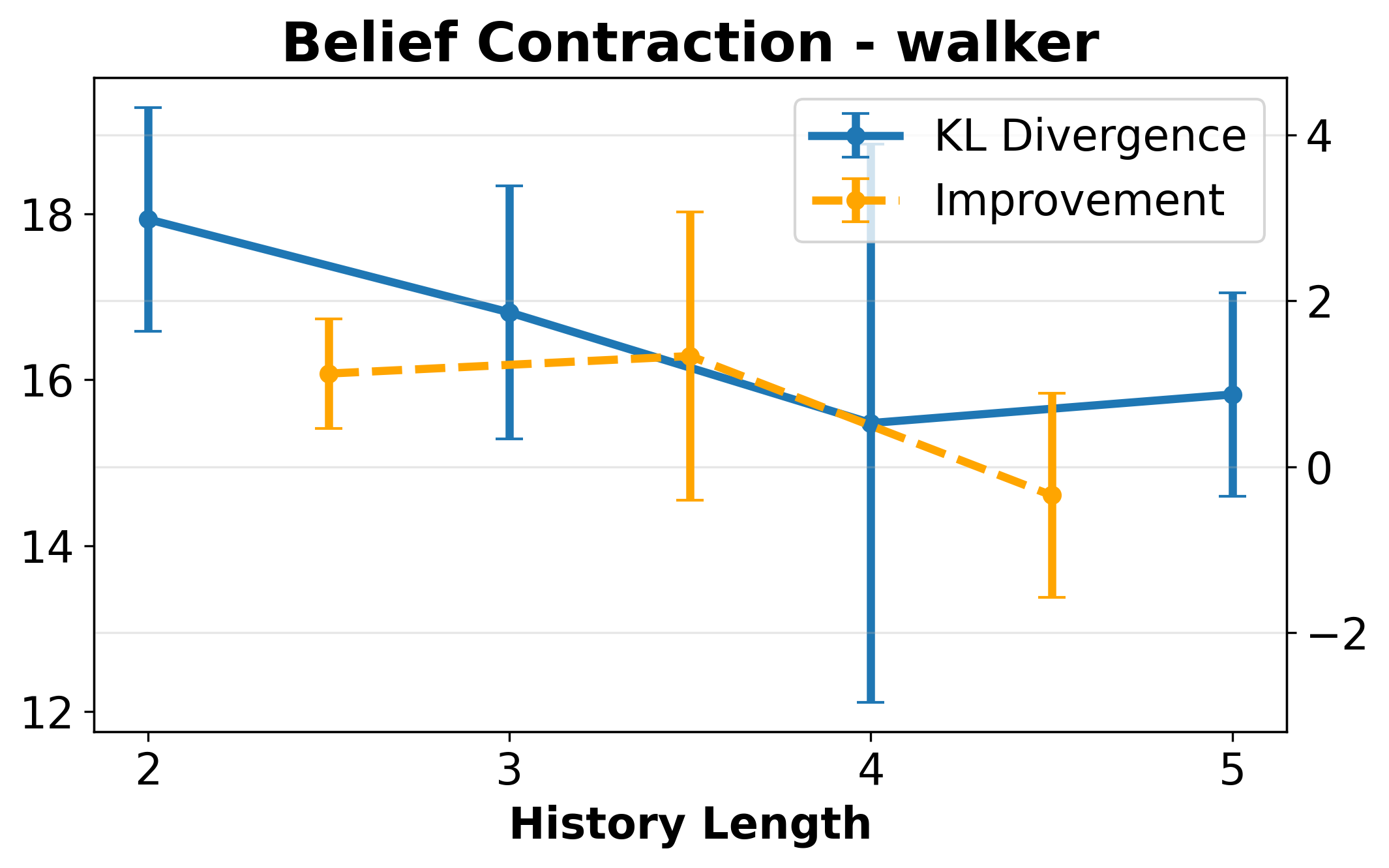}
\includegraphics[width=0.3\textwidth]{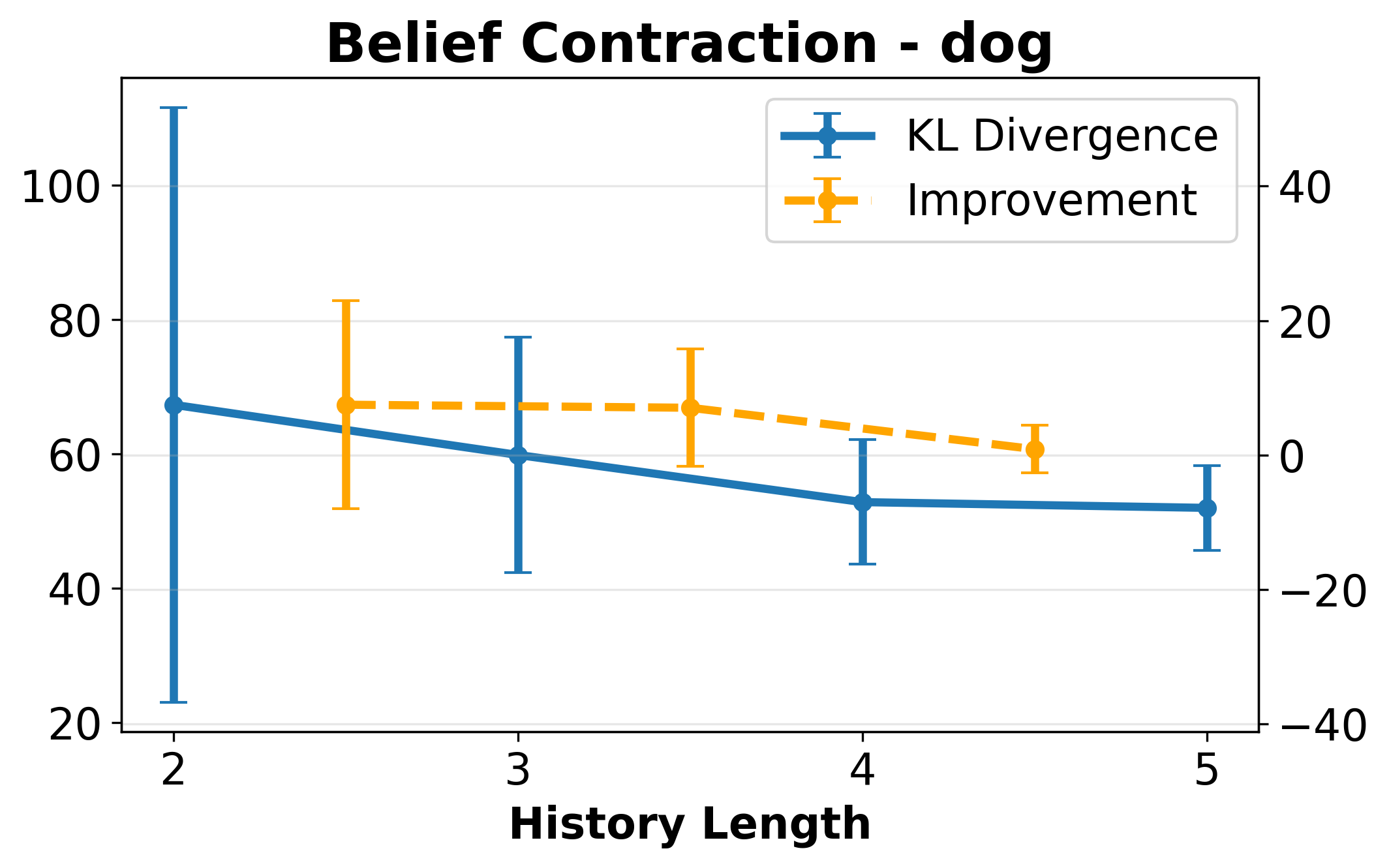}
\includegraphics[width=0.3\textwidth]{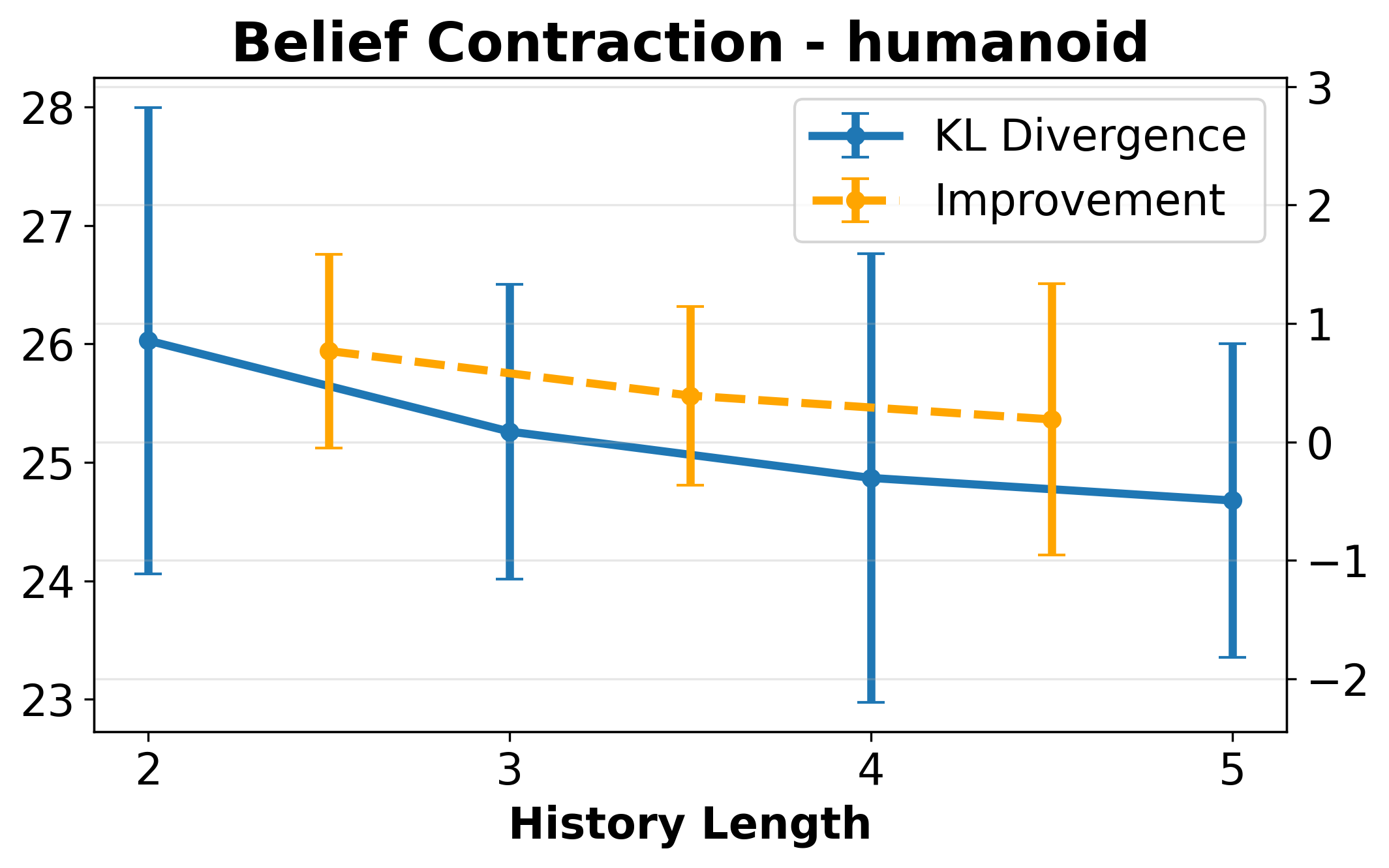}
\caption{KL divergence between learned belief models with frame-stacks $L\in\{2,3,4,5\}$ and an $L=10$ reference model on all tasks. Each neural network predicts a Gaussian distribution over states. We average the divergence over 100 validation episodes from the latent expert. The orange plot denotes the decrease in KL divergence between frame-stack lengths. We repeat each experiment 5 times and plot the mean and standard deviation. This diagnostic proxy decreases modestly as the frame-stack length increases.}\label{fig:belief-contraction}
\end{figure*}
}

\begin{figure*}[h]
\centering
\includegraphics[width=0.4\textwidth]{figures/belief_contraction_humanoid.png}
\includegraphics[width=0.4\textwidth]{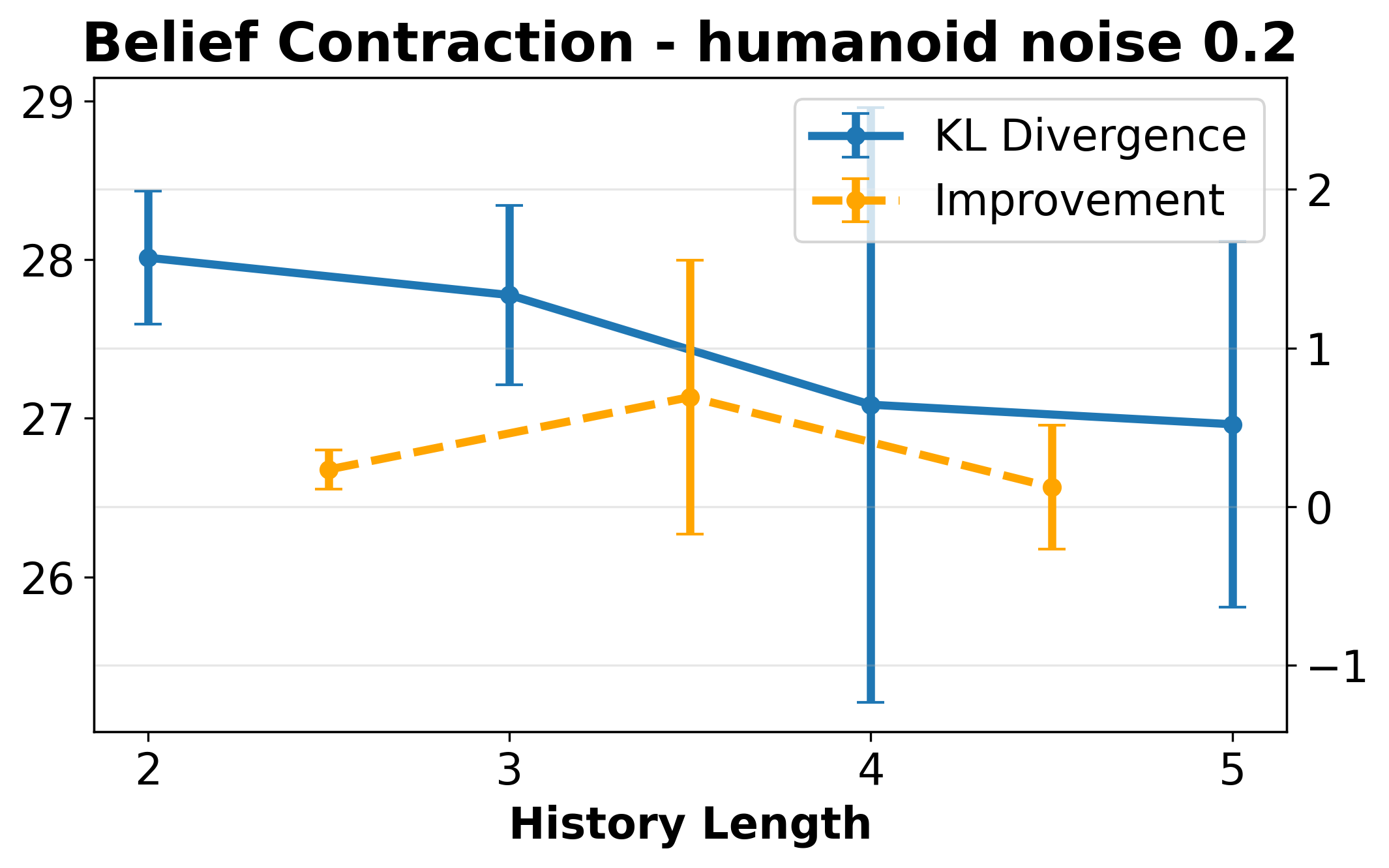}
\caption{KL divergence between learned belief models with frame-stacks $L\in\{2,3,4,5\}$ and an $L=10$ reference model on \texttt{humanoid-walk}, with and without motor noise. Each neural network predicts a Gaussian distribution over states. We average the divergence over 100 validation episodes from the latent expert. The orange plot denotes the decrease in KL divergence between frame-stack lengths. We repeat each experiment 5 times and plot the mean and standard deviation. The diagnostic proxy behaves similarly with and without motor noise.}\label{fig:belief-contraction-sto}
\end{figure*}

\subsubsection{Action prediction error with smoothed experts}

\begin{figure*}[h]
    \centering
    \includegraphics[width=0.4\textwidth]{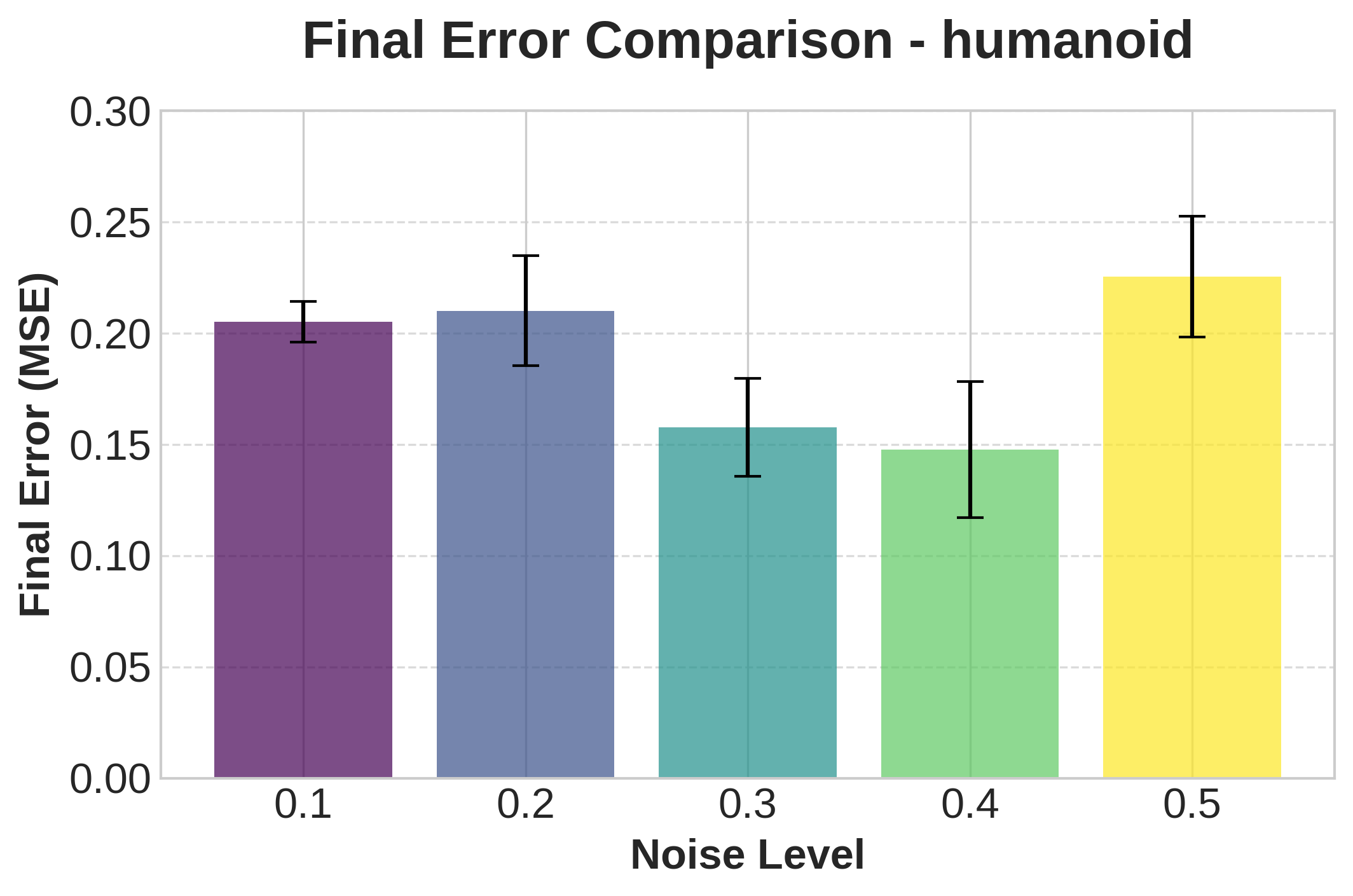}
    \includegraphics[width=0.4\textwidth]{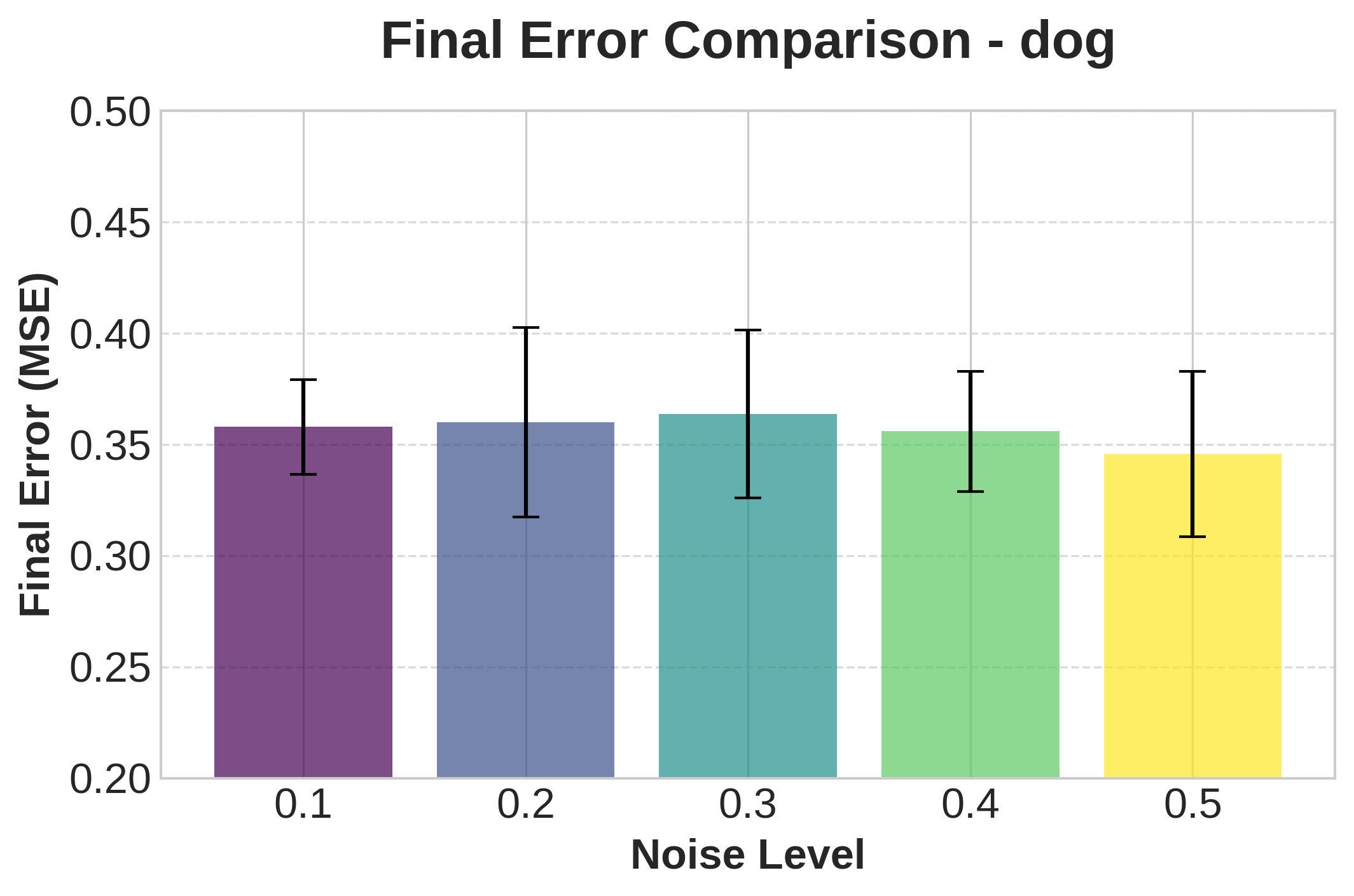}
    \caption{Comparison of the (estimated) action-prediction error of the latent policy on the validation dataset for the \texttt{humanoid-walk} and \texttt{dog-walk} environments with motor noise $\sigma=0.2$, as the noise level for the training environment (i.e. the environment in which the latent expert was trained) varies over $\{0.1,0.2,0.3,0.4,0.5\}$. The action-prediction error was estimated using MSE as a proxy (normalized by dimension of the action space, as detailed in \cref{sec:prelim}). We observe that imitating the latent expert that is trained on the same noise level does not yield the smallest prediction error. Moreover, policies with lower action-prediction error also broadly have higher performance (\cref{fig:smooth}). \loose}
    \label{fig:error_smooth}
\end{figure*}

\newpage

\subsection{Experiment Details}\label{app:hyperparameters}
\paragraph{Hyperparameters for state prediction models} The hyperparameters used and considered for the belief/state prediction models (both deterministic and Gaussian parametrized), corresponding to the experiments in \pref{sec:perfect-decodable,sec:exp-stochastic}, in \pref{tab:belief}. For the neural network architecture, we use the same cnn block prescribed in \citet{fujimoto2025towards}, followed by a three layer neural network with ReLU activation. The architecture remains the same for the policies, and the hyperparameter hidden size refers to the hidden size of the feedforward part of the neural network.
\begin{table}[h]
    \centering
        \caption{Hyperparameters for belief prediction models.}
      \begin{tabular}{c|c|c}
        \toprule
        & Final value & Considered Values \\
        \midrule
      Minibatch size & 256 & 128, 256 \\
      Learning rate & 1e-4 & 1e-3, 2e-4, 1e-4\\
      Optimizer & Adam & Adam \\
      Number of epochs & 100 & 25, 50, 100 \\
      Hidden layer size & 512 & 128, 256, 512 \\
      \bottomrule
      \end{tabular}\\
        \label{tab:belief}
\end{table}

\paragraph{Hyperparameters for expert distillation} The hyperparameters of BC and $\Dagger$ are provided in \pref{tab:bc} and \pref{tab:dagger} respectively. Note that the only exception is that $\Dagger$ is run for 6500 episodes in the motor noise 0.1 and 0.3 experiment because it converges slower than the rest of the experiments.

\begin{table}[h]
    \centering
        \caption{Hyperparameters for BC.}
      \begin{tabular}{c|c|c}
        \toprule
        & Final value & Considered Values \\
        \midrule
      Minibatch size & 256 & 128, 256 \\
      Learning rate & 1e-4 & 1e-3, 2e-4, 1e-4\\
      Optimizer & Adam & Adam \\
      Number of episodes & 2000 & 1000, 2000, 5000 \\
      Number of epochs & 1000 & 100, 500, 1000 \\
      Hidden layer size & 256 & 128, 256, 512 \\
      \bottomrule
      \end{tabular}\\
        \label{tab:bc}
\end{table}

\begin{table}[h]
    \centering
        \caption{Hyperparameters for $\Dagger$.}
      \begin{tabular}{c|c|c}
        \toprule
        & Final value & Considered Values \\
        \midrule
      Minibatch size & 256 & 128, 256 \\
      Learning rate & 1e-4 & 1e-3, 2e-4, 1e-4\\
      Optimizer & Adam & Adam \\
      Number of episodes & 5 & 2, 5, 10, 20 \\
      Number of iterations & 5000 & 1000, 2000, 5000, 100000 \\
      Number of gradient step per iteration & 50 & 20, 50, 100 \\
      Hidden layer size & 256 & 128, 256, 512 \\
      \bottomrule
      \end{tabular}\\
        \label{tab:dagger}
\end{table}

\paragraph{Hyperparameters for RL} The hyperparameters for RL follows the original  hyperparameters prescribed in \citep{fujimoto2025towards}, and we train for 50000 episodes.

\paragraph{Computation details} All of our experiments are run with 1 L40S GPU with 8 threads of CPU.

\end{document}